\documentclass[11pt, reqno,english]{amsart}
\pdfoutput=1

\usepackage[utf8]{inputenc}
\usepackage{graphicx,amsmath,amsthm,amsfonts,amssymb,enumitem}
\allowdisplaybreaks
\RequirePackage[numbers]{natbib}
\RequirePackage[colorlinks,linkcolor=blue,citecolor=blue,urlcolor=blue]{hyperref}
\usepackage[letterpaper]{geometry}
\usepackage{times,mathtools,xcolor}

\DeclareMathOperator{\Var}{Var} 
\usepackage{amssymb}

\newtheorem{theorem}{Theorem}[section]
\newtheorem{lemma}{Lemma}[section]

\newtheorem{assumption}{Assumption}[section]
\usepackage{bbm}
\title[Growing-Dimensional SGD Inference for Least-squares]{Statistical Inference for Linear Functionals of Online \\ Least-squares SGD when $t \gtrsim d^{1+\delta}$}
\date{}
\author{Bhavya Agrawalla}
\author{Krishnakumar Balasubramanian}
\author{Promit Ghosal}

\address{\hspace{-0.16in}Computer Science Department, Carnegie Mellon University. Email: \texttt{bbagrawa@andrew.cmu.edu}
\newline Department of Statistics, University of California, Davis. Email: \texttt{kbala@ucdavis.edu}
\newline\hspace{-0.16in}Department of Statistics, University of Chicago. Email: \texttt{promit@uchicago.edu}}
\usepackage[foot]{amsaddr}

\begin{document}

\maketitle

\begin{abstract}
Stochastic Gradient Descent (SGD) has become a cornerstone method in modern data science. However, deploying SGD in high-stakes applications necessitates rigorous quantification of its inherent uncertainty. In this work, we establish \emph{non-asymptotic Berry--Esseen bounds} for linear functionals of online least-squares SGD, thereby providing a Gaussian Central Limit Theorem (CLT) in a \emph{growing-dimensional regime}. Existing approaches to high-dimensional inference for projection parameters, such as~\cite{chang2023inference}, rely on inverting empirical covariance matrices and require at least $t \gtrsim d^{3/2}$ iterations to achieve finite-sample Berry--Esseen guarantees, rendering them computationally expensive and restrictive in the allowable dimensional scaling. In contrast, we show that a CLT holds for SGD iterates when the number of iterations grows as $t \gtrsim d^{1+\delta}$ for any $\delta > 0$, significantly extending the dimensional regime permitted by prior works while improving computational efficiency. The proposed online SGD-based procedure operates in $\mathcal{O}(td)$ time and requires only $\mathcal{O}(d)$ memory, in contrast to the $\mathcal{O}(td^2 + d^3)$ runtime of covariance-inversion methods. To render the theory practically applicable, we further develop an \emph{online variance estimator} for the asymptotic variance appearing in the CLT and establish \emph{high-probability deviation bounds} for this estimator. Collectively, these results yield the first fully online and data-driven framework for constructing confidence intervals for SGD iterates in the near-optimal scaling regime $t \gtrsim d^{1+\delta}$.

%Stochastic gradient descent (SGD) has emerged as the quintessential method in a data scientist's toolbox. Using SGD for high-stakes applications, however, requires careful quantification of the associated uncertainty. In this work, we establish a non-asymptotic Berry-Esseen bounds for  linear functionals of online  least-squares SGD, providing a Gaussian Central Limit Theorem (CLT) in growing-dimensional regime. Prior approaches for growing-dimensional inference for projection parameters, such as~\cite{chang2023inference}, relied on inverting the empirical covariance matrix and required at least $t \gtrsim d^{3/2}$ iterations to obtain finite-sample Berry--Esseen bounds, making them computationally expensive and limiting the allowable dimensional regimes. In contrast, we show that a CLT holds when the number of SGD iterations $t$ grows as $d^{1+\gamma}$ for any $\gamma > 0$, going beyond the range of dimensions permitted in previous works while also being more computationally efficient. Our online SGD-based approach runs in $\mathcal{O}(td)$ time and requires only $\mathcal{O}(d)$ memory, compared to $\mathcal{O}(td^2 + d^3)$ runtime for covariance-inversion–based methods.  To make the results practical, we develop an online procedure for estimating the variance term appearing in the CLT and establish high-probability bounds for this estimator. Together, these results yield the first fully online and data-driven method for constructing confidence intervals in the near optimal dimensional scaling regime of $d = o(t^{1-\gamma})$ with SGD.
\end{abstract}
%Furthermore, we establish a finite-sample Berry--Esseen bound for projection parameters that holds under the nearly optimal scaling $d = o(t^{1-\gamma})$, substantially improving over the $d = o(t^{2/3})$ regime required by prior assumption-lean results~\cite{chang2023inference}. \kb{why flip }
\section{Introduction}

Stochastic gradient descent~\cite{robbins1951stochastic} is a popular optimization algorithm widely used in data science. It is a stochastic iterative method for minimizing the expected loss function by updating model parameters based on the (stochastic) gradient of the loss with respect to the parameters obtained from a random sample. SGD is widely used for training linear and logistic regression models, support vector machines, deep neural networks, and other such machine learning models on large-scale datasets. Because of its simplicity and effectiveness, SGD has become a staple of modern data science and machine learning, and has been continuously improved and extended to handle more complex scenarios. 

Despite its wide-spread applicability for prediction and point estimation, quantifying the uncertainty associated with SGD is not well-understood. Indeed, uncertainty quantification is a key component of decision making systems, ensuring the credibility and validity of data-driven findings; see, for \emph{e.g.,}~\cite{chee2022plus}, for a concrete medical application where it is not enough to just optimize SGD to obtain prediction performance but is more important to quantify the associated uncertainty. Developing an inferential theory for SGD becomes more challenging in particular in the growing-dimensional setting, when the number of parameters can grow with the number of iterations (or equivalently the number of observations used in online SGD). Such growing-dimensional settings are common in modern statistical machine learning problems and it well-known that online SGD has implicit regularization properties, as examined in several recent works including~\cite{ali2020implicit, varre2021last, zou2021benign,wu2022last,chen2022dimension}.

A crucial step toward developing an inferential theory of SGD is to establish central limit theorems (CLT) and related normal approximation results. Such results in-turn could be used to develop practical inferential procedures. Towards that, in this paper, we establish growing-dimensional CLTs and develop statistical inference methodology for linear functionals of online SGD iterates. Specifically, we focus on the misspecified linear regression model comprised of a  random vector of covariates $X \in \mathbb{R}^d$ and a scalar random variable $Y$. It is well known that the best linear $L_2$ approximation to $Y$ is the linear functional $(\beta^*)^\top X$, where
\begin{align*}
\textstyle
\beta^* := \min_{\theta \in \mathbb{R}^d}\mathbb{E}[(Y-\langle X,\theta\rangle)^2].
\end{align*}
In order to estimate the parameter $\beta^* \in \mathbb{R}^d$, we consider minimizing the above population loss function using online SGD with an initial guess $\theta_0 \in \mathbb{R}^d$. Here, $\langle \cdot, \cdot \rangle$ represents the Euclidean inner-product. Letting the $i^{\textrm{th}}$ random observation be $(X_i, Y_i)$ and the step-size at the $i^{\textrm{th}}$ iterate be $\eta_i$, the online SGD update rule is given by
 \begin{align}\label{eq:ithiterate}
 \theta_{i}:= \theta_{i-1}+ \eta_{i} X_{i}(Y_{i}- \langle X_{i}, \theta_{i-1}\rangle).
\end{align}
We emphasize here that the online SGD uses one observation per iteration, and the observations are assumed to be independent and identically distributed across the iterations. Hence, suppose we run it for $t$ iterations, then the overall number of observations used is also $t$. Letting $a \in \mathbb{R}^{d}$ be a $d$-dimensional deterministic vector, we wish to establish a central limit theorem for the following linear functional $\langle a, \theta_t \rangle$. Technically, in the above discussion we consider a growing dimensional setup in which the dimension $d$ changes with $t$. We simply use $d$ instead of $d_t$ for notational convenience. \\

\textbf{Our Contributions.} We make the following contributions in this work.

\begin{enumerate}[leftmargin=0.22in]
    \item[(1)] We establish a growing-dimensional Central Limit Theorem (CLT) in the form of Berry–Esseen bounds for linear functionals of the least-squares online SGD iterates in~\eqref{eq:ithiterate}. Our main result, stated informally below (and rigorously in Theorem~\ref{th-1-bias-corrected}), provides a finite-sample Gaussian approximation under mild moment and scaling assumptions.

    \vspace{2pt}
    \noindent\textbf{\textit{Informal Statement.}} \emph{Consider the least-squares online SGD update~\eqref{eq:ithiterate} run for $t$ steps with step size $\eta_i = \frac{\eta}{\sqrt{d}\, i^{\alpha}}$ for some $\eta>0$ and $\alpha \in (\tfrac{1}{2},1)$. Suppose further that}
    \begin{itemize}
        \item $\displaystyle \lim_{t,d \to \infty} (\log t + \log d)^2 d^{1/2} t^{-(1-\alpha)} = 0$,
        \item \emph{$X$ and $\epsilon := Y - X^\top \beta^*$ have finite moments of order $4p$, for some absolute constant $p \ge 2$.}
    \end{itemize}
    \emph{Then there exist absolute constants $C_1, C_2 > 0$ such that, for all $t,d \ge C_1$,}
    \begin{align}\label{eq:informalmain}
        \sup_{\gamma \in \mathbb{R}} 
        \bigg|
        \mathbb{P}\!\left(
        \frac{\langle a, \theta_t \rangle - \langle a, \beta^* \rangle}
        {\sqrt{\Var(\langle a, \theta_t \rangle)}}
        \le \gamma
        \right) - \Phi(\gamma)
        \bigg|
        \;\lesssim\;
        C_2 (d t^{-2\alpha})^{\frac{p}{8p + 4}},
    \end{align}
    \emph{where $\Phi(\cdot)$ denotes the CDF of the standard normal distribution.}

    \item[(2)] To make the bound in~\eqref{eq:informalmain} practical for inference, we propose a sub-sampling–based online estimator for the variance term, described in Section~\ref{sec:varest}. We show in Theorem~\ref{varestmainthm} that the additional estimation error is negligible. This yields the first fully data-driven, online framework for growing-dimensional algorithmic inference using stochastic optimization methods such as SGD, operating under the near-optimal scaling $t \gtrsim d^{1+\delta}$ for any $\delta > 0$.
\end{enumerate}

Our results are conceptually related to recent work on finite-sample normal approximation in high-dimensional regression, notably~\cite{chang2023inference}, which obtained Berry–Esseen bounds for projection parameters under general moment assumptions but required $t \gtrsim d^{3/2}$. In contrast, we achieve the same inferential objective under the significantly improved scaling $t \gtrsim d^{1+\delta}$ (by choosing $\alpha$ such that $\frac{1}{2} < \alpha < \frac{1 + 2\delta}{2 + 2\delta}$), without imposing stronger assumptions. Moreover, our approach is computationally and memory efficient, running in $\mathcal{O}(td)$ time and $\mathcal{O}(d)$ space, compared to $\mathcal{O}(td^2 + d^3)$ for covariance-inversion–based methods that require explicit matrix inversion. These theoretical and algorithmic advantages make our method scalable to substantially higher-dimensional regimes. Section~\ref{methodolgy-difference} provides a detailed discussion of the key methodological ingredients enabling this improved scaling.

\vspace{6pt}
Beyond providing a theoretical framework for growing-dimensional inference, our results have practical implications for constructing \emph{algorithmic prediction intervals} in linear regression. For a new test point $a$, independent of the training data used by SGD, choosing $a$ in~\eqref{eq:informalmain} directly yields a predictive confidence interval, complementing prior works on implicit regularization and benign overfitting~\cite{ali2020implicit, varre2021last, zou2021benign, wu2022last, chen2022dimension}. Furthermore, our results can be used to develop \emph{algorithmic Wald-type tests} for feature significance in high-dimensional linear models—an essential tool in empirical sciences such as biology, social science, economics, and medicine~\cite{fang2017testing, shi2019linear, cai2023statistical}. Specifically, testing the null hypothesis $H_0: \beta_i^* = 0$ for a particular feature corresponds to choosing $a = e_i$, the $i$th canonical basis vector in $\mathbb{R}^d$, within our framework, yielding an efficient, online, and statistically valid hypothesis test.

%To compute confidence intervals in practice from the above theorem, it is important to estimate the expectation and variance in~\eqref{eq:informalmain}. Based on an , in Section~\ref{sec:varest}, we also provide estimators of the expectation and variance terms and establish consistency results in the high-dimensional setting. In summary,  

\subsection{Related Works}

\textbf{SGD analyses.} A majority of the SGD analyses in the machine learning and the optimization literature has focused on establishing expectation or high-probability bounds in the fixed-dimensional setting. We refer the interested reader to the survey by~\cite{bottou2018optimization}, and books by~\cite{bubeck2015convex} and~\cite{lan2020first}, for a sampling of such results. There also exists almost sure convergence results for SGD; see~\cite{harold1997stochastic} for a survey of some classical works, and~\cite{mertikopoulos2020almost,sebbouh2021almost,liu2022almost} for some recent works. Recently, several works have looked at analyzing SGD in the growing-dimensional setup. For example,~\cite{paquette2021sgd, paquette2021dynamics, paquette2022homogenization,paquetteimplicit2022} studied mini-batch and online least-squares SGD under growing-dimensional scalings, using tools from random matrix theory. Growing-dimensional diffusion approximations have also been established for SGD in specific problems; see, for~\emph{e.g.},~\cite{wang2017scaling, tan2019online, arous2021online, aroushigh2022,balasubramanian2023high}. Such results extend the classical results~\cite{borkar2009stochastic,benveniste2012adaptive} to the growing-dimensional settings. \cite{chandrasekher2021sharp} study SGD in for certain growing-dimensional non-convex problems using Gaussian process techniques. Furthermore, statistical physics techniques are also used to understand the performance of SGD in growing-dimensions; see, for~\emph{e.g.},~\cite{celentano2021high,gerbelot2022rigorous}. Several of the above results do not study the fluctuations of  SGD. The few papers that establish fluctuation results for SGD do so only in the asymptotic setting. More importantly, none of the above papers focus on constructing online algorithms for obtaining practical confidence intervals. 

\textbf{Asymptotic SGD CLTs and inference.} Studying the asymptotic distribution of SGD goes back to the early works of~\cite{chung1954stochastic, sacks1958asymptotic, fabian1968asymptotic}; see also~\cite{shapiro1989asymptotic}. These works primarily studied the asymptotic distribution of the last iteration of the stochastic gradient algorithm. It was shown later in~\cite{ruppert1988efficient} and~\cite{polyak1992acceleration} that averaging the iterates of the stochastic gradient algorithm has acceleration benefits. This result has been recently extended to implicit stochastic gradient algorithms~\cite{toulis2017asymptotic}, Nesterov's dual averaging algorithm~\cite{duchi2016local}, proximal-point methods~\cite{asi2019stochastic} and Nesterov's accelerated algorithm~\cite{barakat2021stochastic}.  Furthermore,~\cite{dieuleveut2020bridging} and~\cite{yu2020analysis} established asymptotic normality of constant step-size stochastic gradient algorithm in the convex and nonconvex setting respectively. \cite{mou2020linear} examined the relationship between asymptotic CLTs and non-asymptotic expectation bounds in the context of linear regression. Very recently,~\cite{davis2023asymptotic} also extended the seminal result of~\cite{polyak1992acceleration} to non-smooth settings. 

Several works also considered the problem of estimating the asymptotic covariance matrix appearing in the central limit theorem. Towards that~\cite{su2018statistical, fang2018online, lunde2021bootstrapping, chee2022plus} proposed online bootstrap procedures. Furthermore,~\cite{zhu2021online,jin2021statistical} provided trajectory-averaging based online estimators motivated by multivariate time-series analysis. The ideas in the above works are inherently motivated by general methodology and theory on (inhomogenous) Markov chain variance estimation literature~\cite{glynn1991estimating,politis1999subsampling, flegal2010batch, kim2013progressive,lahiri2013resampling}. We also remark that~\cite{li2018approximate, li2018statistical,chen2016statistical,lee2022fast} developed semi-online procedures for covariance estimation. Recently~\cite{jiang2025online} developed methods to handle non-smooth stochastic objectives. We remark that the above works focus on the asymptotic setting, while our focus is on the growing-dimensional non-asymptotic setting.

\textbf{Non-asymptotic rates for SGD CLTs.}
Non-asymptotic rates for SGD CLTs in the smooth strongly-convex setting were derived  in~\cite{anastasiou2019normal}, based on deriving the rates of multivariate Martingale CLTs. ~\cite{shao2022berry} extended the above result to stronger metrics under further assumptions. Recent line of work have established tail-bounds (\cite{durmus2022finite}, \cite{durmus2021tight}) and non-asymptotic CLTs (\cite{samsonov2024gaussian}, \cite{Khusainov2025GaussianApprox}, \cite{WuLiWeiRinaldo2025InferenceTD}, \cite{Samsonov2025StatisticalInferenceLSA}) for SGD for the linear stochastic approximation (LSA) problem. We discuss the relationship between our result and the above mentioned works in Section~\ref{eq:beb} and Appendix \ref{prior-sgd-clt-comparison}.

\section{Growing-dimensional Central Limit Theorem for Online SGD}\label{sec:mainres}
In this section, we first state and discuss the assumptions we make in this work. We next discuss the Berry-Esseen bound on the linear functionals of least-squares SGD iterates in Theorem~\ref{th-1}. 
\subsection{Assumptions}
\begin{assumption}\label{assump} We make the following assumptions to state our main result. Note that all quantities that appear below (except absolute constants) can depend on $t,d$.
\begin{itemize}[leftmargin=0.3in]
\item[(i)] \textbf{Error Lower Bound.} Let $\epsilon:= Y - X^\top \beta^*$, $A := \mathbb{E}[XX^\top]$ and $A_{\sigma} := \mathbb{E}[\epsilon^2 XX^\top]$. There exists $\sigma_{\min} > 0$ such that
\begin{align*}
    \lambda_{\min}(A_{\sigma}) > \sigma^2_{\min} \lambda_{\min}(A),
\end{align*}
where for any positive-definite symmetric matrix $\mathcal A$, $\lambda_{\min}(\mathcal A)$ denotes it's minimum eigenvalue. 
\item[(ii)] \textbf{Error Moment Bound.} There exists an \textbf{absolute constant} $p_{\max} \geq 2$ such that the error $\epsilon := Y - X^\top \beta^*$ satisfies
\begin{align*}
\mathbb{E}[\epsilon^{4p_{\max}}] < \infty.
\end{align*}
Given this assumption, we let $\sigma := \mathbb{E}[\epsilon^{4p_{\max}}]^{\frac{1}{4p_{\max}}}$ throughout the paper. 

\item[(iii)] \textbf{Covariate Lower Bound.} Let $A := \mathbb{E}[XX^\top]$ and $\lambda_{\min}(A), \lambda_{\max}(A)$ denote the smallest and largest eigenvalues of $A$ respectively. We assume $A$ is non-degenerate, that is $\lambda_{\min} (A) > 0$.  
\item[(iv)] \textbf{Covariate Moment Bound.} There exists an \textbf{absolute constant} $p_{\max} \geq 2$ such that
\begin{align*}
\sup_{u \in \mathbb{R}^d, |u| = 1}\mathbb{E}[|u^\top X|^{4p_{\max}}] < \infty
\end{align*}
Given this assumption, we let 
\begin{align*}
\bar{\lambda} := \sup_{u \in \mathbb{R}^d, |u| = 1} \mathbb{E}[|u^\top X|^{4p_{\max}}]^{\frac{1}{2p_{\max}}}
\end{align*}
throughout the paper. In particular, observe that $\bar{\lambda} \geq \sup_{u \in \mathbb{R}^d, |u| = 1} \mathbb{E}[|u^\top X|^{2}] = \lambda_{\max}(A)$ (using Minkowski's inequality). 
\item[(v)] \textbf{Step-Size.} We assume the step-size $\eta_i$ is set to $\eta_i := \frac{\eta}{\sqrt{d}i^{\alpha}}$, where $\eta > 0$ and $\alpha \in (\frac{1}{2}, 1)$. Here $d$ is the dimension of the covariates, that is $X \in \mathbb{R}^d$.
\item[(vi)] \textbf{Bounded Error, Eigenvalue Decay and Moment, Parameter Growth Rates.} We make the following assumptions on the decay rates of $\sigma_{\min}, \lambda_{\min}(A)$ and the growth rate of $\bar{\lambda}, |\beta^* - \theta_0|$.   
\begin{itemize}
    \item $\eta \bar{\lambda} < C$ for an \textbf{absolute constant} $C > 0$.
    \item $\lim\limits_{t,d \to \infty} (\eta \lambda_{\min}(A))^{-1}(\log t + \log d)^2d^{\frac{1}{2}}t^{-(1 - \alpha)} = 0$. 
    \item $\lim\limits_{t,d \to \infty} (\eta \lambda_{\min}(A))^{-3} (\sigma^2\sigma^{-2}_{\min})(\log t + \log d)^2d^{\frac{1}{2}}t^{-\alpha} = 0$.
    \item There exists \textbf{absolute constants} $C_1, C_2 > 0$ such that $\frac{|\beta^* - \theta_0|^2}{\eta \sigma^2} < (td)^{C_1}$ for all $t, d \geq C_2$. 
\end{itemize}
\end{itemize}
\end{assumption}
\noindent 
\textbf{Comparison with prior assumption-lean works.} 
We compare our assumptions with those in the recent finite-sample Berry--Esseen analysis of projection-parameter inference by Chang, Kuchibhotla, and Rinaldo~\cite{chang2023inference} (see their Section 2.2). 

Our assumptions on the covariates and errors—positive definiteness of the population Gram matrix $A = \mathbb{E}[XX^\top]$, non-degenerate error covariance $\lambda_{\min}(\mathbb{E}[\epsilon^2 XX^\top]) \ge \sigma_{\min}^2 \lambda_{\min}(A)$, and finite higher-order directional moments—are essentially the same type of ``assumption-lean'' conditions used in~\cite{chang2023inference}. In particular, the moment bounds parametrized by $p_{\max}$ in our work can be chosen to match or exceed the moment exponents $q_x, q$ used in~\cite{chang2023inference} to achieve their $t \gtrsim d^{3/2}$ scaling. 

Assumption~\ref{assump}(vi) imposes several additional technical conditions on the decay/growth rates of $\eta\lambda_{\min}(A)$, $\eta\bar{\lambda}$, $\sigma/\sigma_{\min}$ and the initialization error $|\beta^* - \theta_0|$. These conditions are required to control the non-asymptotic behavior of the online SGD iterates and ensure that higher-order remainder terms remain negligible in our Berry--Esseen bounds. 

While some parts, such as the boundedness of $\eta\bar{\lambda}$ and controlled growth of the initialization error, place mild constraints on the covariate distribution and choice of initialization, these are generally realistic in practice. For instance, bounded $\eta\bar{\lambda}$ corresponds to assuming finite high-order directional moments of the covariates, which is comparable to the moment assumptions in~\cite{chang2023inference}. Similarly, conditions on $\lambda_{\min}(A), \sigma_{\min}, \sigma$ are mild regularity conditions also required in~\cite{chang2023inference} to avoid ill-conditioned problems. 

Importantly, despite having comparable distributional assumptions, our approach achieves a nearly optimal growing-dimensional scaling $t \gtrsim d^{1+\delta}$ for any $\delta>0$ (improving over $t \gtrsim d^{3/2}$ in~\cite{chang2023inference}) and provides a fully online algorithm with lower computational cost than covariance-matrix inversion. In summary, our distributional assumptions are comparable in strength to those in~\cite{chang2023inference}, and our main contribution is that we achieve growing-dimensional Berry--Esseen bounds for \emph{online SGD iterates} under nearly the same assumption-lean conditions, while simultaneously improving both the dimensional scaling and computational efficiency.

\subsection{Berry-Esseen Bounds for Linear Functionals of Least-squares SGD}\label{eq:beb}
Our first result shows a central limit theorem for linear functionals $\langle a, \theta_t \rangle$ of the least-squares SGD. Define
\begin{align*}
d_K := \sup_{\gamma \in \mathbb{R}} \bigg|\mathbb{P}\bigg(\frac{\langle a, \theta_t \rangle - \mathbb{E}\langle a, \theta_t \rangle}{\sqrt{\Var \langle a, \theta_t \rangle}}\leq \gamma\bigg) - \Phi(\gamma)\bigg|,    
\end{align*}
which is the quantity we wish to bound.
\begin{theorem}\label{th-1}
Under Assumption \ref{assump}, we have for all $t,d \geq C_1$, $2 \leq p \leq p_{\max}$ and $a \in \mathbb{R}^d$ that
\begin{align*}
    d_K\leq C_2(\eta \lambda_{\min}(A))^{-\frac{p}{2p+1}}\Big[\frac{\sigma}{\sigma_{\min}}\Big]^{\frac{2p}{2p+1}}\bigg[(\eta \lambda_{\min}(A))^{-\frac{3p}{4p+2}}(\log t + \log d)^{\frac{3p}{4p+2}}\Big(\frac{d}{t^{2\alpha}}\Big)^{\frac{p}{8p+4}} + \Big(\frac{t^{\frac{1}{p}-\alpha}}{\sqrt{d}}\Big)^{\frac{p}{2p+1}}\bigg].
\end{align*}
Here $C_1, C_2 > 0$ are absolute constants.
\end{theorem}
Our proof technique to obtain the above result is based one expressing $\langle a, \theta_t \rangle $ as a sum of certain martingale difference sequence. Based on the representation, one could leverage Berry-Esseen bounds developed for martingales~\cite{bolthausen1982exact,mourrat2013rate}. However, computing the quadratic variation and moment terms appearing in the Berry-Esseen bounds becomes highly non-trivial. We compute these by a careful application of Lemma \ref{uniformconvexity3}, which controls how the norm of a random variable changes if we add a zero mean fluctuation. The proof technique for Lemma \ref{uniformconvexity3} is heavily borrowed from \cite{huang2022matrix}, which proves a more general inequality for random matrices.   We prove Theorem \ref{th-1} in Appendix~\ref{sec:proofsketch}. 
\\\\
Our next results show that under the Assumptions \ref{assump}, the error encountered by replacing the biased center $\mathbb{E}\langle a, \theta_t \rangle$ with the true parameter $\langle a, \beta^* \rangle$ is negligible.
\begin{theorem}\label{bias-correction-term}
Under Assumption \ref{assump}, we have for all $t,d \geq C_1$ and $a \in \mathbb{R}^d$ that
\begin{align*}
\frac{|\mathbb{E}_{\theta_t}\langle a, \theta_t \rangle - \langle a, \beta^* \rangle |}{\sqrt{\Var_{\theta_t} \langle a, \theta_t \rangle}} \leq C_2(\eta \lambda_{\min}(A))^{-\frac{1}{2}}(e^{-\eta \lambda_{\min}(A)d^{-\frac{1}{2}}t^{1-\alpha}}d^{\frac{1}{2}}t^{\alpha})\bigg[\frac{|\beta^* - \theta_0|}{\sigma_{\min}\sqrt{\eta}}\bigg].
\end{align*} 
Here $C_1, C_2 > 0$ are absolute constants. 
\end{theorem}
We prove Theorem \ref{bias-correction-term} in Appendix~\ref{sec:bias-corrected-CLT}.
\\\\
Using these, we now provide our main result, which is a bias-corrected 
 high-dimensional central limit theorem for linear functionals of the least-squares SGD. Define
 \begin{align*}
     d^{true}_K := \sup_{\gamma \in \mathbb{R}} \bigg|\mathbb{P}\bigg(\frac{\langle a, \theta_t \rangle - \langle a, \beta^* \rangle}{\sqrt{\Var \langle a, \theta_t \rangle}}\leq \gamma\bigg) - \Phi(\gamma)\bigg| 
 \end{align*}

\begin{theorem}\label{th-1-bias-corrected}
Under Assumption \ref{assump}, we have for all $t,d \geq C_1,$ $2 \leq p \leq p_{\max}$ and $a \in \mathbb{R}^d$ that
\begin{align}\label{maineq}
     d^{true}_K\leq C_2(\eta \lambda_{\min}(A))^{-\frac{p}{2p+1}}\Big[\frac{\sigma}{\sigma_{\min}}\Big]^{\frac{2p}{2p+1}}\bigg[(\eta \lambda_{\min}(A))^{-\frac{3p}{4p+2}}(\log t + \log d)^{\frac{3p}{4p+2}}\Big(\frac{d}{t^{2\alpha}}\Big)^{\frac{p}{8p+4}} + \Big(\frac{t^{\frac{1}{p}-\alpha}}{\sqrt{d}}\Big)^{\frac{p}{2p+1}}\bigg].
\end{align}
Here $C_1, C_2 > 0$ are absolute constants. 
\end{theorem}
\begin{proof}
Throughout the proof, we let $C> 0$ and $c > 0$ respectively denote large and small enough generic absolute constants.
\\\\
Define $\Delta := \frac{\mathbb{E}_{\theta_t}\langle a, \theta_t \rangle - \langle a, \beta^* \rangle}{\sqrt{\Var_{\theta_t} \langle a, \theta_t \rangle}}$. For any $\gamma \in \mathbb{R}$, we have
\begin{align*}
\bigg|\mathbb{P}\bigg(\frac{\langle a, \theta_t \rangle - \langle a, \beta^* \rangle}{\sqrt{\Var_{\theta_t} \langle a, \theta_t \rangle}} \leq \gamma\bigg)& - \Phi(\gamma)\bigg| 
= \bigg|\mathbb{P}\bigg(\frac{\langle a, \theta_t \rangle - \mathbb{E}_{\theta_t}\langle a, \theta_t \rangle}{\sqrt{\Var_{\theta_t} \langle a, \theta_t \rangle}} + \Delta \leq \gamma\bigg) - \Phi(\gamma)\bigg|
\\
&= \bigg|\mathbb{P}\bigg(\frac{\langle a, \theta_t \rangle - \mathbb{E}_{\theta_t}\langle a, \theta_t \rangle}{\sqrt{\Var_{\theta_t} \langle a, \theta_t \rangle}} \leq \gamma - \Delta \bigg) - \Phi(\gamma - \Delta) + \Phi(\gamma - \Delta) - \Phi(\gamma)\bigg|
\\
&\leq \sup_{\gamma' \in \mathbb{R}} \bigg|\mathbb{P}\bigg(\frac{\langle a, \theta_t \rangle - \mathbb{E}_{\theta_t}\langle a, \theta_t \rangle}{\sqrt{\Var_{\theta_t} \langle a, \theta_t \rangle}} \leq \gamma'\bigg) - \Phi(\gamma')\bigg| + \sup_{\gamma' \in \mathbb{R}}|\Phi(\gamma' + |\Delta|) - \Phi(\gamma')|
\end{align*}
We can use Theorem ~\ref{th-1} to bound the first term. For the second term, observe for any $\gamma' \in \mathbb{R}$ that
\begin{align*}
\Phi(\gamma' + |\Delta|) - \Phi(\gamma') &= \int_{s = \gamma'}^{\gamma' + |\Delta|} \frac{e^{-\frac{s^2}{2}} ds}{\sqrt{2\pi}}
\\
&\leq \int_{s = \gamma'}^{\gamma' + |\Delta|} \frac{ds}{\sqrt{2\pi}}
\\
&= \frac{|\Delta|}{\sqrt{2\pi}}.
\end{align*}
These imply that $d^{true}_K \leq d_K + \frac{|\Delta|}{\sqrt{2\pi}}$. Define
\begin{align*}
    \mathbf{R}:= C(\eta \lambda_{\min}(A))^{-\frac{p}{2p+1}}[\sigma/\sigma_{\min}]^{\frac{2p}{2p+1}}[(\eta \lambda_{\min}(A))^{-\frac{3p}{4p+2}}(\log t + \log d)^{\frac{3p}{4p+2}}(dt^{-2\alpha})^{\frac{p}{8p+4}} + (d^{-\frac{1}{2}}t^{\frac{1}{p}-\alpha})^{\frac{p}{2p+1}}],
\end{align*}
and observe from Assumption \ref{assump} (vi) that
\begin{align*}
\mathbf{R} \geq (td)^{-C}
\end{align*}
for all $t,d \geq C$. But observe from Theorem \ref{bias-correction-term} and Assumption \ref{assump} (vi) that
\begin{align*}
|\Delta| &:= \frac{|\mathbb{E}_{\theta_t} \langle a, \theta_t \rangle - \langle a, \beta^* \rangle|}{\sqrt{\Var_{\theta_t} \langle a, \theta_t \rangle}}
\\
\\
&\leq \frac{Ce^{-\eta \lambda_{\min}(A)d^{-\frac{1}{2}}t^{1-\alpha}}(\eta \lambda_{\min}(A))^{-\frac{1}{2}}|\beta^* - \theta_0|}{\sigma_{\min}\sqrt{\eta}}
\\
&\leq e^{-c (\log t + \log d)^2}(td)^C,
\end{align*}
for all $t,d \geq C$. These imply that $|\Delta| \leq \mathbf{R}$ for all large enough $t,d$, which gives us the desired result.
\end{proof}
\noindent \textbf{Remark 1 (Allowed growth rate of $d$).} Suppose $\frac{\sigma^2}{\sigma^2_{\min}} < C$ and $\eta \lambda_{\min}(A) > c$ for absolute constants $C,c > 0$. Then it suffices to have $d^{\frac{1}{2}}t^{-(1-\alpha)} \to 0$ (in Assumption \ref{assump} (vi)), and we can see that the Berry Esseen rate in both Theorem \ref{th-1-bias-corrected} and it's data-driven version Theorem \ref{varestmainthm} go to $0$ in this regime. Thus, our rates are valid and go to $0$ as long as $t \geq d^{1+\delta}$ for any $\delta > 0$, by choosing $\alpha$ such that $\frac{1}{2} < \alpha < \frac{1 + 2\delta}{2 + 2\delta}$. Further, we also show in \textbf{Remark 5} that the width of confidence intervals constructed by our procedure decay to $0$ under these assumptions. This enables finite-sample inference for linear regression projection parameters under much faster dimension growth $(t \gtrsim (d^{1+\delta}))$ compared to the previous scaling of $t \gtrsim d^{3/2}$ needed in \cite{chang2023inference}, while making similar minimal assumptions on the data generating process (and also being more computationally efficient). 
\\\\
\noindent \textbf{Remark 2 (Dependence on $p$).} We suppressed the dependence of $C_1$ and $C_2$ in Theorem \ref{th-1-bias-corrected} on $p$ by assuming that $p_{\max}$ in Assumption \ref{assump} is an absolute constant. Carefully tracking the dependence gives us that while $C_1$ is independent of $p$, $C_2$ can grow as $e^{p^K}$ for some absolute constant $K > 0$. Thus, choosing a higher value of moment $p$ can give better asymptotic behaviour of the CLT error, at the cost of much bigger constants. Choosing the value of $p$ optimally for a given finite $t,d$ is left as interesting future work.   
\\\\
\noindent \textbf{Remark 3 (Comparison to Existing Results).} We now place our growing-dimensional SGD CLT in the context of the broader literature:

Existing non-asymptotic normal approximation results for SGD include~\cite{anastasiou2019normal,shao2022berry,durmus2022finite, durmus2021tight, samsonov2024gaussian, Khusainov2025GaussianApprox, WuLiWeiRinaldo2025InferenceTD, Samsonov2025StatisticalInferenceLSA}. While these works provide explicit Berry--Esseen bounds for smooth, strongly convex problems, they are either restricted to low-dimensional regimes (e.g., $d = o(t^{1/4})$ or $o(t^{1/2})$) or rely on independence or well-behaved conditional variance assumptions on the SGD noise. Consequently, the results in these prior works do not directly apply to our setup (see Appendix~\ref{prior-sgd-clt-comparison} for a detailed discussion).

Thus, while focusing specifically on linear regression, our result allows $t \gtrsim d^{1+\delta}$ for any $\delta > 0$, and provides explicit rates for linear functionals of online least-squares SGD iterates under assumption-lean moment conditions. To the best of our knowledge, no prior work handles online least-squares SGD under minimal moment assumptions in growing-dimensional scaling regimes such as ours.

%\end{itemize}
\section{Online Variance Estimation}\label{sec:varest}
Theorem~\ref{th-1-bias-corrected} shows that $(\langle a,\theta_t \rangle - \langle a, \beta^* \rangle)/ \sqrt{\mathrm{Var}(\langle a,\theta_t \rangle)}$ converges in distribution to standard normal distribution, with the explicit rate provided. In order to obtain practical confidence intervals based on Theorem~\ref{th-1-bias-corrected}, we need an estimate for  $\mathrm{Var}(\langle a,\theta_t \rangle)$. Towards that, we now discuss an online procedure for estimating the variance terms appearing in the CLT. Our approach has some resemblance to the larger literature~\cite{politis1999subsampling,lahiri2013resampling} on variance estimation with dependent data as the SGD iterate in~\eqref{eq:ithiterate} is inherently an inhomogenous Markov chain. However, the specific details of our methodology and our theoretical analysis are motivated by the growing-dimensional regime that we consider. 
\\\\
For variance estimation (Theorem \ref{varestmainthm}), we assume, in addition to assumptions \ref{assump}, a mild spectral-regularity condition. 
\begin{assumption}\label{assump-varest}
\textbf{(Lower Bounded Minimum Eigenvalue).} Let $A := \mathbb{E}[XX^\top]$. We assume that the minimum eigenvalue satisfies 
\begin{align*}
\eta \lambda_{\min}(A) > c,
\end{align*} 
for an absolute constant $c > 0$. 
\end{assumption}
This assumption simplifies the choice of the cutoff parameter $t_0$ defined below; it could be relaxed at the cost of a more intricate definition of $t_0$.
\\\\
\textbf{Definition  Of The Variance Estimator.}
Let
\begin{align*}
u_{i_1,i_2} := [\prod_{j = i_1}^{i_2}(I - \eta_{t-i_2 + j}X_{j}X^\top_{j})]a, \quad t_0 := t^{\alpha}d^{\frac{1}{2}}(\log t + \log d)^2    
\end{align*}
Theorem \ref{lem:VarAsymp} and Lemma \ref{variance-sigma-term-2} together imply
\begin{align*}
    \Var \langle a, \theta_t \rangle \approx \mathbb{E}[\mathbf{V}], \quad \mathbf{V} := \sum_{i = t-t_0 + 1}^{t} \eta^2_i (Y_i - X^\top_i \beta^*)^2(u^\top_{i+1,t} X_i)^2  .
\end{align*}
Crucially $\mathbf{V}$ only on the most recent $t_0$ data points $\{(X_{t-t_0 + 1}, Y_{t-t_0 + 1}), \dots, (X_t, Y_t)\}$. Hence the entire data stream can be partitoned into approximately $t/t_0$ i.i.d blocks, each providing an independent copy of $\mathbf{V}$. Averaging these blocks should then give a tight estimator for $\Var \langle a, \theta_t \rangle$. 
\\\\
Because $\mathbf{V}$ involves the unknown $\beta^*$, we substitute the halfway SGD iterate $\theta_{\frac{t}{2}}$ as a plugin estimate, obtained from the first half of the data. For block $k = 1,2 \dots ,\frac{t}{2t_0}$ (in the second half of the stream), define
\begin{align*}
    \hat{\mathbf{V}}_{k} := \sum_{i = s_k}^{s_k + t_0 - 1}  \eta^2_{i + \frac{t}{2} - kt_0} (Y_i - X^\top_i \theta_{\frac{t}{2}})^2 (u^\top_{i+1, s_k + t_0 - 1}X_i)^2,
\end{align*}
where $s_k := t/2 + (k-1)t_0 + 1$. 
The online variance estimator $\hat{V}_t$ is then
\begin{align*}
    \hat{V_t} := \frac{2t_0}{t}\sum_{k = 1}^{\frac{t}{2t_0}} \hat{\mathbf{V}}_k.
\end{align*}
\begin{theorem}\label{varestmainthm}
Assume Assumptions \ref{assump} and \ref{assump-varest}. 
For sufficiently large $t,d \geq C$ (absolute constant $C > 0$),
\begin{itemize}
    \item \textbf{Relative-error bound.} We have that,
\begin{align*}
    \mathbb{E}\bigg|\frac{\hat{V_t} - \Var \langle a, \theta_t \rangle}{\Var \langle a, \theta_t \rangle}\bigg| \leq C(\sigma^2/\sigma^2_{\min})(\log t + \log d)^3d^{\frac{1}{4}}t^{-\frac{(1 - \alpha)}{2}}.
     \end{align*}
     \item \textbf{Distributional accuracy.} Define,
\begin{align*}
    \omega :=  \omega(t,d) = (\sigma/\sigma_{\min})(\log t + \log d)^{\frac{3}{2}}d^{\frac{1}{8}}t^{-\frac{(1 - \alpha)}{4}}
\end{align*}
and 
\begin{align*}
    \hat{d}^{true}_{K} := \sup_{\gamma \in \mathbb{R}} \bigg|\mathbb{P}\bigg(\frac{\langle a, \theta_t \rangle - \langle a, \beta^* \rangle}{\sqrt{\hat{V}_t}}\leq \gamma\bigg) - \Phi(\gamma)\bigg|.
\end{align*}
Then with probability at least $1 - C\omega$ that
\begin{align*}
\hat{d}^{true}_{K} \leq d^{true}_{K} + C\omega
\end{align*}
where $d^{true}_{K}$ is the Kolmogorov distance appearing in Theorem \ref{th-1-bias-corrected}. 
\end{itemize}
\end{theorem}
The above theorem shows that the error incurred by approximating the true variance in (\ref{maineq}), with the proposed online estimation procedure is negligible. Furthermore, as we show in \textbf{Remark 4}, the overall end-to-end procedure is fully online, i.e., requiring only a single-pass over the data, thereby maintaining the advantage of SGD. The CLT result in~Theorem~\ref{th-1} and its data-driven version~Theorem~\ref{varestmainthm} together provide a theoretically principled end-to-end statistical methodology for performing growing-dimensional statistical inference with the online SGD algorithm in growing-dimensional linear regression  models. 
\\\\
\noindent \textbf{Remark 4 (Online Construction of $\hat{V}_t$).} We now show explicitly that $\hat{V}_t$ above can be constructed with $O(td)$ time and $O(d)$ memory. 
\\\\
Consider the \textbf{last block} (indices $i = t-t_0 + 1, \dots, t$), whose contribution is
\begin{align*}
    \hat{\mathbf{V}}_{last} := \sum_{i = t-t_0 + 1}^t \eta^2_{i} (Y_i - X^\top_i \theta_{t/2})^2(u^\top_{i+1,t}X_i)^2
\end{align*}
To compute this efficiently in a single pass, we process the block \emph{backward in time}, i.e.\ from $i = t$ down to $i = t - t_0 + 1$. 
Since the $\{X_i\}$ are i.i.d., the samples within the block are exchangeable; therefore, processing them backward in time (or in any arbitrary order) yields the same distributional result and does not affect correctness.
\\\\
Observe that the sequence of row vectors $u^\top_{i+1, t}$ from $i = t$ down to $i = t - t_0 + 1$ satisfies $u^\top_{t+1, t} = a^\top$ and the simple recursion
\begin{align*} 
u^\top_{i,t}  &=  u^\top_{i+1,t}(I -   \eta_i X_i X^\top_i)
\\
&= u^\top_{i+1,t} - \eta_i (u^\top_{i+1,t} X_i)X^\top_i.
\end{align*}
Thus, we initialize 
\begin{align*}
\hat{\mathbf{V}}_{last} \leftarrow 0, u^\top_{t+1, t} \leftarrow a^\top,
\end{align*}
and for each step from $i = t$ down to $i = t-t_0 + 1$, we perform the following updates:
\begin{enumerate}
    \item Compute the scalar $s_i = u^\top_{i+1,t} X_i$ ;
    \item Update the variance sum
    \[
    \hat{\mathbf{V}}_{last} \mathrel{+}= \eta_i^2 
        \bigl(Y_i - X_i^\top \theta_{t/2}\bigr)^2 s_i^2;
    \]
    \item Update the vector
    \[
    u^\top_{i,t} = u^\top_{i+1,t} - \eta_i s_i X_i^\top.
    \]
\end{enumerate}
This procedure requires storing only the current $u^\top_{i+1,t}$ (a vector in $\mathbb{R}^d$) and a few scalar quantities, giving a total memory cost of $O(d)$. 
Since each iteration costs $O(d)$ time and there are $t$ samples, the overall complexity is $O(td)$.
\\\\
Because the data $\{(X_i, Y_i)\}$ are i.i.d., each block’s contribution has the same distributional law as the last block computed above. 
Consequently, the backward update scheme applies identically to every block, and processing the data in reverse order (or any order within each block) does not affect correctness. 
Therefore, the full estimator $\hat{\mathbf{V}}_{last}$ can be evaluated online in $O(td)$ time and $O(d)$ memory, as claimed.
\\\\
\textbf{Modified construction when $t$ is not known in advance:} The constructions above assumed that the total number of samples $t$ is known in advance. 
This assumption can be relaxed by using a dyadic batching strategy. 
Specifically, for each integer $n \ge 1$, use the samples 
\[
\{(X_i, Y_i)\}_{i = 2^n}^{2^{n+1}-1}
\]
to compute an estimate $\hat{V}_{2^n}$ of $\Var \langle a, \theta_{2^n} \rangle$, 
following the same procedure as in the known-$t$ case. 
\\\\
Now, if the actual number of available samples $t$ satisfies 
$2^{m+1} \le t < 2^{m+2}$ for some $m \ge 1$, 
we can use the variance scaling result from Theorem~\ref{lem:VarAsymp} 
to construct an estimator for $\Var \langle a, \theta_t \rangle$ as
\begin{align*}
    \hat{V}_t = \hat{V}_{2^m}\,(2^m / t)^{\alpha}.
\end{align*}
Since $t / 2^m \le 4$, this rescaling affects the variance only by a constant factor, 
and hence the estimator $\hat{V}_t$ inherits the same asymptotic guarantees 
as those established in Theorem~\ref{varestmainthm}, up to multiplicative constants.
\\\\
As discussed before, the main observation behind proving Theorem~\ref{varestmainthm} are the following observations about the variance itself.  
\begin{theorem}\label{lem:VarAsymp}
Recall that $A_{\sigma} := \mathbb{E}[\epsilon^2 XX^\top]$ and $A := \mathbb{E}[XX^\top]$, let $\mathbf{e}_1, \mathbf{e_2}, \dots \mathbf{e}_d$ be an eigen-basis of $A$ with corresponding eigen-values $\lambda_1 \geq \lambda_2 \geq \dots \lambda_d > 0$. Finally for all $1 \leq k, k' \leq d$, let $a_k := \langle \mathbf{e}_k, a \rangle$ and $[A_{\sigma}]_{k,k'} := \langle  \mathbf{e}_k, A_\sigma \mathbf{e}_{k'}\rangle$ denote the respective components of $a$ and $A_{\sigma}$ in the above basis.
\\\\
Under Assumption \ref{assump}, we have for all $t, d \geq C_1$ that
\begin{align*}
\Var \langle a, \theta_t \rangle = (1 + \mathcal E)\eta d^{-\frac{1}{2}}t^{-\alpha} \sum_{k,k'=1}^d \frac{a_{k}a_{k'}[A_{\sigma}]_{k,k'}}{\lambda_k + \lambda_{k'}},
\end{align*}
where $|\mathcal E| \leq C_2(\log t + \log d)^2 [(\eta \lambda_{\min}(A))^{-1}d^{\frac{1}{2}}t^{-(1-\alpha)} + (\eta \lambda_{\min}(A))^{-3}\sigma^2\sigma_{\min}^{-2}d^{\frac{1}{2}}t^{-\alpha}]$. Here $C_1, C_2 > 0$ are absolute constants.
\end{theorem}
\begin{lemma}\label{variance-sigma-term-2}
Recall that $\epsilon := Y - X^\top \beta^*, A_{\sigma} := \mathbb{E}[\epsilon^2 XX^\top]$ and $A := \mathbb{E}[XX^\top]$. 
Further, let
\begin{align*}
R_i:= \prod_{j = i+1}^{t}(I - \eta_{j}X_{j}X^\top_{j}),  \quad u_{i+1,t} := R_ia.
\end{align*}
Let $\mathbf{e}_1, \mathbf{e_2}, \dots \mathbf{e}_d$ be an eigen-basis of $A$ with corresponding eigen-values $\lambda_1 \geq \lambda_2 \geq \dots \lambda_d > 0$. Finally for all $1 \leq k, k' \leq d$, let $a_k := \langle \mathbf{e}_k, a \rangle$ and $[A_{\sigma}]_{k,k'} := \langle  \mathbf{e}_k, A_\sigma \mathbf{e}_{k'}\rangle$ denote the respective components of $a$ and $A_{\sigma}$ in the above basis.
\\\\
Assume Assumptions \ref{assump} and \ref{assump-varest}. Let $t_0:= t^{\alpha}d^{\frac{1}{2}}(\log t + \log d)^2$. We then have for all $t,d \geq C_1$ that
\begin{align*}
\mathbb{E}\bigg[\sum_{i = t-t_0 + 1}^t \eta_i^2 [(Y_i - X^\top_i \beta^*)^2(u^\top_{i+1,t} X_i)^2] \bigg] = (1+\mathcal E)\eta d^{-\frac{1}{2}}t^{-\alpha}\sum_{k,k' = 1}^d \frac{a_k a_{k'}[A_{\sigma}]_{k,k'}}{\lambda_{k} + \lambda_{k'}},
\end{align*}
where $|\mathcal E| \leq C_2(\log t + \log d)^2 [d^{\frac{1}{2}}t^{-(1-\alpha)} + \sigma^2\sigma^{-2}_{\min}d^{\frac{1}{2}}t^{-\alpha}]$. Here $C_1, C_2 > 0$ are absolute constants. 
\end{lemma}
Proofs for Theorem \ref{varestmainthm}, Theorem \ref{lem:VarAsymp} and Lemma \ref{variance-sigma-term-2} appear in Appendix \ref{varestproof}.
\\\\
\noindent \textbf{Remark 5 (Width Of Confidence Interval).} Recall from Theorem \ref{varestmainthm} that $$\omega := \omega(t,d) = (\sigma/\sigma_{\min})(\log t + \log d)^{3/2}d^{1/8}t^{-\frac{1-\alpha}{4}}.$$
Theorem \ref{varestmainthm} and \ref{lem:VarAsymp} tell us that with probability at least $1 - C\omega$, the width of the confidence interval constructed by our procedure is smaller than $$C(\sigma \sqrt{\eta})|a|d^{-\frac{1}{4}}t^{-\frac{\alpha}{2}},$$ which goes to $0$ as $t,d \to \infty$.
\\\\
Suppose Assumption \ref{assump}, \ref{assump-varest} and $\sigma/\sigma_{\min} < C$ for an absolute constant $C > 0$. Under these, $\omega := \omega (t, d)$ goes to $0$ and our method enables construction of tight, non-asymptotic confidence intervals for the projection parameter $\langle a, \beta^* \rangle$ in the near-optimal dimensional scaling regime $t \gtrsim d^{1+\delta}$, by choosing $\alpha$ such that $\frac{1}{2} < \alpha < \frac{1 + 2\delta}{2 + 2\delta}$. Further, it requires only $O(d)$ memory, $O(td)$ time and a single pass over the data.
\section{Conclusion}
We established a growing-dimensional central limit theorem (in the form of a Berry-Esseen bound) for linear functionals of online  SGD iterates for the growing-dimensional, assumption-lean linear regression model. We also provide data-driven and fully-online estimators of the variance terms appearing in the central limit theorem and establish rates of convergence results in the growing-dimensional setting. Our contributions in this paper makes the first concrete step towards growing-dimensional online statistical inference with stochastic optimization algorithms under the near optimal scaling of $t \gtrsim d^{1+\delta}$. 
\\\\
It is also of great interest to extend the analysis to 
\begin{itemize}[noitemsep,leftmargin=0.2in]
    \item quadratic functionals of online least-squares SGD iterates: Note, that in this case, we should seek for chi-square approximation rates; recent results, for example~\cite{gaunt2017chi}, might be leveraged.
    \item relatively tamer non-convex problems like phase retrieval and matrix sensing.
    \item  growing-dimensional robust regression problems, with the main complication being handling the subtleties arising due to non-smoothness~\citep{jiang2025online}.
\end{itemize}   
We hope that our work will attract future research aimed at addressing these important problems.
%\clearpage

%\section{Proof Sketch}

% We also use the $300$ samples collected using $b = 1$ above to numerically verify the observation about variance from Theorem \ref{lem:VarAsymp}. We observe that their sample variance differs from the theoretical leading order term in Theorem \ref{lem:VarAsymp} by only around less than $10$ percent.
\bibliography{references}
\bibliographystyle{amsplain}

\appendix

\section{Proof steps for growing-dimensional SGD CLT}\label{sec:proofsketch}
We now state the main steps in the proof of Theorem~\ref{th-1}. Before we proceed, we re-emphasize that a naive application of (non-asymptotic) delta method based on results from~\cite{anastasiou2019normal} or~\cite{shao2022berry} would only result in a relatively low-dimensional result. 

\subsection*{\textcolor{purple}{Step 1}: Expressing $\langle a, \theta_t\rangle$
as Martingale Difference Sequence}

The first step in our proof consists of expressing $\langle a, \theta_t\rangle$ as a martingale difference sequence. To do so, we have the following result providing an alternative representation of the SGD iterates. 

\begin{lemma}\label{prop:iterate}
Let $\epsilon_i := Y_i - X^\top_i \beta^*$ for all $1 \leq i \leq t$. The $i^{\textrm{th}}$ least-squares online SGD iterate in~\eqref{eq:ithiterate} is given by:
\begin{align*}
       \theta_i = \bigg(\prod_{j = 0}^{i-1}(I - \eta_{i-j}X_{i-j}X^\top_{i-j})\bigg)\theta_0 + \sum_{j = 1}^{i}\eta_j\bigg(\prod_{k = 0}^{i-j-1}(I - \eta_{i-k}X_{i-k}X_{i-k}^\top)\bigg)X_j(X_j^\top \beta^* + \epsilon_j).
\end{align*}
In particular, the $t^{\textrm{th}}$ iterate (i.e., last iterate) is given by
\begin{align*}
    \theta_t = \bigg(\prod_{i = 0}^{t-1}(I - \eta_{t-i}X_{t-i}X^\top_{t-i})\bigg)\theta_0 + \sum_{i = 1}^{t}\eta_i\bigg(\prod_{k = 0}^{t-i-1}(I - \eta_{t-k}X_{t-k}X_{t-k}^\top)\bigg)X_i(X_i^\top \beta^* + \epsilon_i)
\end{align*}
\end{lemma}

Based on the above result, we construct our martingale difference sequence as follows. For all $1 \leq i \leq t$, define 
\begin{align*}
    M_{i} = \mathbb{E}(\langle a, \theta_t \rangle  | X_t, Y_t,X_{t-1}, Y_{t-1},\dots X_{t-i+1}, Y_{t-i+1})-\mathbb{E}(\langle a, \theta_t \rangle  | X_t,Y_t,X_{t-1}, Y_{t-1},\dots X_{t-i+2}, Y_{t-i+2}).
\end{align*}
Further, let $\mathfrak{F}_{i-1}$ be the $\sigma$-field generated by $\{X_t,Y_t, X_{t-1}, Y_{t-1},\dots X_{t-i+2}, Y_{t-i+2}\}$ for all $1 \leq i \leq t$. Then it is easy to see that $(M_i)_{1\leq i\leq t}$ is a martingale w.r.t. the filtration $(\mathfrak{F}_{i-1})_{1\leq i\leq t}$. This is because
\begin{align*}
\mathbb{E}[M_i | \mathfrak{F}_{i-1}] &=  \mathbb{E}[\mathbb{E}[\langle a, \theta_t \rangle| \mathfrak F_i] |  \mathfrak F_{i-1}] - \mathbb{E}[\mathbb{E}[\langle a, \theta_t \rangle| \mathfrak F_{i-1}] |  \mathfrak F_{i-1}]
\\
&= \mathbb{E}[\langle a, \theta_t \rangle| \mathfrak F_{i-1}] - \mathbb{E}[\langle a, \theta_t \rangle| \mathfrak F_{i-1}]
\\
&= 0,
\end{align*}
where the second inequality follows because $\mathfrak{F}_{i-1} \subseteq \mathfrak{F}_i$.
In the following lemma, we formally write $\langle a, \theta_t \rangle$ in terms of this martingale.

\begin{lemma}\label{lemma:martingalerep}
We have
\begin{align*}
    \langle a, \theta_t \rangle - \mathbb{E}(\langle a, \theta_t \rangle) = \sum_{i=1}^{t} M_i.
\end{align*}
Furthermore, for all $1 \leq i \leq t$,
\begin{align*}
    M_{t - i + 1} := &\bigg \langle a, \eta_i\bigg(\prod_{j = 0}^{t-i-1}(I - \eta_{t-j}X_{t-j}X^\top_{t-j})\bigg)(X_iX_i^\top - A) \bigg(\prod_{j = 1}^{i-1}(I - \eta_{i-j}A)\bigg)(\beta^* - \theta_0) \nonumber\\
    &+\epsilon_i \eta_i\bigg(\prod_{j = 0}^{t-i-1}(I - \eta_{t-j}X_{t-j}X^\top_{t-j})\bigg) X_i \bigg \rangle. 
\end{align*}
\end{lemma}

\subsection*{\textcolor{purple}{Step 2}: Applying the Martingale CLT}
The above representation, enables us to leverage Berry-Esseen bounds developed for one-dimensional martingale difference sequences. For a square integrable martingale difference sequence $\mathbf{M} = (M_1, M_2,\dots M_t)$, let
\begin{align*}
  S(\mathbf{M}) := \sum_{i = 1}^{t}M_i,   \qquad  s^2(\mathbf{M}) := \sum_{i = 1}^{t}\mathbb{E}(M_i^2), \qquad
    V^2(\mathbf{M}) := s^{-2}(\mathbf{M})\sum_{i = 1}^{t}\mathbb{E}\big(M_i^2|\mathfrak{F}_{i-1}\big).
\end{align*}
For a random variable $U$, let $\|U\|_p := \mathbb{E}[|U|^p]^{\frac{1}{p}}$. Then, we have the following well-known result. 
\begin{theorem}[\cite{haeusler1988rate}]\label{ppn:Mourrat} 
Fix some $p \geq 1$. There exists $C_p > 0$ such that \\
\begin{align}\label{eq:rhsmartingale}
    D(\mathbf{M}) \leq C_p\bigg(\|V^2(\mathbf{M}) - 1\|_p^p + s^{-2p}(\mathbf{M})\sum_{i = 1}^{t}\|M_i\|_{2p}^{2p}\bigg)^{\frac{1}{2p+1}},
\end{align}
where
\begin{align*}
    D(\mathbf{M})\coloneqq \sup_{\kappa \in \mathbb{R}}|\mathbb{P}(S(\mathbf{M})/s(\mathbf{M}) \leq \kappa) - \Phi(\kappa)|.
\end{align*}
\end{theorem}
\noindent To apply Theorem~\ref{ppn:Mourrat} to our setting, first observe that
if $i < j$, then $$M_i M_j = [f(X_{t-j+2}, \epsilon_{t-j+2}, \dots X_t, \epsilon_t)]^\top(\epsilon_{t-j+1}X_{t-j+1})$$ for some function $f$. Using the later $X_k's$ are independent of $X_{t-j+1}, \epsilon_{t-j+1}$ and the standard fact that $$\mathbb{E}[\epsilon_{t-j+1}X_{t-j+1}] = \mathbb{E}[(Y-X^\top \beta^*)X] = 0, $$we have
\begin{align*}
   \mathbb{E}(M_iM_j) = 0 \qquad \forall i \not =  j.
\end{align*}
Thus
\begin{align*}
    s^2(\mathbf{M}) &:= \sum_{i = 1}^t\mathbb{E}[M^2_i]
    \\
    &=\mathbb{E}\bigg(\sum_{i = 1}^tM_i\bigg)^2 
    \\
    &= \mathbb{E}[\langle a, \theta_t\rangle - \mathbb{E}\langle a, \theta_t \rangle]^2
    \\
    &= \Var(\langle a, \theta_t \rangle).
\end{align*}
This immediately implies that
\begin{align*}
    D(\textbf{M}) = \sup_{\gamma \in \mathbb{R}}\bigg|\mathbb{P}\bigg(\frac{\langle a, \theta_t \rangle - \mathbb{E}\langle a, \theta_t \rangle}{\sqrt{\Var\langle a, \theta_t \rangle}}\leq \gamma\bigg) - \Phi(\gamma)\bigg|,
\end{align*}
which is the quantity that we wish to upper bound. 

\subsection*{\texorpdfstring{\textcolor{purple}{Step 3}: Alternative representation for the RHS of~\eqref{eq:rhsmartingale}}{Step 3: Alternative representation for the RHS of Eq.~(\ref{eq:rhsmartingale})}}

%\subsection*{\textcolor{purple}{Step 3}: Alternative representation for the RHS of~\eqref{eq:rhsmartingale}}

To derive the required result, it remains to compute the quadratic variance and moment terms, 
\begin{align}\label{eq:VandS}
\|V^2(\mathbf{M}) - 1\|_p^p\qquad~\text{and}~\qquad~s^{-2p}(\mathbf{M})\sum_{i = 1}^t\|M_i\|_{2p}^{2p},
\end{align}
appearing in the right hand side of~\eqref{eq:rhsmartingale}. To proceed, we introduce the following notations that will be used to state our results. We define 

\begin{itemize}\label{notation}
    \item $R_i \coloneqq \prod_{j = i+1}^{t}(I - \eta_{j}X_{j}X_{j}^\top)$ and, $S_i \coloneqq \prod_{j = 1}^{i-1}(I - \eta_{i-j}A)$
    \item $u_i \coloneqq R_ia$ and, $v_i \coloneqq S_i(\beta^* - \theta_0)$
    \item $\mathcal A_i := \mathbb{E}[(X_iX_i^\top - A)v_i v_i^\top (X_iX_i^\top - A) + \epsilon^2_iX_iX_i^\top + \epsilon_i X_iv^\top_i(X_iX_i^\top - A) + \epsilon_i (X_iX^\top_i - A)v_i X^\top_i]$.
\end{itemize}
Observe that in the preceding definition, all quantities are deterministic except for $R_i$ and $u_i := R_i a$. The matrix $R_i$, being a product of random matrices, requires careful analysis; in particular, deriving moment and concentration bounds for $R_i a$ constitutes a central component of our proof (see Section~\ref{R_i}). With this notation established, we now provide alternative representations for the terms in \eqref{eq:VandS}.
\begin{lemma}\label{lem:VandS}
We have that
\begin{align*}
    V^2(\mathbf{M}) - 1 =&\frac{\sum_{i = 1}^t \eta_i^2 (u^\top_i \mathcal A_i u_i - \mathbb{E}[u^\top_i \mathcal A_i u_i]) }{\sum_{i = 1}^t \eta_i^2 \mathbb{E}\langle u_i, (X_iX_i^\top - A)v_i + \epsilon_iX_i \rangle^2},
\end{align*}
and
\begin{align*}
    s^{-2p}(\mathbf{M})\sum_{i=1}^{t}\|M_i\|_{2p}^{2p} =\frac{\sum_{i = 1}^t \eta_i^{2p} \mathbb{E}\langle u_i, (X_iX_i^\top - A)v_i + \epsilon_iX_i \rangle^{2p}}{(\sum_{i = 1}^t \eta_i^2 \mathbb{E}\langle u_i, (X_iX_i^\top - A)v_i + \epsilon_iX_i \rangle^2)^p}.
\end{align*}
\end{lemma}

%\subsection*{\textcolor{purple}{Step 4}: Bounding the RHS of ~\eqref{eq:rhsmartingale}} 

\subsection*{\texorpdfstring{\textcolor{purple}{Step 4}: Bounding the RHS of~\eqref{eq:rhsmartingale}}{Step 4: Bounding the RHS of Eq.~(\ref{eq:rhsmartingale})}}

Based on the above representation, we have the following results that provide upper bounds on $\|V^2(\mathbf{M}) - 1\|_p^p$ and $s^{-2p}(\mathbf{M})\sum_{i = 1}^t\|M_i\|^{2p}_{2p}.$
\begin{theorem}\label{thm:BoundFirstTerm} Recall the assumptions \ref{assump} on $X, \epsilon$ and the step-size $\eta_i$. Under these assumptions, we have that
\begin{align*}
    \|V^2(\mathbf{M}) - 1\|_p^p &\leq C^p[\sigma^2/\sigma^2_{\min}]^p (\eta \lambda_{\min}(A))^{-\frac{5p}{2}}(\log t + \log d)^{\frac{3p}{2}}t^{-\frac{p\alpha}{2}}d^{\frac{p}{4}}, 
\end{align*} 
for all $t,d \geq C$ and $2 \leq p \leq p_{\max}$. Here $C > 0$ represents a generic absolute constant.
\end{theorem}
\begin{theorem}\label{thm:BoundSecondTerm} Recall the assumptions \ref{assump} on $X, \epsilon$ and the step-size $\eta_i$. Under these assumptions, we have that
\begin{align*}
s^{-2p}(\mathbf{M})\sum_{i = 1}^t\|M_i\|^{2p}_{2p} \leq C^p (\eta \lambda_{\min}(A))^{-p}[\sigma^2/\sigma^2_{\min}]^p d^{-\frac{p}{2}}t^{1-p\alpha} ,
\end{align*}
for all $t,d \geq C$ and $2 \leq p \leq p_{\max}$. Here $C > 0$ represents a generic absolute constant.
\end{theorem}

\subsection*{\textcolor{purple}{Step 5}: Completing the proof}
We now have all the ingredients required to prove our main result.
\vspace{0.1in}

\begin{proof}[Proof of Theorem~\ref{th-1}]
    To prove Theorem~\ref{th-1}, we need Theorem~\ref{ppn:Mourrat} and the fact that $(x + y)^{\frac{1}{n}} < x^{\frac{1}{n}} + y^{\frac{1}{n}}$ for all $x,y > 0$ and $n > 1$.  Using these and the bounds from Theorem~\ref{thm:BoundFirstTerm} and~\ref{thm:BoundSecondTerm}, we get for all $t,d \geq C$ and $2 \leq p \leq p_{\max}$ that
\begin{align*}
D(\mathbf{M})
&\leq C(\|V^2(\mathbf{M}) - 1\|^p_p + s^{-2p}(\mathbf{M})\sum_{i = 1}^t \|M_i\|^{2p}_{2p})^{\frac{1}{2p+1}}
\\
&\leq C[\|V^2(\mathbf{M}) - 1\|^{\frac{p}{2p+1}}_p + [s^{-2p}(\mathbf{M})\sum_{i = 1}^t \|M_i\|^{2p}_{2p}]^{\frac{1}{2p+1}}]
\\
&\leq C(\eta \lambda_{\min}(A))^{-\frac{p}{2p+1}}[\sigma/\sigma_{\min}]^{\frac{2p}{2p+1}}[(\eta \lambda_{\min}(A))^{-\frac{3p}{4p+2}}(\log t + \log d)^{\frac{3p}{4p+2}}(dt^{-2\alpha})^{\frac{p}{8p+4}} + (d^{-\frac{1}{2}}t^{\frac{1}{p}-\alpha})^{\frac{p}{2p+1}}],
\end{align*}
as desired.
\end{proof}

\section{Bounding the Bias-Correction Term (Proof of Theorem \ref{bias-correction-term} in Section 
 \ref{sec:mainres})}\label{sec:bias-corrected-CLT}
Recall from the proof of Theorem \ref{th-1-bias-corrected} that $\Delta := \frac{\mathbb{E}\langle a, \theta_t \rangle - \langle a, \beta^* \rangle}{\sqrt{\Var \langle a, \theta_t \rangle}}$. To bound $|\Delta|$, we will first use \ref{lem:bias-expression} to express the numerator in an alternate way and upper bound it. Finally, Lemma \ref{denominator} provides a suitable lower bound on the denominator. Combining these allows us to prove Theorem \ref{bias-correction-term}.

\begin{proof}[Proof of Theorem \ref{bias-correction-term}]
We have from Lemma \ref{lem:bias-expression} and Lemma \ref{denominator} that
\begin{align*}
\frac{|\mathbb{E}\langle a, \theta_t \rangle - \langle a, \beta^* \rangle|}{\sqrt{\Var \langle a, \theta_t \rangle}} \leq C(\eta \lambda_{\min}(A))^{-\frac{1}{2}}(e^{-\eta \lambda_{\min}(A)d^{-\frac{1}{2}}t^{1-\alpha}}d^{\frac{1}{2}}t^{\alpha})\bigg[\frac{|\beta^* - \theta_0|}{\sigma_{\min}\sqrt{\eta}}\bigg],
\end{align*}
for all $t,d \geq C$, as desired.
\end{proof}
\begin{lemma}\label{lem:bias-expression}
We have for all $a \in \mathbb{R}^d$ that
\begin{align*}
\mathbb{E}\langle a, \theta_t \rangle - \langle a, \beta^* \rangle = a^\top \bigg[\prod_{i = 1}^t \bigg(I - \frac{\eta A}{\sqrt{d}i^{\alpha}}\bigg)\bigg] (\theta_0 - \beta^*).
\end{align*}
In particular, we also have that
\begin{align*}
|\mathbb{E}\langle a, \theta_t \rangle - \langle a, \beta^* \rangle| \leq e^{-\eta \lambda_{\min}(A)d^{-\frac{1}{2}}t^{1-\alpha}}|a||\beta^* - \theta_0|,
\end{align*}
\end{lemma}
\begin{proof}
Observe that
\begin{align*}
\mathbb{E}\langle a, \theta_i - \beta^* \rangle &= \mathbb{E}\langle a, (\theta_{i-1} - \beta^*) + \eta_i X_i(Y_i - \langle X_i , \theta_{i-1} \rangle) \rangle
\\
&=  \mathbb{E}\langle a, (I - \eta_i  X_i X_i^\top)(\theta_{i - 1} - \beta^*) + \eta_i \epsilon_i X_i\rangle) \rangle
\\
&= \mathbb{E}\langle a, (I - \eta_i A)(\theta_{i - 1} - \beta^*) \rangle
\end{align*}
Multiplying these from $i = 1$ to $t$ gives us that
\begin{align*}
\mathbb{E}\langle a, \theta_t \rangle - \langle a, \beta^* \rangle = a^\top \bigg[\prod_{i = 1}^t \bigg(I - \frac{\eta A}{i^{\alpha}}\bigg)\bigg] (\theta_0 - \beta^*),
\end{align*}
proving the first part of the lemma. 
\\
Now Assumption \ref{assump} that $\eta \lambda_{\max}(A)< \eta \bar{\lambda} < C$ implies that $0 < \lambda_{\max}\bigg(I - \frac{\eta A}{\sqrt{d}i^{\alpha}}\bigg) < 1$ for all large enough $d$. This implies that
\begin{align*}
    \mathbb{E}\langle a, \theta_t \rangle - \langle a, \beta^* \rangle &< e^{-\eta \lambda_{\min}(A) d^{-\frac{1}{2}}\sum_{i=1}^t i^{-\alpha}}|a||\theta_0 - \beta^*| \\
    &< e^{-\eta \lambda_{\min}(A) d^{-\frac{1}{2}}t^{1 - \alpha}}|a||\theta_0 - \beta^*|,
\end{align*}
for  all $t,d \geq C$, as desired.
\end{proof}
\section{Proofs for Lemmas in Section~\ref{sec:proofsketch}}

\begin{proof}[Proof of Lemma~\ref{prop:iterate}]
Recall that the update formula is given by
\begin{align*}
    \theta_{i}:= \theta_{i-1}+ \eta_{i}X_{i}(Y_{i}- X_{i}^\top\theta_{i-1}),
\end{align*}
which on simplification gives
\begin{align*}
    \theta_{i} = (I - \eta_i X_i X_i^\top)\theta_{i-1} + \eta_{i}X_iY_i.
\end{align*}
Unraveling the recursion gives us that
\begin{align*}
    \theta_i = \bigg(\prod_{j = 0}^{i-1}(I - \eta_{i-j}X_{i-j}X^\top_{i-j})\bigg)\theta_0 + \sum_{j = 1}^{i}\eta_j\bigg(\prod_{k = 0}^{i-j-1}(I - \eta_{i-k}X_{i-k}X_{i-k}^\top)\bigg)X_jY_j.
\end{align*}
By the definiton $\epsilon_j := Y_j - X^\top_j \beta^*$, this implies
\begin{align*}
       \theta_i = \bigg(\prod_{j = 0}^{i-1}(I - \eta_{i-j}X_{i-j}X^\top_{i-j})\bigg)\theta_0 + \sum_{j = 1}^{i}\eta_j\bigg(\prod_{k = 0}^{i-j-1}(I - \eta_{i-k}X_{i-k}X_{i-k}^\top)\bigg)X_j(X_j^\top \beta^* + \epsilon_j).
\end{align*}
\end{proof}

\begin{proof}[Proof of Lemma~\ref{lemma:martingalerep}]
By the telescoping sum, we have 
\begin{align*}
    M_1 + \dots + M_t = \langle a, \theta_t \rangle - \mathbb{E}\langle a, \theta_t \rangle.
\end{align*}
Now see that,
\begin{align*}
   \mathbb{E}(\theta_t  | X_t,Y_t, X_{t-1}, Y_{t-1},\dots X_{i}, Y_{i}) &=\bigg(\prod_{j = 0}^{t-i}(I - \eta_{t-j}X_{t-j}X^\top_{t-j})\bigg)\bigg(\prod_{j = 1}^{i-1}(I - \eta_{i-j}A)\bigg)\theta_0 \\&+ \sum_{j = i}^t\eta_j\bigg(\prod_{k = 0}^{t-j-1}(I - \eta_{t-k}X_{t-k}X^\top_{t-k})\bigg)X_j(X^\top_j\beta^* + \epsilon_j)
   \\&+ \sum_{j = 1}^{i-1}\bigg(\prod_{k = 0}^{t-i}(I - \eta_{t-k}X_{t-k}X^\top_{t-k})\bigg)\bigg(\prod_{k = 1}^{i-1-j}(I - \eta_{i-k}A)\bigg)[\eta_jA\beta^* + \eta_j \mathbb{E}[\epsilon_jX_j]]
   \\
   &=\bigg(\prod_{j = 0}^{t-i}(I - \eta_{t-j}X_{t-j}X^\top_{t-j})\bigg)\bigg(\prod_{j = 1}^{i-1}(I - \eta_{i-j}A)\bigg)\theta_0 \\&+ \sum_{j = i}^t\eta_j\bigg(\prod_{k = 0}^{t-j-1}(I - \eta_{t-k}X_{t-k}X^\top_{t-k})\bigg)X_j(X^\top_j\beta^* + \epsilon_j)
   \\&+ \sum_{j = 1}^{i-1}\bigg(\prod_{k = 0}^{t-i}(I - \eta_{t-k}X_{t-k}X^\top_{t-k})\bigg)\bigg(\prod_{k = 1}^{i-1-j}(I - \eta_{i-k}A)\bigg)[\eta_jA\beta^*],
\end{align*}
where the last inequality follows from the standard fact that $\mathbb{E}[\epsilon X] = \mathbb{E}[(Y - X^\top \beta^*)X] = 0$. Now, using this and the definition of $M_{t-i+1} := \mathbb{E}[\langle a, \theta_t \rangle| X_t, Y_t, \dots X_i, Y_i] - \mathbb{E}[\langle a, \theta_t \rangle| X_t, Y_t, \dots X_{i+1}, Y_{i+1}]$ 
\begin{align*}
    M_{t-i+1}&=  \bigg(\prod_{j = 0}^{t-i-1}(I - \eta_{t-j}X_{t-j}X^\top_{t-j})\bigg)(\eta_iA - \eta_iX_iX_i^\top) \bigg(\prod_{j = 1}^{i-1}(I - \eta_{i-j}A)\bigg)\theta_0 \\ & +\bigg(\prod_{j = 0}^{t-i-1}(I - \eta_{t-j}X_{t-j}X^\top_{t-j})\bigg)\eta_iX_i(X_i^\top \beta^* + \epsilon_i) \\ &+\bigg(\prod_{j = 0}^{t-i-1}(I - \eta_{t-j}X_{t-j}X^\top_{t-j})\bigg)(\eta_i A - \eta_i X_iX^\top_i)\sum_{j = 1}^{i-1}\bigg(\prod_{k = 1}^{i-1-j}(I - \eta_{i-k}A)\bigg)\eta_j A \beta^* \\&-
    \bigg(\prod_{j = 0}^{t-i-1}(I - \eta_{t-j}X_{t-j}X^\top_{t-j})\bigg)\eta_iA\beta^*.
\end{align*}
But observe from Lemma \ref{matrix-algebra-identity} that,
\begin{align*}
    I - \sum_{j = 1}^{i-1}\bigg(\prod_{k = 1}^{i-1-j}(I - \eta_{i-k}A)\bigg)\eta_j A = \prod_{j = 1}^{i-1}(I - \eta_{i - j}A).
\end{align*}
Thus we get that, for $1 \leq i \leq t$,
\begin{align*}
    M_{t - i + 1} = &\bigg \langle a, \eta_i\bigg(\prod_{j = 0}^{t-i-1}(I - \eta_{t-j}X_{t-j}X^\top_{t-j})\bigg)(X_iX_i^\top - A) \bigg(\prod_{j = 1}^{i-1}(I - \eta_{i-j}A)\bigg)(\beta^* - \theta_0) \\&+
    \epsilon_i \eta_i\bigg(\prod_{j = 0}^{t-i-1}(I - \eta_{t-j}X_{t-j}X^\top_{t-j})\bigg) X_i \bigg \rangle.
\end{align*}
\end{proof}

\begin{proof}[Proof of  Lemma~\ref{lem:VandS}]  
Recall the following definitions from Notation \ref{notation}.
\begin{itemize}
    \item $R_i \coloneqq \prod_{j = i+1}^{t}(I - \eta_{j}X_{j}X_{j}^\top)$ and, $S_i \coloneqq \prod_{j = 1}^{i-1}(I - \eta_{i-j}A)$
    \item $u_i \coloneqq R_ia$ and, $v_i \coloneqq S_i(\beta^* - \theta_0)$
    \item $\mathcal A_i := \mathbb{E}[(X_iX_i^\top - A)v_i v_i^\top (X_iX_i^\top - A) + \epsilon^2_iX_iX_i^\top + \epsilon_i X_iv^\top_i(X_iX_i^\top - A) + \epsilon_i (X_iX^\top_i - A)v_i X^\top_i]$.
\end{itemize}
Substituting these into the expression from $M_{t-i+1}$ from Lemma \ref{lemma:martingalerep} gives us that
\begin{align*}
M_{t-i+1} = \eta_i \langle u_i, (X_iX_i^\top - A)v_i + \epsilon_i X_i \rangle
\end{align*}
Now, recall that $\mathfrak{F}_{t-i}$ is the $\sigma$-field generated by $\{X_t,\epsilon_t, \ldots , X_{i+1}, \epsilon_{i+1}\}$. Observing that $u_i$ conditioned on $\mathfrak{F}_{t-i}$ is deterministic, we immediately obtain that
\begin{align*}
\mathbb{E}[M^2_{t-i+1} | \mathfrak{F}_{t-i}] &= \eta^2_i  [u^\top_i\mathbb{E}([(X_iX_i^\top - A)v_i + \epsilon_i X_i][(X_iX_i^\top - A)v_i + \epsilon_i X_i]^\top)u_i]
\\
&= \eta^2_i [u^\top_i \mathcal A_i u_i]
\end{align*}
Substituting these into the expressions for $V^2(\mathbf{M}) - 1$ and $s^{-2p}(\mathbf{M})\sum_{i = 1}^t \|M_i\|^{2p}_{2p}$ gives us the claimed identities. 
\end{proof}

\section{Proof for the CLT Rates (Theorems~\ref{thm:BoundFirstTerm} and~\ref{thm:BoundSecondTerm} in Section~\ref{sec:proofsketch})}

We start with deriving a lower bound on $s^2(\mathbf{M})$ and an upper bound on $\|M_i\|^{2p}_{2p}$ in Section~\ref{useful}, which will be useful to bound $V^2(\mathbf{M}) - 1$ and $s^{-2p}(\mathbf{M})\sum_{i=1}^t\|M_i\|^{2p}_{2p}$ in Section~\ref{first-term} and Section~\ref{second-term} respectively. 

To proceed, we also introduce the following notations.
\begin{itemize}
    \item $\mathcal D_i \coloneqq \eta_i^2 \mathbb{E}\langle u_i, (X_iX^\top_i - A)v_i + \epsilon_i X_i \rangle^2$
    \item $\mathcal N_i \coloneqq \eta_i^2 (u_i^\top \mathcal A_i u_i - \mathbb{E}[u^\top_i \mathcal A_i u_i])$
    \item $\mathcal N \coloneqq \sum_{i = 1}^t\mathcal N_i$
    \item $\mathcal D \coloneqq \sum_{i = 1}^t \mathcal D_i$
\end{itemize}
\subsection{Bounds on $s^2(\mathbf{M})$ and $\|M_i\|^{2p}_{2p}$}\label{useful}
\begin{lemma}\label{denominator}
Under Assumption \ref{assump}, we have for all $t, d \geq C$ that
\begin{align*}
s^2(\mathbf{M}) \geq c(\eta \lambda_{\min}(A))(\eta \sigma^2_{\min})d^{-\frac{1}{2}}t^{-\alpha}|a|^2.
\end{align*}
Here $C,c > 0$ are absolute constants.
\end{lemma}
\begin{proof}
Throughout the proof, we let $C, c > 0$ respectively denote large and small enough generic absolute constants.
\\\\
Recall the notation that $A := \mathbb{E}[XX^\top]$, $\epsilon := Y - X^\top \beta^*$. and $A_{\sigma} :=  \mathbb{E}[\epsilon^2 XX^\top]$.
Let $\mathbf{e}_1, \mathbf{e_2}, \dots \mathbf{e}_d$ be an eigen-basis of $A$ with corresponding eigen-values $\lambda_1 \geq \lambda_2 \geq \dots \lambda_d > 0$. Finally for all $1 \leq k, k' \leq d$, let $a_k := \langle \mathbf{e}_k, a \rangle$ and $[A_{\sigma}]_{k,k'} := \langle  \mathbf{e}_k, A_\sigma \mathbf{e}_{k'}\rangle$ denote the respective components of $a$ and $A_{\sigma}$ in the above basis.
\\\\
Recall Theorem \ref{lem:VarAsymp} that for all $t,d \geq C$, we have
\begin{align*}
    s^2(\mathbf{M}) = \Var \langle a, \theta_t \rangle = (1+\mathcal E)\eta d^{-\frac{1}{2}}t^{-\alpha}\sum_{k,k' = 1}^d \frac{a_k a_{k'}[A_{\sigma}]_{k,k'}}{\lambda_{k} + \lambda_{k'}}
\end{align*}
where $|\mathcal E| \leq C(\log t + \log d)^2 [(\eta \lambda_{\min}(A))^{-1}d^{\frac{1}{2}}t^{-(1-\alpha)} + (\eta \lambda_{\min}(A))^{-3}\sigma^2\sigma_{\min}^{-2}d^{\frac{1}{2}}t^{-\alpha}]$.
\\\\
Now, observe that
\begin{align*}
\sum_{k,k' = 1}^d \frac{a_ka_{k'}[A_{\sigma}]_{k,k'}}{\lambda_k + \lambda_{k'}} &\geq \frac{a^\top A_{\sigma} a}{2\lambda_{\max}(A)}
\\
&\geq |a|^2 (2\lambda_{\max}(A))^{-1}\lambda_{\min}(A)\sigma^2_{\min}
\\
&\geq |a|^2 (2\eta \lambda_{\max}(A))^{-1}(\eta \lambda_{\min}(A))\sigma^2_{\min}
\\
&\geq c|a|^2 \sigma^2_{\min}(\eta \lambda_{\min}(A)),
\end{align*} 
Here the second inequality follows from Lemma \ref{noise-moment-lower-bound}, and the last inequality followed from Assumption \ref{assump} that $\eta \lambda_{\max}(A) < \eta \bar{\lambda} < C$.  
\\\\
Now observe from Assumption \ref{assump} that $$\lim\limits_{t,d \to \infty}(\log t + \log d)^2 (\eta \lambda_{\min}(A))^{-3} \sigma^2 \sigma^{-2}_{\min}d^{\frac{1}{2}}t^{-\alpha}$$ and
$$\lim\limits_{t,d \to \infty} (\eta \lambda_{\min}(A))^{-1}(\log t + \log d)^2 d^{\frac{1}{2}}t^{-(1 - \alpha)} = 0.$$
These imply that $1 + \mathcal E > \frac{1}{2}$ for all $t,d \geq C$. Combining this with the above equations gives us that
\begin{align*}
s^2(\mathbf{M}) &= (1+\mathcal E)\eta d^{-\frac{1}{2}}t^{-\alpha}\sum_{k,k' = 1}^d \frac{a_k a_{k'}[A_{\sigma}]_{k,k'}}{\lambda_{k} + \lambda_{k'}}
\\
&\geq \frac{1}{2}\eta d^{-\frac{1}{2}}t^{-\alpha}\sum_{k,k' = 1}^d \frac{a_k a_{k'}[A_{\sigma}]_{k,k'}}{\lambda_{k} + \lambda_{k'}}
\\
&\geq c(\eta \lambda_{\min}(A))(\eta \sigma^2_{\min})d^{-\frac{1}{2}}t^{-\alpha}|a|^2,
\end{align*}
as desired.
\end{proof}

\begin{lemma}\label{first-term-numerator}
Recall the martingale construction $(M_i)_{1 \leq i \leq t}$ from Lemma \ref{lemma:martingalerep}. Under Assumption \ref{assump}, we have for all $2 \leq p \leq p_{\max}$ and $t, d \geq C_1$ that
\begin{align*}
    \sum_{i = 1}^t\|M_i\|^{2p}_{2p} &\leq C_2^p \eta^p |a|^{2p}\sigma^{2p}t^{1-2p\alpha}d^{-p}.
\end{align*}
Here $C_1, C_2 > 0$ are absolute constants.
\end{lemma}
\begin{proof}
Throughout the proof, we let $C, c > 0$ respectively denote large and small enough generic absolute constants.
\\\\
Recall Notation \ref{notation} that 
\begin{itemize}
    \item $R_i \coloneqq \prod_{j = i+1}^{t}(I - \eta_{j}X_{j}X_{j}^\top)$ and, $S_i \coloneqq \prod_{j = 1}^{i-1}(I - \eta_{i-j}A)$
    \item $u_i \coloneqq R_ia$ and, $v_i \coloneqq S_i(\beta^* - \theta_0)$
\end{itemize}
Using Lemma \ref{lemma:martingalerep}, we get for all $1 \leq i \leq t$ that
$$
\|M_{t-i+1}\|_{2p}^{2p} = \eta_i^{2p}\mathbb{E}\langle u_i, (X_iX^\top_i - A)v_i + \epsilon_i X_i \rangle^{2p}.
$$
Now, using $\mathbb{E}|U + V|^{2p} \leq 2^{2p-1}(\mathbb{E}|U|^{2p} + \mathbb{E}|V|^{2p})$ for $p \geq 1$, we can say that
\begin{align*}
    \|M_{t-i+1}\|_{2p}^{2p} \leq (C\eta_i)^{2p}[\mathbb{E}\langle u_i, (X_i X^\top_i - A)v_i \rangle^{2p} + \mathbb{E}\langle u_i, \epsilon_i X_i \rangle^{2p}].
\end{align*}
Since $u_i$ is independent of $X_i$ and $\epsilon_i$, we can use Lemma \ref{X-moment-upper-bound} and Lemma \ref{noise-moment-upper-bound} to obtain that
\begin{align*}
   \|M_{t-i+1}\|^{2p}_{2p} \leq (C \eta_i)^{2p} [\bar{\lambda}^{2p}\mathbb{E}|u_i|^{2p}|v_i|^{2p} + \sigma^{2p} \bar{\lambda}^p \mathbb{E}|u_i|^{2p}].
\end{align*}
Now, substituting the upper bounds on $\mathbb{E}|u_i|^{2p}$ and $|v_i|^{2p}$ from Lemma \ref{conc1} and Lemma \ref{S_imoment} respectively, give us for all $t,d \geq C$ that
\begin{align*}
    \|M_{t-i+1}\|^{2p}_{2p} &\leq (C\eta_i)^{2p}\bar{\lambda}^p|a|^{2p}[e^{-2p\eta \lambda_{\min}(A)d^{-\frac{1}{2}}\sum_{j = 1}^t j^{-\alpha}}\bar{\lambda}^p |\beta^* - \theta_0|^{2p}+ e^{-2p\eta \lambda_{\min}(A)d^{-\frac{1}{2}}\sum_{j = i+1}^t j^{-\alpha}}\sigma^{2p}]
    \\
    &\leq (C\eta_i)^{2p}\bar{\lambda}^p|a|^{2p}[e^{-2p\eta \lambda_{\min}(A)d^{-\frac{1}{2}}t^{1-\alpha}}\bar{\lambda}^{p}|\beta^* - \theta_0|^{2p} + e^{-2p\eta \lambda_{\min}(A)d^{-\frac{1}{2}}\sum_{j = i+1}^t j^{-\alpha}}\sigma^{2p}]
    \\
    &\leq C^{2p}d^{-p}i^{-2p\alpha}|a|^{2p}[e^{-2p\eta \lambda_{\min}(A)d^{-\frac{1}{2}}t^{1-\alpha}}|\beta^* - \theta_0|^{2p} + e^{-2p\eta \lambda_{\min}(A)d^{-\frac{1}{2}}\sum_{j = i+1}^t j^{-\alpha}}\eta^p\sigma^{2p}]
    \\
    &\leq C^{2p}\eta^pd^{-p}i^{-2p\alpha}|a|^{2p}[e^{-2pc(\log t + \log d)^{2}}(td)^{pC}+ e^{-2p\eta \lambda_{\min}(A)d^{-\frac{1}{2}}\sum_{j = i+1}^tj^{-\alpha}}]\sigma^{2p}
    \\
    &\leq C^{2p}\eta^pd^{-p}i^{-2p\alpha}|a|^{2p}[e^{-2pc(\log t + \log d)^{2}}(td)^{pC}+ e^{-2p\eta \lambda_{\min}(A)d^{-\frac{1}{2}}(t-i)t^{-\alpha}}]\sigma^{2p}.
\end{align*}
Here, the third and fourth inequalities followed using Assumptions \ref{assump} that 
\begin{itemize}
    \item $\eta \bar{\lambda} < C$. 
    \item $\frac{|\beta^* - \theta_0|^2}{\eta \sigma^2} < (td)^{C}$ for all $t,d \geq C$.
    \item $\lim\limits_{t,d \to \infty} (\eta \lambda_{\min}(A))^{-1} (\log t + \log d)^2 d^{\frac{1}{2}}t^{-(1 - \alpha)} = 0$.
\end{itemize}
Now, let $t_0 := \frac{Kt^{\alpha}d^{\frac{1}{2}}(\log t + \log d)^2}{2\eta \lambda_{\min}(A)}$ for an absolute constant $K > 0$, and observe for $i \leq t-t_0$ that
\begin{align*}
    \|M_{t-i+1}\|^{2p}_{2p} \mathbf{1}_{i  \leq t-t_0} &\leq C^p \eta^p d^{-p}i^{-2p\alpha}|a|^{2p}[e^{-2pc(\log t + \log d)^2}(td)^{pC} + (td)^{-pK}]\sigma^{2p}
    \\
    &\leq C^p \eta^p d^{-p}i^{-2p\alpha}|a|^{2p}(td)^{-pK}\sigma^{2p}
\end{align*}
for all large enough $t,d$. Now, observe for $i \geq t-t_0$ that
\begin{align*}
    \|M_{t-i+1}\|^{2p}_{2p} \mathbf{1}_{i  \geq t-t_0} &\leq C^{2p}\eta^p d^{-p}(t-t_0)^{-2p\alpha}|a|^{2p}\sigma^{2p}
    \\
    &\leq C^{2p}\eta^p d^{-p}t^{-2p\alpha}|a|^{2p}\sigma^{2p}
\end{align*}
Together, these give us that
\begin{align*}
    \sum_{i = 1}^t \|M_{i}\|^{2p}_{2p} &= \sum_{i = 1}^t \|M_{t-i+1}\|^{2p}_{2p}
    \\
    &\leq C^p\eta^p d^{-p}|a|^{2p}\sigma^{2p}(\sum_{i = 1}^{t-t_0} i^{-2p\alpha}(td)^{-pK} + \sum_{i = t-t_0}^tt^{-2p\alpha})
    \\
    &\leq C^p\eta^p d^{-p}|a|^{2p}\sigma^{2p}[(td)^{-pK} + t_0 t^{-2p\alpha}]
    \\
    &\leq C^p\eta^p d^{-p}|a|^{2p}\sigma^{2p}[(td)^{-pK} + t^{1-2p\alpha}]
\end{align*}
The first term in the bracket can be made arbitrarily smaller than the second by choosing $K  > 0$ to be a large enough absolute constant. This gives us the desired result.
\end{proof}

\subsection{Proof of Theorem~\ref{thm:BoundFirstTerm}}\label{first-term}
One of the major tools for proving Theorem \ref{thm:BoundFirstTerm} is Lemma \ref{uniformconvexity3}, which controls how the $p-$norm of a random variable changes if we add a zero mean fluctuation. The proof technique for Lemma \ref{uniformconvexity3} is heavily borrowed from \cite{huang2022matrix}, which proves a more general inequality for random matrices.
\subsubsection{Moment and Concentration Bounds For $R_ia$}\label{R_i}
Recall that $R_i := \prod_{j = i+1}^{t}(I - \eta_{j} X_{j}X^\top_{j})$. 
\begin{lemma}\label{conc1}
Under Assumption \ref{assump}, we have for all $2 \leq p \leq p_{\max}$ and $t,d \geq C_1$ that
\begin{align*}
\mathbb{E}[|R_ia|^{2p}] < C_2^{p}e^{-2p \eta \lambda_{\min}(A)d^{-\frac{1}{2}}\sum_{j = i+1}^t j^{-\alpha}}|a|^4,
\end{align*}
for all fixed $a \in \mathbb{R}^d$. Here $C_1, C_2 > 0$ are absolute constants.
\end{lemma}
\begin{proof}
Throughout the proof, we let $C > 0$ denote a large enough and generic absolute constant.
\\\\
Let $$K_{j}:= \mathbb{E}[(I - \eta_{j}X_{j}X^\top_{j})(I - \eta_{j}X_{j}X^\top_{j})].$$
Further, for all $i+1 \leq k \leq t+1$, define $u_{k,t}$ as the running product
\begin{align*}
u_{k, t} := [\prod_{j = k}^{t}(I - \eta_{j}X_{j}X^\top_{j})]a
\end{align*}
In particular, $u_{i+1,t}:= R_ia$ and $u_{t+1, t} := a$. Now, observe for all $i+1 \leq k \leq t$ that
\begin{align*}
|u_{k,t}|^2
&= |(I - \eta_k X_kX_k^\top)u_{k+1,t}|^2
\\
&= u^{\top}_{k+1,t} (I - \eta_k X_k X_k^\top)(I - \eta_k X_k X_k^\top)u_{k+1,t}
\\
&= u_{k+1, t}^\top [(I - \eta_k X_k X_k^\top)(I - \eta_k X_k X_k^\top )- K_k]u_{k+1, t} + u_{k+1,t}^\top K_k u_{k+1,t}
\end{align*}
Let $U_k:= u_{k+1,t}^\top K_k u_{k+1,t}$ and $V_k := u^\top_{k+1,t}[(I - \eta_kX_kX_k^\top)(I - \eta_kX_kX_k^\top) - K_k]u_{k+1,t}$. Observe that $X_k$ is independent of $u_{k+1,t}$, therefore $\mathbb{E}[V_k|U_k] = 0$. Lemma \ref{uniformconvexity3} now gives us that
\begin{align*}
\mathbb{E}[|u_{k,t}|^{2p}]^{\frac{2}{p}} = \mathbb{E}[|U_k + V_k|^p]^{\frac{2}{p}} \leq \mathbb{E}[|U_k|^p]^{\frac{2}{p}} + C(p-1) \mathbb{E}[|V_k|^p]^{\frac{2}{p}}
\end{align*}
for all $p \geq 2$ and an absolute constant $C>0$. Below we make the claims that
\begin{align*}
    \underbrace{\mathbb{E}|U_k|^p \leq (1 - 2\eta\lambda_{\min}(A)d^{-\frac{1}{2}}k^{-\alpha} + Ck^{-2\alpha})^p\mathbb{E}|u_{k+1,t}|^{2p}}_{(\mathbf{I})}, \quad \underbrace{\mathbb{E}|V_k|^p \leq C^p(d^{-\frac{p}{2}}k^{-p\alpha} + k^{-2p\alpha})\mathbb{E}[|u_{k+1,t}|^{2p}]}_{(\mathbf{II})},
\end{align*}
for an absolute constant $C>0$. These tell us that
\begin{align*}
\mathbb{E}[|u_{k,t}|^{2p}]^{\frac{2}{p}} &\leq \mathbb{E}[|U_k|^p]^{\frac{2}{p}} + C(p-1) \mathbb{E}[|V_k|^p]^{\frac{2}{p}}
\\
&\leq [(1 - 2\eta\lambda_{\min}(A)d^{-\frac{1}{2}}k^{-\alpha} + Ck^{-2\alpha})^2 + Cpd^{-1}k^{-2\alpha} + Cpk^{-4\alpha}]\mathbb{E}[|u_{k+1,t}|^{2p}]^{\frac{2}{p}}
\\
&\leq [1 - 4\eta \lambda_{\min}(A)d^{-\frac{1}{2}}k^{-\alpha} + Cpk^{-2\alpha}]\mathbb{E}[|u_{k+1,t}|^{2p}]^{\frac{2}{p}}
\\
&\leq e^{-4\eta \lambda_{\min}(A)d^{-\frac{1}{2}}k^{-\alpha} + Cpk^{-2\alpha}}\mathbb{E}[|u_{k+1,t}|^{2p}]^{\frac{2}{p}}
\end{align*}
where $C > 0$ is an absolute constant. Here the second last inequality follows from $(x + y)^{\frac{2}{p}} \leq x^{\frac{2}{p}} + y^{\frac{2}{p}}$ for $x,y > 0$ and $p \geq 2$, and the last inequality follows using the fact that $p \geq 2$.
\\\\
Finally multiplying all such inequalities for $k = t$ to $i+1$ gives us that
\begin{align*}
\mathbb{E}[|u_{i+1,t}|^{2p}]^{\frac{2}{p}} &\leq  e^{- 4\eta \lambda_{\min}(A)d^{-\frac{1}{2}}\sum_{j = i+1}^tj^{-\alpha} + Cp \sum_{j = i+1}^t j^{-2\alpha}}\mathbb{E}[|u_{t+1,t}|^{2p}]^{\frac{2}{p}}
\\
&\leq  e^{- 4\eta \lambda_{\min}(A)d^{-\frac{1}{2}}\sum_{j = i+1}^tj^{-\alpha} + Cp}|a|^4
\\
&\leq C^{p}e^{-4\eta \lambda_{\min}(A)d^{-\frac{1}{2}}\sum_{j = i+1}^t j^{-\alpha}}|a|^4.
\end{align*}
Here the first inequality follows using the fact that $\sum_{k = 1}^\infty k^{-2\alpha}  < C$ for $\alpha > \frac{1}{2}$. Finally, raising both sides to the $\frac{p}{2}^{th}$ power gives us the desired result.
\\\\
It now remains to prove the claims $(\mathbf{I})$ and $(\mathbf{II})$ which we do as follows.
\\\\
\textsc{Proof of $(\mathbf{I})$:}
Observe that
\begin{align*}
\mathbb{E}[|U_k|^p] &= \mathbb{E}_{u_{k+1,t}}|u^\top_{k+1,t}\mathbb{E}_X[(I - \eta_{k}X_{k}X^\top_{k})(I - \eta_{k}X_{k}X^\top_{k})]u_{k+1,t}|^p
\\
&= \mathbb{E}_{u_{k+1,t}}(\mathbb{E}_{X_k}|(I - \eta_k X_k X_k^\top)u_{k+1,t}|^2)^p 
\\
&< \mathbb{E}_{u_{k+1,t}}[(1 - 2\eta_{k}\lambda_{\min}(A) + d\eta^2_{k}\bar{\lambda}^2)^p|u_{k+1,t}|^{2p}]
\\
&< (1 - 2\eta_{k}\lambda_{\min}(A) + d\eta^2_k\bar{\lambda}^2)^p \mathbb{E}|u_{k+1,t}|^{2p}
\\
&< (1 - 2\eta\lambda_{\min}(A)d^{-\frac{1}{2}}k^{-\alpha} + Ck^{-2\alpha})^p \mathbb{E}|u_{k+1,t}|^{2p},
\end{align*}
as desired. Here the third inequality follows using Lemma \ref{momenthelper1}, second last from the independence of $u_{i,t-1}$ and $X_t$, and the last one from assumptions \ref{assump} that $\eta \bar{\lambda} < C$.
\\\\
\textsc{Proof of $(\mathbf{II})$:}
Define $W_k(X_k) := (I - \eta_kX_kX_k^\top)(I - \eta_k X_k X_k^\top) - K_k$. Then we have,
\begin{align*}
W_k(X_k) &= (I - \eta_k X_k X_k^\top)(I - \eta_k X_k X_k^\top) - K_k
\\
&= (I - \eta_k X_k X_k^\top)(I - \eta_k X_k X_k^\top) - \mathbb{E}[(I - \eta_k X_k X_k^\top)(I - \eta_k X_k X_k^\top)]
\\
&= 2\eta_k (A - X_k X_k^\top) + \eta_k^2 (X_k X_k^\top X_k X^\top_k - \mathbb{E}[X_k X_k^\top X_k X^\top_k])
\end{align*}
This gives us for any fixed vector $u$ that
\begin{align*}
\mathbb{E}_{X_k} |u^\top W_k(X_k) u|^{p} &=  \mathbb{E}_{X}|2\eta_k(u^\top A u -  |X^\top u|^2) + \eta_k^2 (|XX^\top u|^2 - \mathbb{E}_X|XX^\top u|^2)|^{p}
\\
&\leq C^p \eta_k^p[(u^\top A u)^p + \mathbb{E}|X^\top u|^{2p}] + C^p \eta_k^{2p}[\mathbb{E}|XX^\top u|^{2p} + (\mathbb{E}|XX^\top u|^2)^p] 
\\
&\leq C^p \eta_k^p \bar{\lambda}^p |u|^{2p} + C^p \eta_k^{2p}d^p \bar{\lambda}^{2p}|u|^{2p},
\\
&\leq (C^pd^{-\frac{p}{2}}t^{-p\alpha} + C^p t^{-2p\alpha})|u|^{2p}.
\end{align*}
Here the third inequality follows from assumptions \ref{assump} on $X^\top u$ and Lemma \ref{momenthelper1} and the fourth inequality follows from assumptions \ref{assump} that $\eta \bar{\lambda} < C$. Now, using the independence of $X_k$ and $u_{k+1,t}$ along with the above gives us that
\begin{align*}
\mathbb{E}|V|^p &= \mathbb{E}|u_{k+1,t}^\top W_k(X_k) u_{k+1,t}|^p
\\
&\leq C^p(d^{-\frac{p}{2}}t^{-p\alpha} + t^{-2p\alpha})\mathbb{E}[|u_{k+1,t}|^{2p}],
\end{align*}
as desired.
\end{proof}
\begin{lemma}\label{conc2}
For a random variable $W \in \mathbb{R}$ define $\|W\|_p := \mathbb{E}[|W|^p]^{\frac{1}{p}}$. 
Recall that $$R_i := \prod_{j = i+1}^{t} (I - \eta_{j}X_{j}X^\top_{j}).$$ Under Assumption \ref{assump}, we have for all $t, d \geq C_1$ that
\begin{align*}
\|a^\top R_i \mathcal A_i R_i a - \mathbb{E}[a^\top R_i \mathcal A_i R_i a]\|^2_p 
\leq C_2pe^{- 4\eta \lambda_{\min}(A)d^{-\frac{1}{2}}\sum_{k = i+1}^t k^{-\alpha}}\lambda_{\max}(\mathcal A_i)^2 |a|^4 \sum_{j = i+1}^t j^{-2\alpha}
\end{align*}
for all positive-definite symmetric matrices $\mathcal A_i \in \mathbb{R}^{d \times d}$, fixed $a \in \mathbb{R}^d$ and $2 \leq p \leq p_{\max}$. Here, $C_1, C_2>0$ are absolute constants.
\end{lemma}
\begin{proof}
Throughout the proof, we let $C > 0$ denote a large enough and generic absolute constant.
\\\\
As in the proof of Lemma \ref{conc1}, for all $i+1 \leq k \leq t+1$, we define $u_{k,t}$ as the running product
\begin{align*}
u_{k,t} := [\prod_{j = k}^{t}(I - \eta_{j}X_{j}X^\top_{j})]a
\end{align*}
In particular, $u_{i+1,t} := R_ia$ and $u_{t+1,t} := a$. Further, we also define the sequence of matrices $\{\mathcal A_{i,k}\}_{k = i}^{t}$ recursively as $\mathcal A_{i,i} := \mathcal A_i$ and 
\begin{align*}
\mathcal A_{i,k} := \mathbb{E}_{X}[(I - \eta_{k}X_kX_k^\top)\mathcal A_{i,k-1} (I - \eta_{k}X_kX_k^\top)]
\end{align*}
for all $i+1 \leq k \leq t$. These definitions will help us use Lemma \ref{uniformconvexity3} recursively to obtain the bound, as in the proof of Lemma \ref{conc1}. 
\\\\
Now observe for any $i+1 \leq k \leq t$ that,
\begin{align*}
u^\top_{k,t} \mathcal A_{i,k-1} u_{k,t} - \mathbb{E}[u^\top_{k,t} \mathcal A_{i,k-1} u_{k,t}] = &u^\top_{k+1,t}[(I - \eta_k X_kX_k^\top)\mathcal A_{i,k-1} (I - \eta_k X_kX_k^\top) - \mathcal A_{i,k}]u_{k+1,t} 
\\
&+ (u^\top_{k+1,t}\mathcal A_{i,k} u_{k+1,t} - \mathbb{E}[u^\top_{k+1,t}\mathcal A_{i,k} u_{k+1,t}])
\end{align*}
As in the proof of Lemma \ref{conc1}, let $$V_k := u^\top_{k+1,t}[(I - \eta_kX_kX_k^\top)\mathcal A_{i,k-1}(I - \eta_kX_kX_k^\top) - \mathcal A_{i,k}
]u_{k+1,t}$$ and $$U_k:= u^\top_{k+1,t} \mathcal A_{i,k}u_{k+1,t} - \mathbb{E}[u^\top_{k+1,t}\mathcal A_{i,k} u_{k+1,t}].$$ Observe that $\mathbb{E}[V_k|U_k] = 0$. Lemma \ref{uniformconvexity3} now tells us that
\begin{align*}
& \|u_{k,t}^\top \mathcal A_{i,k-1} u_{k,t} - \mathbb{E}(u^\top_{k,t} \mathcal A_{i,k-1} u_{k,t})\|^2_p\\
= & \mathbb{E}[|U_k + V_k|^p]^{\frac{2}{p}} 
\\
 \leq  & \mathbb{E}[|U_k|^p]^{\frac{2}{p}} + C(p-1)\mathbb{E}[|V_k|^p]^{\frac{2}{p}}
 \\
 = & \|u^\top_{k+1,t}\mathcal A_{i,k}u_{k+1,t} - \mathbb{E}(u^\top_{k+1,t} \mathcal A_{i,k} u_{k+1,t})\|^2_p + C(p-1)\|V_k\|^2_p
\end{align*}
Below we make the claim that
\begin{align*}
    \underbrace{\|V_k\|^2_p \leq K_{k+1,t} k^{-2\alpha} \lambda_{\max}(\mathcal A_{i,k-1})^2|a|^4}_{(\mathbf{I})},
\end{align*}
where $K_{k+1,t} := Ce^{- 4\eta \lambda_{\min}(A)d^{-\frac{1}{2}}\sum_{j = k+1}^{t}j^{-\alpha}}$. This tells us that 
\begin{align*}
& \|u^\top_{k,t}\mathcal A_{i,k-1}u_{k,t} - \mathbb{E}(u^\top_{k,t}\mathcal A_{i,k-1}u_{k,t})\|^2_p \\
\leq & \|U_k\|^2_p + C(p-1)\|V_k\|^2_p 
\\
\leq & \|u^\top_{k+1,t}\mathcal A_{i,k}u_{k+1,t} - \mathbb{E}(u^\top_{k+1,t}\mathcal A_{i,k}u_{k+1,t})\|^2_p 
\\
&+ pK_{k+1,t}k^{-2\alpha}\lambda_{\max}(\mathcal A_{i,k-1})^2|a|^4.
\end{align*}
Finally adding all such inequalities from $k = i+1$ to $t$ gives us that
\begin{align*}
& \|u^\top_{i+1,t}\mathcal A_{i,i} u_{i+1,t} - \mathbb{E}(u^\top_{i+1,t} \mathcal A_{i,i} u_{i+1,t})\|^2_p \\
\leq & \|u^\top_{t+1,t}\mathcal A_{i,t} u_{t+1,t} - \mathbb{E}(u^\top_{t+1,t} \mathcal A_{i,t} u_{t+1,t})\|^2_p 
\\
 &~~~+ p|a|^4\sum_{j = i+1}^tK_{j+1,t}j^{-2\alpha}\lambda_{\max}(\mathcal A_{i,j-1})^2
\\
= & \|a^\top\mathcal A_{i,t} a - \mathbb{E}(a^\top\mathcal A_{i,t} a)\|^2_p + p|a|^4\sum_{j = i+1}^tK_{j+1,t}j^{-2\alpha}\lambda_{\max}(\mathcal A_{i,j-1})^2
\\
= & p|a|^4\sum_{j = i+1}^tK_{j+1,t}j^{-2\alpha}\lambda_{\max}(\mathcal A_{i,j-1})^2
\\
\leq  & p|a|^4\sum_{j = i+1}^tK_{j+1,t}j^{-2\alpha}\lambda_{\max}K_{i+1, j-1} \lambda_{\max}(\mathcal A_i)^2
\\
\leq & Cp K_{i+1,t} \lambda_{\max}(\mathcal A_i)^2|a|^4\sum_{j = i+1}^t e^{4\eta \lambda_{\min}(A)d^{-\frac{1}{2}}j^{-\alpha}}j^{-2\alpha}
\\
\leq & Cpe^{- 4\eta \lambda_{\min}(A)d^{-\frac{1}{2}}\sum_{k = i+1}^t k^{-\alpha}}\lambda_{\max}(\mathcal A_i)^2 |a|^4 \sum_{j = i+1}^t j^{-2\alpha},
\end{align*}
as desired. Here the third last inequality follows from Lemma \ref{conchelper} which is proved in Appendix \ref{sec:app3}, and the last inequality follows from Assumption \ref{assump} using $$\eta \lambda_{\min}(A) < \eta \lambda_{\max}(A) < \eta \bar{\lambda} < C.$$
\\\\
It now remains to prove claim $(\mathbf{I})$ which we do as follows.
\\\\
\textsc{Proof of $(\mathbf{I})$:}
Define $W_k(X) := (I - \eta_kXX^\top)\mathcal A_{i,k-1}(I - \eta_kXX^\top) - \mathcal A_{i,k}$. Then we have,
\begin{align*}
W_k(X) &= \eta_k [(A - XX^\top) \mathcal A_{i,k-1} + \mathcal A_{i,k-1}(A - XX^\top) + \eta_k [XX^\top \mathcal A_{i,k-1} XX^\top - \mathbb{E}[XX^\top \mathcal A_{i,k-1} XX^\top]]
\end{align*}
This gives us for any fixed vector $u$ that
\begin{align*}
\mathbb{E}|u^\top W_k(X) u|^p &= \eta^p_k\mathbb{E}|2u^\top \mathcal A_{i,k-1} (A - XX^\top)u  + \eta_k(u^\top XX^\top \mathcal A_{i,k-1} XX^\top u - \mathbb{E}[u^\top XX^\top \mathcal A_{i,k-1}XX^\top u])|^p
\\
&\leq \eta^p_k\lambda_{\max}(\mathcal A_{i,k-1})^p\mathbb{E}|2|u|(|Au| + |XX^\top u|) + \eta_k(|XX^\top u|^2 + \mathbb{E}|XX^\top u|^2)|^p
\\
&\leq C^p \eta_k^p \lambda_{\max}(\mathcal A_{i,k-1})^p(|u|^p(|Au|^p + \mathbb{E}|XX^\top u|^p) +  \eta_k^p(\mathbb{E}|XX^\top u|^{2p} + (\mathbb{E}|XX^\top u|^2)^p))
\\
&\leq C^p\lambda_{\max}(\mathcal A_{i,k-1})^p(\eta_kd^{\frac{1}{2}}\bar{\lambda})^p (1 + (\eta_k d^{\frac{1}{2}}\bar{\lambda})^p)|u|^{2p}
\\
&\leq C^p(k^{-\alpha})^p \lambda_{\max}(\mathcal A_{i,k-1})^p(1 + (k^{-\alpha})^p)|u|^{2p}.
\end{align*}
Here the second-last inequality follows from Lemma \ref{momenthelper1} and the last inequality follows using assumptions \ref{assump} that $\eta \bar{\lambda} < C$. Now, using the independence of $X_k$ and $u_{k+1,t}$ along with the above gives us that
\begin{align*}
\mathbb{E}|V_k|^p &= \mathbb{E}|u^\top_{k+1,t} W_k(X_k) u_{k+1,t}|^p
\\
&\leq C^p(k^{-\alpha})^p \lambda_{\max}(\mathcal A_{i,k-1})^p(1 + (k^{-\alpha})^p)\mathbb{E}|u_{k+1,t}|^{2p}
\\
&\leq C^p(k^{-\alpha})^p \lambda_{\max}(\mathcal A_{i,k-1})^p(1 + (k^{-\alpha})^p)e^{- 2p\eta\lambda_{\min}(A)d^{-\frac{1}{2}}\sum_{j = k+1}^{t}j^{-\alpha}}|a|^{2p}
\\
&\leq C^pk^{-p\alpha} \lambda_{\max}(\mathcal A_{i,k-1})^pe^{- 2p\eta\lambda_{\min}(A)d^{-\frac{1}{2}}\sum_{j = k+1}^{t}j^{-\alpha}}|a|^{2p}
\end{align*}
Raising both sides to the $\frac{2}{p}^{th}$ power gives us that
\begin{align*}
\mathbb{E}[|V_k|^p]^{\frac{2}{p}} &\leq Ck^{-2\alpha} \lambda_{\max}(\mathcal A_{i,k-1})^2e^{- 4\eta\lambda_{\min}(A)d^{-\frac{1}{2}}\sum_{j = k+1}^{t}j^{-\alpha}}|a|^4,
\end{align*}
as desired.
\end{proof}
\subsubsection{Final Step of the proof of Theorem~\ref{thm:BoundFirstTerm}}
Recall from Lemma~\ref{lem:VandS} that $V^2(\mathbf{M})-1 = \mathcal{N}/\mathcal{D}$. In Lemma~\ref{denominator}, we have shown a lower bound to $\mathcal{D}$. We now proceed to bound $\mathbb{E}[|\mathcal N|^p]$. Recall that $\mathcal N = \sum_{i = 1}^t \mathcal N_i$, so we will bound each $\mathbb{E}[|\mathcal N_i|^p]$ and then use Jensen's inequality on the function $x \to x^p$. To bound $\mathbb{E}[|\mathcal N_i|^p]$, we will use Lemma \ref{conc2} on the identity $\mathcal N_i := u_i^\top \mathcal A_i u_i - \mathbb{E}[u_i^\top \mathcal A_i u_i]$ that we showed in Lemma \ref{lem:VandS}.
\\\\
Throughout the proof, we let $C> 0$ and $c > 0$ denote large and small enough generic absolute constants.
\begin{proof}[Proof of Theorem~\ref{thm:BoundFirstTerm}.] 
Recall from Notation \ref{notation} that $u_i := R_i a$, $v_i := S_i(\beta^* - \theta_0)$ and 
\begin{align*}
\mathcal A_i := \mathbb{E}[(X_iX_i^\top - A)v_i v_i^\top (X_iX_i^\top - A) + \epsilon^2_iX_iX_i^\top + \epsilon_i X_iv^\top_i(X_iX_i^\top - A) + \epsilon_i (X_iX^\top_i - A)v_i X^\top_i].
\end{align*}
Now Lemma \ref{conc2} give us for all $2 \leq p \leq p_{\max}$ that 
\begin{align*}
\mathbb{E}[|\mathcal N_i|^p]^{\frac{2}{p}} &= \eta_i^4\mathbb{E}[|u_i^\top \mathcal A_i u_i - \mathbb{E}[u_i^\top \mathcal A_i u_i]|^p]^{\frac{2}{p}}
\\
&\leq Cp\eta^4_i \lambda_{\max}(\mathcal A_i)^2 |a|^4e^{- 4\eta \lambda_{\min}(A) d^{-\frac{1}{2}}\sum_{j = i+1}^t j^{-\alpha}}\sum_{j = i+1}^t j^{-2\alpha}.
\end{align*}
Here $C > 0$ is an absolute constant. Next Lemma \ref{lem:calNnegEc} gives us that
\begin{align*}
\lambda_{\max}(\mathcal A_i) &\leq C(\sigma^2 \bar{\lambda} + \bar{\lambda}^2e^{-2\eta \lambda_{\min}(A)d^{-\frac{1}{2}}\sum_{j  = 1}^{i-1}j^{-\alpha}}|\beta^* - \theta_0|^2).
\end{align*}
Together, these give us for all $2\leq p \leq p_{\max}$ and $1 \leq i \leq t$ that
\begin{align*}
\mathbb{E}[|\mathcal N_i|^p]^{\frac{2}{p}} &\leq Cp\eta_i^4 |a|^4\bar{\lambda}^2e^{-4\sum_{j = i+1}^t \eta_j \lambda_{\min}(A)}(\bar{\lambda}^2e^{-4\sum_{j = 1}^{i-1}\eta_j \lambda_{\min}(A)}|\beta^* - \theta_0|^4 + \sigma^4)\sum_{j = i+1}^t j^{-2\alpha}
\\
&< Cp|a|^4i^{-4\alpha} d^{-2}(e^{-4\eta \lambda_{\min}(A)d^{-\frac{1}{2}}\sum_{j = 1}^tj^{-\alpha}}|\beta^* - \theta_0|^4 + \eta^2 \sigma^4 e^{-4 \eta \lambda_{\min}(A)d^{-\frac{1}{2}}\sum_{j  = i+1}^t j^{-\alpha}}\sum_{j = i+1}^t j^{-2\alpha})
\\
&< Cp|a|^4 i^{-4\alpha}d^{-2}(e^{-4\eta \lambda_{\min}(A)d^{-\frac{1}{2}}t^{1-\alpha}}|\beta^* - \theta_0|^4 + \eta^2  \sigma^4 e^{-4\eta \lambda_{\min}(A)d^{-\frac{1}{2}}\sum_{j = i+1}^t j^{-\alpha}}\sum_{j = i+1}^t j^{-2\alpha}).
\end{align*}
Here the second inequality followed from assumptions \ref{assump} that $\eta \bar{\lambda} < C$. 
\\\\
Now consider the cutoff $t_0 := \frac{Kt^{\alpha} d^{\frac{1}{2}}(\log t + \log d)}{\eta \lambda_{\min}(A)}$ for an absolute constant $K > 0$. Observe for $i \leq t - t_0$ that
\begin{align*}
\mathbb{E}[|\mathcal N_i|^p]^{\frac{2}{p}}\mathbf{1}_{i  \leq t-t_0} \leq Cp|a|^4 i^{-4\alpha} d^{-2}(e^{-4\eta \lambda_{\min}(A)d^{-\frac{1}{2}}t^{1-\alpha}}|\beta^* - \theta_0|^4 + \eta^2 \sigma^4 t^{-4K}d^{-4K}).
\end{align*}
Further observe for $i \geq t - t_0$ that 
\begin{align*}
\mathbb{E}[|\mathcal N_i|^p]^{\frac{2}{p}}\mathbf{1}_{i \geq t-t_0} \leq Cp|a|^4 t^{-4\alpha}d^{-2} (e^{-4\eta \lambda_{\min}(A)d^{-\frac{1}{2}}t^{1-\alpha}}|\beta^* - \theta_0|^4 + \eta^2 \sigma^4 t_0 t^{-2\alpha})
\end{align*}
Finally,  combining these observations gives us for all $t,d \geq C$ that
\begin{align*}
\mathbb{E}|\mathcal N|^p &= \mathbb{E}|\sum_{i = 1}^t \mathcal N_i|^p
\\
&\leq (\sum_{i = 1}^t (\mathbb{E}|\mathcal N_i|^p)^{\frac{1}{p}})^p
\\
&\leq (Cp)^pd^{-p}|a|^{2p} \bigg[\bigg(\sum_{i = 1}^t i^{-2\alpha} (e^{-2\eta \lambda_{\min}(A)d^{-\frac{1}{2}}t^{1-\alpha}}|\beta^* - \theta_0|^2 + \eta \sigma^2 t^{-2K}d^{-2K})\bigg) 
\\
&\quad \quad \quad \quad\quad \quad  + t_0 t^{-2\alpha} (e^{-2\eta \lambda_{\min}(A)d^{-\frac{1}{2}}t^{1-\alpha}} |\beta^* - \theta_0|^2 + \eta \sigma^2 t_0^{\frac{1}{2}}t^{-\alpha})\bigg]^p 
\\
&\leq (Cp)^p|a|^{2p} d^{-p}(e^{-2\eta \lambda_{\min}(A)d^{-\frac{1}{2}}t^{1-\alpha}}|\beta^* - \theta_0|^2 + \eta \sigma^2 (t^{\frac{3}{2}}_0t^{-3\alpha}))^p
\\
&\leq (Cp)^p |a|^{2p}(e^{-2\eta \lambda_{\min}(A)d^{-\frac{1}{2}}t^{1-\alpha}}d^{-1}|\beta^* - \theta_0|^2 + \eta \sigma^2 (\log t + \log d)^{\frac{3}{2}}(\eta \lambda_{\min}(A))^{-\frac{3}{2}} t^{-\frac{3\alpha}{2}}d^{-\frac{1}{4}})^p
\\
&\leq (Cp)^p |a|^{2p} (\eta\sigma^2)^p[e^{-c(\log t + \log d)^2} (td)^C+(\log t + \log d)^{\frac{3}{2}} (\eta \lambda_{\min}(A))^{-\frac{3}{2}} t^{\frac{-3\alpha}{2}} d^{-\frac{1}{4}}]^p
\\
&\leq (Cp)^p |a|^{2p} (\eta\sigma^2)^p[(\log t + \log d)^{\frac{3}{2}} (\eta \lambda_{\min}(A))^{-\frac{3}{2}} t^{\frac{-3\alpha}{2}} d^{-\frac{1}{4}}]^p.
\\
&\leq C^p |a|^{2p} (\eta\sigma^2)^p[(\log t + \log d)^{\frac{3}{2}} (\eta \lambda_{\min}(A))^{-\frac{3}{2}} t^{-\frac{3\alpha}{2}} d^{-\frac{1}{4}}]^p,
\end{align*}
Here,
\begin{itemize}
    \item The third-last inequality follows from Assumption \ref{assump} as $\frac{|\beta^* - \theta_0|^2}{\eta \sigma^2} < (td)^{C}$ for all $t, d \geq C$. 
    \item The second-last inequality follows by observing that the first term in the bracket $e^{-c(\log t + \log d)^2}(td)^C$ becomes much smaller than the second for large enough $t,d$, because of Assumption \ref{assump} as $$\eta \lambda_{\min}(A) < \eta \lambda_{\max}(A) < \eta \bar{\lambda} < C.$$
\end{itemize}  
Now, recall from Lemma \ref{denominator} that $s^2(\mathbf{M}) \geq c(\eta \lambda_{\min}(A))(\eta \sigma^2_{\min})d^{-\frac{1}{2}}t^{-\alpha}|a|^2$ for an absolute constant $c > 0$. Together, these bounds give us that
\begin{align*}
\|V^2(\mathbf{M}) - 1\|^p_p &= \frac{\mathbb{E}|\mathcal N|^p}{|\mathcal D|^p}
\\
&\leq C^p[\sigma^2/\sigma^2_{\min}]^p (\eta \lambda_{\min}(A))^{-\frac{5p}{2}}(\log t + \log d)^{\frac{3p}{2}}t^{-\frac{p\alpha}{2}}d^{\frac{p}{4}},
\end{align*}
for all $t,d \geq C$ and $2 \leq p \leq p_{\max}$, as desired. 
\end{proof}

\subsection{Proof of Theorem~\ref{thm:BoundSecondTerm}}\label{second-term}

\begin{proof}[Proof of Theorem~\ref{thm:BoundSecondTerm}]
Recall that Lemma \ref{denominator} shows a lower bound on $s^{2}(\mathbf{M})$ and we can use Lemma \ref{first-term-numerator} to get an upper bound on $\sum_{i = 1}^t\|M_i\|^{2p}_{2p}$.  Combining those two bounds gives us that
\begin{align*}
s^{-2p}(\mathbf{M})\sum_{i = 1}^t \|M_i\|^{2p}_{2p} &\leq C^p (\eta \lambda_{\min}(A))^{-p}[\sigma^2/\sigma^2_{\min}]^p d^{-\frac{p}{2}}t^{1-p\alpha} 
\end{align*}
for all $t, d\geq C$ and $2 \leq p \leq p_{\max}$. Here $C > 0$ represents a generic absolute constant. This completes the proof.
\end{proof}

\section{Proofs for Variance Estimation (Theorem~\ref{varestmainthm} and Theorem~\ref{lem:VarAsymp} in Section~\ref{sec:varest})}\label{varestproof}
\begin{proof}[Proof of Theorem~\ref{varestmainthm}]
Throughout the proof, we let $C  > 0$ and $c > 0$ respectively denote large and small enough generic absolute constants.
\\\\
Recall that 
\begin{align*}
u_{i_1,i_2} := [\prod_{j = i_1}^{i_2}(I - \eta_{t-i_2 + j}X_{j}X^\top_{j})]a, \quad t_0 := t^{\alpha}d^{\frac{1}{2}}(\log t + \log d)^2    
\end{align*}
and
\begin{align*}
    \hat{\mathbf{V}}_{k} := \sum_{i = s_k}^{s_k + t_0 - 1}  \eta^2_{i + \frac{t}{2} - kt_0} (Y_i - X^\top_i \theta_{\frac{t}{2}})^2 [u^\top_{i+1, s_k + t_0 - 1}X_i]^2 \quad \forall 1 \leq k \leq \frac{t}{2t_0}, 
\end{align*}
where $s_k := t/2 + (k-1)t_0 + 1$. Now, define
\begin{align*}
    \mathbf{V}_k := \sum_{i = \frac{t}{2} + (k-1)t_0 + 1}^{\frac{t}{2} + kt_0}  \eta^2_{i + \frac{t}{2} - kt_0} (Y_i - X^\top_i \beta^*)^2 [u^\top_{i+1, s_k + t_0 - 1}X_i]^2 \quad \forall 1 \leq k \leq \frac{t}{2t_0}.
\end{align*}
Now recall that $\hat{V}_t := \frac{2t_0}{t}\sum_{k = 1}^{t/(2t_0)}\hat{\mathbf{V}}_k$. We also define $V_t := \frac{2t_0}{t}\sum_{k = 1}^{t/(2t_0)}{\mathbf{V}}_k$ and observe that
\begin{align*}
    \mathbb{E}|\hat{V}_t - \Var \langle a, \theta_t \rangle| &\leq \mathbb{E}|\hat{V}_t - V_t| + \mathbb{E}|V_t - \Var \langle a, \theta_t \rangle|
    \\
    &= \mathbb{E}\bigg|\frac{2t_0}{t}\sum_{k = 1}^{t/(2t_0)}(\hat{\mathbf{V}}_k - \mathbf{V}_k)\bigg| + \mathbb{E}|V_t - \Var \langle a, \theta_t \rangle|
    \\
    &\leq \frac{2t_0}{t}\sum_{k = 1}^{t/(2t_0)}\mathbb{E}|\mathbf{V}_k - \hat{\mathbf{V}}_k| + \mathbb{E}[V_t - \Var \langle a, \theta_t \rangle]
    \\
    &\leq \frac{2t_0}{t}\sum_{k = 1}^{t/(2t_0)}\mathbb{E}|\mathbf{V}_k - \hat{\mathbf{V}}_k|  + \mathbb{E}|V_t - \mathbb{E}[V_t]| + [\mathbb{E}[V_t] - \Var \langle a, \theta_t \rangle]
\end{align*}
Below we show that
\begin{align*}
    &\underbrace{\mathbb{E}|\mathbf{V}_k - \hat{\mathbf{V}}_k| \leq C\eta \sigma^2 |a|^2 (\log t + \log d)^2(d^{-\frac{1}{4}}t^{-\frac{3\alpha}{2}}) \quad \forall \quad 1\leq k \leq \frac{t}{2t_0},}_{(\mathbf{I})}
    \\
    &\underbrace{\mathbb{E}|V_t - \mathbb{E}[V_t]|^2 \leq C\eta^2\sigma^4|a|^4(\log t + \log d)^6 t^{-1-\alpha}d^{-\frac{1}{2}}}_{(\mathbf{II})},
    \\
    &\underbrace{\mathbb{E}[V_t] - \Var \langle a, \theta_t \rangle] \leq |\mathcal E|\Var \langle a, \theta_t \rangle,}_{(\mathbf{III})}
\end{align*}
where $|\mathcal E|$ is the same error term that appears in Theorem \ref{lem:VarAsymp}.
Combining these bounds gives us that,
\begin{align*}
    \mathbb{E}|\hat{V}_t - \Var \langle a, \theta_t \rangle| 
    &\leq (C\eta \sigma^2 |a|^2)(\log (td))^2d^{-\frac{1}{4}}[t^{-\frac{3\alpha}{2}}+(\log (td))t^{-\frac{1}{2}-\frac{\alpha}{2}}] 
    \\
    &+|\mathcal E|\Var \langle a, \theta_t \rangle
\end{align*}
Using this and the lower bound on $\Var \langle a, \theta_t \rangle \geq c (\eta \lambda_{\min}(A))(\eta\sigma^2_{\min})d^{-\frac{1}{2}}t^{-\alpha}|a|^2$ from Lemma \ref{denominator}, along with the Assumption \ref{assump-varest} that $\eta \lambda_{\min}(A) > c$ gives us that
\begin{align*}
    \frac{\mathbb{E}|\hat{V}_t - \Var \langle a , \theta_t \rangle|}{\Var \langle a, \theta_t \rangle} &\leq |\mathcal E| + C(\sigma^2/\sigma^2_{\min})[\log (td)]^2d^{\frac{1}{4}}[t^{-\frac{\alpha}{2}} + [\log (td)]t^{-\frac{(1-\alpha)}{2}}],
\end{align*}
where $|\mathcal E| \leq C(\log t + \log d)^2[d^{\frac{1}{2}}t^{-(1 - \alpha)} + \sigma^2\sigma^{-2}_{\min}d^{\frac{1}{2}}t^{-\alpha}]$. Now, under Assumption \ref{assump}, the dominating error term is $C(\sigma^2/\sigma^2_{\min})(\log t + \log d)^{3}d^{\frac{1}{4}}t^{-\frac{(1 - \alpha)}{2}}$.
\\\\
Therefore,
\begin{align*}
    \frac{\mathbb{E}|\hat{V}_t - \Var \langle a, \theta_t \rangle|}{\Var \langle a, \theta_t \rangle} &\leq 
     C(\sigma^2/\sigma^2_{\min})(\log t + \log d)^3d^{\frac{1}{4}}t^{-\frac{(1 - \alpha)}{2}},
\end{align*}
for all $t,d \geq C$, as desired. This completes the proof of the first part of Theorem \ref{varestmainthm}.
\\\\
For the second part, we first use Markov's inequality and obtain that
$$\mathbb{P}\Big(\Big|\frac{\hat{V}_t}{\Var(\langle a, \theta_{t} \rangle)}-1\Big|\geq \kappa \cdot \mathbb{E}\Big[\Big|\frac{\hat{V}_t}{\Var(\langle a, \theta_{t} \rangle)}-1\Big|\Big]\Big)\leq \frac{1}{\kappa}, \quad \forall \kappa>0.$$
Set $\kappa := \bigg(\frac{\mathbb{E}|\hat{V}_t - \Var \langle a, \theta_t \rangle|}{\Var \langle a, \theta_t \rangle}\bigg)^{-\frac{1}{2}}$ and let $\omega:=  \bigg(\frac{\mathbb{E}|\hat{V}_t - \Var \langle a, \theta_t \rangle|}{\Var \langle a, \theta_t \rangle}\bigg)^{\frac{1}{2}}$. Using the above inequality, we get 
\begin{align}\label{eq:bbostrap_last_1}
\begin{aligned}
    \sup_{\gamma \in \mathbb{R}} \inf_{|\xi|\leq \omega } &\bigg|\mathbb{P}\bigg(\frac{\langle a, \theta_t \rangle - \langle a, \beta^* \rangle}{\sqrt{\Var\langle a, \theta_t \rangle}}\leq (1+\xi)\gamma\bigg) - \mathbb{P}\bigg(\frac{\langle a, \theta_t \rangle - \langle a, \beta^* \rangle}{\sqrt{\hat{V}_t}}\leq \gamma\bigg) \bigg| \\
    &\leq  \mathbb{P}\Big(\Big|\frac{\hat{V}_t}{\Var(\langle a, \theta_{t} \rangle)}-1\Big|\geq \kappa \cdot  \mathbb{E}\Big[\Big|\frac{\hat{V}_t}{\Var(\langle a, \theta_{t} \rangle)}-1\Big|\Big]\Big)\leq \frac{1}{\kappa}.
    \end{aligned}
\end{align}
Now recall from Theorem \ref{th-1-bias-corrected} that we have
$$
\sup_{\gamma \in \mathbb{R}} \bigg|\mathbb{P}\bigg(\frac{\langle a, \theta_t \rangle - \langle a, \beta^* \rangle}{\sqrt{\Var\langle a, \theta_t \rangle}}\leq \gamma\bigg) - \Phi(\gamma)\bigg| \leq d^{true}_K.$$
Due to the above inequality and Lipschitz continuity of $\Phi(\gamma)$, we have that
\begin{align}\label{eq:bootstrap_last_2}
\begin{aligned}
\sup_{ \gamma \in \mathbb{R}}\sup_{|\xi|\leq \omega} \bigg|\mathbb{P}\bigg(\frac{\langle a, \theta_t \rangle - \langle a, \beta^* \rangle}{\sqrt{\Var\langle a, \theta_t \rangle}}\leq (1+\xi)\gamma\bigg) - \Phi(\gamma)\bigg| &\leq d^{true}_K +C\omega.
\end{aligned}
\end{align}
Combining \eqref{eq:bbostrap_last_1} and \eqref{eq:bootstrap_last_2} and recalling the bound on $\mathbb{E}[|\tfrac{\hat{V}_t}{\Var(\langle a, \theta_{t} \rangle)}-1|]$ from the first part of Theorem \ref{varestmainthm} yields the desired result and completes the proof. 
\\\\
It now remains to prove $\mathbf{(I)}$, $\mathbf{(II)}$ and $\mathbf{(III)}$  which we do below.
\\\\
\textsc{Proof of $(\mathbf{I})$:}
For ease of notation, we denote $u^{sub}_{i,k} := u_{i+1, s_k + t_0 - 1}$ where $s_k := t/2 + (k-1)t_0 + 1$. Now, observe that
\begin{align*}
    \mathbb{E}|\mathbf{V}_k - \hat{\mathbf{{V}}}_k| &= \mathbb{E}\bigg|\sum_{i = \frac{t}{2} + (k-1)t_0 + 1}^{\frac{t}{2} + kt_0} \eta^2_{i + \frac{t}{2} - kt_0} [(Y_i - X^\top_i \theta_{\frac{t}{2}})^2 - (Y_i - X^\top_i \beta^*)^2](X^\top_iu^{sub}_{i,k})^2\bigg|
    \\
    &= \mathbb{E}\bigg|\sum_{i = \frac{t}{2} + (k-1)t_0 + 1}^{\frac{t}{2} + kt_0} \eta^2_{i + \frac{t}{2} - kt_0} [[\epsilon_i + X^\top_i (\beta^* - \theta_{\frac{t}{2}})]^2 - \epsilon^2_i](X^\top_iu^{sub}_{i,k})^2\bigg|
    \\
    &= \mathbb{E}\bigg|\sum_{i = \frac{t}{2} + (k-1)t_0 + 1}^{\frac{t}{2} + kt_0} \eta^2_{i + \frac{t}{2} - kt_0} [(X^\top_i(\beta^* - \theta_{\frac{t}{2}}))^2 + 2\epsilon_i (X^\top_i(\beta^* - \theta_{\frac{t}{2}}))](X^\top_iu^{sub}_{i,k})^2\bigg|
    \\
    &= \mathbb{E}\bigg|\sum_{i = t-t_0 + 1}^{t} \eta^2_i [(X^\top_i(\beta^* - \theta_{\frac{t}{2}}))^2 + 2\epsilon_i (X^\top_i(\beta^* - \theta_{\frac{t}{2}}))](X^\top_iu_i)^2\bigg|
    \\
    &\leq \sum_{i = t-t_0 + 1}^t \eta^2_i \mathbb{E}[|(X^\top_i(\beta^* - \theta_{\frac{t}{2}}))^2 + 2\epsilon_i (X^\top_i(\beta^* - \theta_{\frac{t}{2}}))|(X^\top_i u_i)^2]
    \\
    &= \sum_{i = t-t_0 + 1}^t \eta^2_i \mathbb{E}_{\theta_{t/2}, u_i}\bigg[\mathbb{E}_{X_i, \epsilon_i}[|(X^\top_i(\beta^* - \theta_{\frac{t}{2}}))^2 + 2\epsilon_i (X^\top_i(\beta^* - \theta_{\frac{t}{2}}))|(X^\top_i u_i)^2]\bigg]
    \\
    &\leq \sum_{i = t-t_0 + 1}^t \eta^2_i \mathbb{E}_{\theta_{t/2}, u_i}\bigg[\mathbb{E}_{X_i, \epsilon_i}[|(X^\top_i(\beta^* - \theta_{\frac{t}{2}}))^2 + 2\epsilon_i (X^\top_i(\beta^* - \theta_{\frac{t}{2}}))|(X^\top_i u_i)^2]\bigg]
    \\
    &\leq \sum_{i = t-t_0 + 1}^t \eta^2_i \mathbb{E}_{\theta_{t/2}, u_i}\bigg[\mathbb{E}_{X_i}[(X^\top_i(\beta^* - \theta_{\frac{t}{2}}))^2(X^\top_i u_i)^2] + 2\mathbb{E}_{\epsilon_i, X_i}[|\epsilon_i X^\top_i(\beta^* - \theta_{\frac{t}{2}})|(X^\top_iu_i)^2]\bigg]
    \\
    &\leq \sum_{i = t-t_0 + 1}^t \eta^2_i \mathbb{E}_{\theta_{t/2}, u_i}\bigg[\mathbb{E}_{X_i}[(X^\top_i(\beta^* - \theta_{\frac{t}{2}}))^4]^{\frac{1}{2}}\mathbb{E}_{X_i}[(X^\top_i u_i)^4]^{\frac{1}{2}} 
    \\
    &\qquad \qquad \qquad \qquad \quad + 2\mathbb{E}[\epsilon^2_i]^{\frac{1}{2}}\mathbb{E}_{X_i}[(X^\top_i(\beta^* - \theta_{\frac{t}{2}}))^2]^{\frac{1}{2}}\mathbb{E}_{X_i}[(X^\top_iu_i)^4]^{\frac{1}{2}}\bigg]
    \\
    &\leq \sum_{t-t_0 + 1}^t \eta^2_{i}\mathbb{E}_{\theta_{t/2}, u_i}[\bar{\lambda}^2|\beta^* - \theta_{t/2}|^2|u_i|^2 + \sigma \bar{\lambda}^{\frac{3}{2}} |\beta^* - \theta_{t/2}||u_i|^2]
    \\
    &\leq (\bar{\lambda}^2\mathbb{E}|\beta^* - \theta_{t/2}|^2 + \sigma\bar{\lambda}^{\frac{3}{2}}\mathbb{E}|\beta^* - \theta_{t/2}|)\sum_{i = t-t_0 + 1}^t\eta^2_i\mathbb{E}|u_i|^2
    \\
    &\leq (\bar{\lambda}^2\mathbb{E}|\beta^* - \theta_{t/2}|^2 + \sigma\bar{\lambda}^{\frac{3}{2}}\mathbb{E}|\beta^* - \theta_{t/2}|)\sum_{i = t-t_0 + 1}^t\eta^2_i\mathbb{E}[|u_i|^4]^{\frac{1}{2}}
    \\
    &\leq C\eta^2_tt_0(\bar{\lambda}^2\mathbb{E}|\beta^* - \theta_{t/2}|^2 + \sigma\bar{\lambda}^{\frac{3}{2}}\mathbb{E}|\beta^* - \theta_{t/2}|)|a|^2
    \\
    &\leq C\eta^2_tt_0 (\bar{\lambda}^2d^{\frac{1}{2}}t^{-\alpha}(\eta \lambda_{\min}(A))^{-1} (\eta \sigma^2) + \sigma \bar{\lambda}^{\frac{3}{2}}d^{\frac{1}{4}}t^{-\frac{\alpha}{2}}(\eta \lambda_{\min}(A))^{-\frac{1}{2}}(\sqrt{\eta}\sigma))|a|^2
    \\
    &\leq C\eta^2_tt_0 \bar{\lambda}\sigma^2|a|^2(d^{\frac{1}{2}}t^{-\alpha}(\eta \lambda_{\min}(A))^{-1} + d^{\frac{1}{4}}t^{-\frac{\alpha}{2}}(\eta \lambda_{\min}(A))^{-\frac{1}{2}})
    \\
    &\leq C\eta^2_tt_0 \bar{\lambda}\sigma^2|a|^2 d^{\frac{1}{4}}t^{-\frac{\alpha}{2}}
    \\
    &\leq C\eta \sigma^2 |a|^2  (d^{-\frac{1}{2}}t^{-\alpha}) (d^{\frac{1}{4}}t^{-\frac{\alpha}{2}})(\log t + \log d)
    \\
    &\leq C\eta \sigma^2 |a|^2 (\log t + \log d)^2(d^{-\frac{1}{4}}t^{-\frac{3\alpha}{2}}),
\end{align*}
as desired. Here,
\begin{itemize}
    \item The fourth line follows from the fact that $X_i's$ are i.i.d.
    \item The tenth line follows from moment  Assumptions \ref{assump} on $X_i$ and $\epsilon_i$; and the facts that $\beta^* - \theta_{t/2}$ and $u_i$ are independent of $X_i$ (for $t-t_0 + 1 \leq i \leq t$).
    \item The eleventh line follows from the independence of $\beta^* - \theta_{\frac{t}{2}}$ and $u_i$.
    \item The thirteenth line follows from Lemma \ref{conc1} and Assumption \ref{assump} $$\lim\limits_{t,d \to \infty} \frac{t_0}{t} \leq \lim\limits_{t,d \to \infty}C(\eta \lambda_{\min}(A))^{-1}(\log t + \log d)^2 d^{\frac{1}{2}}t^{-(1 - \alpha)} = 0 \implies \eta^2_{(t-t_0)} \leq C \eta^2_t$$
    \item The fourteenth line follows from Lemma \ref{euclidean-norm-bound}.
    \item The fifteenth line follows from Assumption \ref{assump} that $\eta \bar{\lambda} < C$.
    \item The sixteenth line follows from Assumption \ref{assump-varest} that $\eta \lambda_{\min}(A) > c$. 
\end{itemize}
\textsc{Proof of $(\mathbf{II})$:} 
For ease of notation, let $u_i := [\prod_{j = i+1}^{t}(I - \eta_{j}X_{j}X^\top_{j})]a$.
Recall that $$\epsilon_i := Y_i - X^\top_i \beta^* \quad \text{and} \quad \mathbf{V} := \sum_{i = t-t_0+1}^t \eta^2_i \epsilon^2_i (u^\top_iX_i)^2.$$ Now observe that
\begin{align*}
\mathbb{E}[V_t - \mathbb{E}[V_t]]^2 &= \frac{4t^2_0\sum_{k = 1}^{\frac{t}{2t_0}}\Var [\mathbf{V}_k]}{t^2} 
\\
&= \frac{2t_0\Var[\mathbf{V}]}{t}
\\
&\leq \frac{2t_0\mathbb{E}[\mathbf{V}^2]}{t}
\\
&\leq \frac{2t_0 \mathbb{E}[\sum_{i = t-t_0 + 1}^t \eta^2_i\epsilon^2_i(u^\top_iX_i)^2]^2}{t}
\\
&\leq \frac{2t_0^2\mathbb{E}[\sum_{i = t-t_0 + 1}^t\eta^4_i\epsilon^4_i(u^\top_iX_i)^4]}{t}
\\
&\leq \frac{(Ct_0^2\eta^4_t)\mathbb{E}[\sum_{i = t-t_0 + 1}^t\epsilon^4_i(u^\top_iX_i)^4]}{t}
\\
&\leq \frac{(Ct^3_0\eta^4_t)\sigma^4\bar{\lambda}^2\sum_{i = t-t_0 + 1}^t \mathbb{E}[|u_i|^4]}{t}
\\
&\leq \frac{(Ct^3_0\eta^4_t)\sigma^4\bar{\lambda}^2|a|^4}{t}
\\
&\leq \frac{Ct^3_0\eta^2\sigma^4|a|^4}{d^2t^{1 + 4\alpha}}
\\
&\leq \frac{C\eta^2\sigma^4|a|^4(\log t + \log d)^6t^{3\alpha}d^{\frac{3}{2}}}{d^2t^{1 + 4\alpha}}
\\
&\leq C\eta^2\sigma^4|a|^4(\log t + \log d)^6 t^{-1-\alpha}d^{-\frac{1}{2}}, 
\end{align*}
as desired. Here,
\begin{itemize}
    \item The second line follows from the observations that $V_k's$ are i.i.d and the $X_i's$ are also i.i.d.
    \item The sixth line follows from Assumption \ref{assump} $$\lim\limits_{t,d \to \infty} \frac{t_0}{t} \leq \lim\limits_{t,d \to \infty}C(\eta \lambda_{\min}(A))^{-1}(\log t + \log d)^2 d^{\frac{1}{2}}t^{-(1 - \alpha)} = 0 \implies \eta^2_{(t-t_0)} \leq C \eta^2_t.$$
    \item The seventh line follows from moment Assumptions \ref{assump} on $\epsilon_i, X_i$ and the independence of $u_i$ and $X_i$.
    \item The eighth line follows from Lemma \ref{conc1}.
    \item The ninth line follows from Assumption \ref{assump} that $\eta \bar{\lambda} < C$
\end{itemize} 
\textsc{Proof of $(\mathbf{III})$:} 
Lemma \ref{variance-sigma-term-2} and Lemma \ref{variance-sigma-term} give us that
\begin{align*}
|\mathbb{E}[V_t] - \Var \langle a, \theta_t \rangle| \leq |\mathcal E|\Var \langle a, \theta_t \rangle,
\end{align*}
as desired. Here $\mathcal E$ is the same error term that appears in Theorem \ref{lem:VarAsymp}. 
\\\\
Thus we have proved all claims and are done.
\end{proof}   
\begin{proof}[Proof of Theorem~\ref{lem:VarAsymp}]
Throughout the proof, we let $C  > 0$ and $c > 0$ respectively denote large and small enough generic absolute constants.
\\\\
Recall from the note below Theorem \ref{ppn:Mourrat} that $\Var \langle a, \theta_t \rangle = \sum_{i = 1}^t \mathbb{E}[M_{t-i+1}^2]$, where
\begin{align*}
    M_{t-i+1} = \mathbb{E}(\langle a, \theta_t \rangle  | X_t,\epsilon_t,X_{t-1}, \epsilon_{t-1},\dots X_{i}, \epsilon_{i})-\mathbb{E}(\langle a, \theta_t \rangle  | X_t,\epsilon_t,X_{t-1}, \epsilon_{t-1},\dots X_{i+1}, \epsilon_{i+1}).
\end{align*}
is the martingale difference sequence defined in Lemma \ref{lemma:martingalerep}. Futher, recall from the proof of Lemma \ref{lem:VandS} that $\mathbb{E}[M_{t-i+1}^2] = \mathcal D_i$, where
\begin{align*}
    \mathcal D_i = \eta_i^2 \mathbb{E}\langle u_i, (X_iX^\top_i - A)v_i + \epsilon_iX_i \rangle^2,
\end{align*}
$u_i := [\prod_{j = i+1}^{t}(I - \eta_{j}X_{j}X^\top_{j})]a$ and $v_i := [\prod_{j = 1}^{i-1}(I - \eta_{i-j}A)](\beta^* - \theta_0)$.
\\\\
Now observe that
\begin{align}
\Var \langle a, \theta_t \rangle &= \sum_{i = 1}^t \eta_i^2 \mathbb{E}\langle u_i, (X_iX^\top_i - A)v_i + \epsilon_i X_i \rangle^2
\\
&= \sum_{i = 1}^t \eta_i^2 \mathbb{E}[(u_i^\top (X_iX^\top_i - A)v_i)^2 + \epsilon^2_i(u_i^\top X_i)^2 - 2(u_i^\top A v_i)u^\top_i(\epsilon_i X_i) + 2\epsilon_i (u_i^\top X_i)^2(X^\top_iv_i)]
\\
&= \sum_{i = 1}^t \eta_i^2 \mathbb{E}[(u_i^\top (X_iX^\top_i - A)v_i)^2 + \epsilon^2_i(u_i^\top X_i)^2 + 2\epsilon_i (u_i^\top X_i)^2(X^\top_iv_i)],\label{eqn:variance-equation}
\end{align}
where the last inequality follows using the standard fact that $\mathbb{E}[\epsilon_i X_i] = \mathbb{E}[(Y - X^\top \beta^*)X] = 0$. Now, recall that $A_{\sigma} := \mathbb{E}[\epsilon^2XX^\top]$, we get that
\begin{align}
\Var \langle a, \theta_t \rangle = \sum_{i = 1}^t \eta^2_i\mathbb{E}[u^\top_i A_{\sigma} u_i] + \mathcal E_1 + \mathcal E_2,\label{eqn:variance-simplification}
\end{align}
where $|\mathcal E_1| \leq \sum_{i = 1}^t \eta^2_i \mathbb{E}(u^\top_i (X_iX^\top_i - A)v_i)^2$ and $|\mathcal E_2| \leq \sum_{i = 1}^t2\eta^2_i|\mathbb{E}[\epsilon_i (u^\top_i X_i)^2 (X_i^\top v_i)]|)$. Now, Lemma \ref{variance-sigma-term} gives us that
\begin{align*}
\sum_{i = 1}^t \eta^2_i\mathbb{E}[u^\top_i A_{\sigma} u_i] = (1 + \mathcal E)\eta d^{-\frac{1}{2}}t^{-\alpha} \sum_{k,k'=1}^d \frac{a_{k}a_{k'}[A_{\sigma}]_{k,k'}}{\lambda_k + \lambda_{k'}},
\end{align*}
where $|\mathcal E| \leq  C(\log t + \log d)^2 [(\eta \lambda_{\min}(A))^{-1}d^{\frac{1}{2}}t^{-(1-\alpha)} + (\eta \lambda_{\min}(A))^{-3}\sigma^2\sigma_{\min}^{-2}d^{\frac{1}{2}}t^{-\alpha}]$. Further, Lemma \ref{X-moment-upper-bound} and Lemma \ref{variance-exp-term} together give us that
\begin{align*}
|\mathcal E_1| &\leq \sum_{i = 1}^t \eta^2_i\mathbb{E}(u^\top_i(X_iX^\top_i - A)v_i)^2
\\
&\leq C\bar{\lambda}^2 \sum_{i = 1}^t \eta^2_i\mathbb{E}|u_i|^2 |v_i|^2
\\
&\leq C\eta^2 \bar{\lambda}^2 d^{-1} e^{-2\eta \lambda_{\min}(A) d^{-\frac{1}{2}}t^{1-\alpha}}|\beta^* - \theta_0|^2 |a|^2
\\
&\leq C d^{-1} e^{-2\eta \lambda_{\min}(A) d^{-\frac{1}{2}}t^{1-\alpha}}|\beta^* - \theta_0|^2 |a|^2.
\end{align*}
Here the last inequality follows from assumptions \ref{assump} that $\eta \bar{\lambda} < C$. 
\\\\
Also, Lemma \ref{variance-non-linear-term} gives us that
\begin{align*}
|\mathcal E_2| &\leq C\sum_{i = 1}^t\eta^2_i\mathbb{E}[\epsilon_i (u^\top_iX_i)^2 (X^\top_iv_i)]
\\
&\leq Cd^{-1}(\sigma \sqrt{\eta}) |a|^2 |\beta^* - \theta_0|e^{-\eta \lambda_{\min}(A)d^{-\frac{1}{2}}t^{1-\alpha}}
\end{align*}
Now let $$\mathbf{U}:= C(\log t + \log d)^2 [(\eta \lambda_{\min}(A))^{-1}d^{\frac{1}{2}}t^{-(1-\alpha)} + (\eta \lambda_{\min}(A))^{-3}\sigma^2\sigma_{\min}^{-2}d^{\frac{1}{2}}t^{-\alpha}].$$ 
We show below that
\begin{align*}
    |\mathcal E_1| + |\mathcal E_2| \leq \mathbf{U}\cdot\eta d^{-\frac{1}{2}}t^{-\alpha}\sum_{k,k' = 1}^d \frac{a_k a_{k'}[A_{\sigma}]_{k,k'}}{\lambda_k + \lambda_{k'}}
\end{align*}
for all $t, d \geq C$.
\\\\
To prove this, first observe that
\begin{align*}
\mathbf{R} &:= \mathbf{U}\cdot\eta d^{-\frac{1}{2}}t^{-\alpha}\sum_{k,k' = 1}^d \frac{a_k a_{k'}[A_{\sigma}]_{k,k'}}{\lambda_k + \lambda_{k'}}
\\
&\geq \mathbf{U}. \eta d^{-\frac{1}{2}}t^{-\alpha} (a^\top A_{\sigma} a) (2 \lambda_{\max}(A))^{-1}
\\
&\geq \mathbf{U}. [c(\eta\sigma^2_{\min})\lambda_{\min}(A) |a|^2d^{-\frac{1}{2}} t^{-\alpha} (\lambda_{\max}(A))^{-1}]
\\
&\geq \mathbf{U}. [c(\eta\sigma^2_{\min})(\eta\lambda_{\min}(A)) |a|^2d^{-\frac{1}{2}} t^{-\alpha} (\eta\lambda_{\max}(A))^{-1}]
\\
&\geq \mathbf{U}. [c(\eta\sigma^2_{\min})(\eta\lambda_{\min}(A)) |a|^2d^{-\frac{1}{2}} t^{-\alpha}]
\\
&\geq \mathbf{U}. [c (\eta \lambda_{\min}(A))^{-2} (\eta \sigma^2)t^{-2\alpha}|a|^2]
\\
&\geq \mathbf{U}.[c(\eta \sigma^2)t^{-2\alpha}|a|^2]
\\
&\geq [d^{\frac{1}{2}}t^{-(1 - \alpha)}][c.(\eta \sigma^2)t^{-2\alpha}]|a|^2
\\
&\geq (td)^{-C} |a|^2 |\beta^* - \theta_0|^2.
\end{align*}
Here,
\begin{itemize}
    \item The third line follows from Lemma \ref{noise-moment-lower-bound}.
    \item The fifth line follows from Assumption \ref{assump} using $\eta \lambda_{\max}(A) < \eta \bar{\lambda} < C$.
    \item The sixth line follows from Assumption \ref{assump} using 
    $$\lim\limits_{t,d \to \infty} (\eta \lambda_{\min}(A))^{-3} (\sigma^2\sigma^{-2}_{\min})(\log t + \log d)^2d^{\frac{1}{2}}t^{-\alpha} = 0.$$
    \item The seventh line follows from Assumption \ref{assump} that $\eta \lambda_{\min}(A) < \eta \lambda_{\max}(A) < \eta \bar{\lambda} < C$.
    \item The last line follows from Assumption \ref{assump} that
    $$\frac{|\beta^* - \theta_0|^2}{\eta \sigma^2} < (td)^{C} \quad \forall \quad  t, d \geq C.$$
\end{itemize} 
On the other hand, we have
\begin{align*}
|\mathcal E_1| + |\mathcal E_2| &\leq Cd^{-1}e^{-\eta \lambda_{\min}(A)d^{-\frac{1}{2}}t^{-\alpha}}|a|^2 (\sigma\sqrt{\eta} |\beta^* - \theta_0| + |\beta^* - \theta_0|^2)
\\
&\leq C|a|^2|\beta^* - \theta_0|^2 (td)^C e^{-\eta \lambda_{\min}(A)d^{-\frac{1}{2}}t^{1-\alpha}}
\\
&\leq C|a|^2|\beta^* - \theta_0|^2 (td)^C e^{-c(\log t + \log d)^2},
\end{align*}
Here, 
\begin{itemize}
    \item The second inequality followed from Assumption \ref{assump} as $$\frac{|\beta^* - \theta_0|^2}{\eta \sigma^2} < (td)^{C} \quad \forall \quad  t, d \geq C.$$
    \item The third inequality followed from Assumption \ref{assump} as $$\lim\limits_{t,d \to \infty} (\eta \lambda_{\min}(A))^{-1}(\log t + \log d)^2d^{\frac{1}{2}}t^{-(1 - \alpha)} = 0.$$
\end{itemize}
These imply that $|\mathcal E_1| + |\mathcal E_2| < \mathbf{R}$ for all large enough $t,d$. Combining this with equation \ref{eqn:variance-equation} and \ref{eqn:variance-simplification} from earlier gives us the desired result.
\end{proof}
\begin{proof}[Proof of Lemma \ref{variance-sigma-term-2}.]
Throughout the proof, we let $C  > 0$ and $c > 0$ respectively denote large and small enough generic absolute constants.
\\\\
Recall the notation from Lemma \ref{variance-sigma-term} and the definition $t_0 : = t^{\alpha}d^{\frac{1}{2}}(\log t + \log d)^2$. We want to show that
\begin{align*}
    \sum_{i = t-t_0 + 1}^t \eta^2_i \mathbb{E}[u^\top_i A_{\sigma} u_i] &= (1 + \mathcal E) \eta d^{-\frac{1}{2}}t^{-\alpha}\sum_{k,k' = 1}^d \frac{a_k a_{k'}[A_{\sigma}]_{k,k'}}{\lambda_k + \lambda_{k'}}
\end{align*}
where $|\mathcal E| \leq C(\log t + \log d)^2 [(\eta \lambda_{\min}(A))^{-1}d^{\frac{1}{2}}t^{-(1-\alpha)} + (\eta\lambda_{\min}(A))^{-3}\sigma^2\sigma^{-2}_{\min}d^{\frac{1}{2}}t^{-\alpha}]$.
\\\\
For notational convenience, we also define $$\mathbf{U} := C(\log t + \log d)^2 [(\eta \lambda_{\min}(A))^{-1}d^{\frac{1}{2}}t^{-(1-\alpha)} + (\eta\lambda_{\min}(A))^{-3}\sigma^2\sigma^{-2}_{\min}d^{\frac{1}{2}}t^{-\alpha}]$$
and  
$$\mathbf{R} :=  \mathbf{U}\cdot\eta d^{-\frac{1}{2}}t^{-\alpha}\sum_{k,k' = 1}^d \frac{a_k a_{k'}[A_{\sigma}]_{k,k'}}{\lambda_k + \lambda_{k'}}.$$
Since Lemma \ref{variance-sigma-term} already shows that $$\sum_{i = 1}^t \eta^2_i \mathbb{E}[u^\top_i A_{\sigma} u_i] = (1 + \mathcal E)\eta d^{-\frac{1}{2}}t^{-\alpha}\sum_{k,k' = 1}^d \frac{a_k a_{k'}[A_{\sigma}]_{k,k'}}{\lambda_k + \lambda_{k'}},$$ it suffices to show that $$\sum_{i = 1}^{t-t_0} \eta^2_i \mathbb{E}[u^\top_i A_{\sigma} u_i] < \mathbf{R} \quad \forall \quad t,d \geq C.$$
To this end, observe that
\begin{align*}
    \sum_{i = 1}^{t-t_0}\eta^2_i\mathbb{E}[u^\top_i A_{\sigma} u_i] &= \sum_{i = 1}^{t-t_0} \eta^2_i \mathbb{E}[\epsilon^2_i(u^\top_i X_i)^2]
    \\
    &\leq \sum_{i = 1}^{t-t_0}\eta_i^2 \mathbb{E}[\epsilon^4_i]^{\frac{1}{2}}\mathbb{E}[(u^\top_iX_i)^4]^{\frac{1}{2}}
    \\
    &\leq C(\sigma^2 \bar{\lambda})\sum_{i = 1}^{t-t_0}\eta^2_i  \mathbb{E}[|u^4_i|]^{\frac{1}{2}}
    \\
    &\leq C(\sigma^2 \bar{\lambda})\sum_{i = 1}^{t-t_0}\eta^2_ie^{-2\eta \lambda_{\min}(A)d^{-\frac{1}{2}}\sum_{j = i+1}^t j^{-\alpha}}|a|^2
    \\
    &\leq C(\sigma^2 \bar{\lambda})\sum_{i = 1}^{t-t_0}\eta^2_ie^{-2\eta \lambda_{\min}(A)d^{-\frac{1}{2}}t^{-\alpha}t_0}|a|^2
    \\
    &\leq Ce^{-c(\log t + \log d)^2} (\eta \sigma^2)|a|^2,
\end{align*}
where $C, c > 0$ are absolute constants. Here the third line follows from Assumption \ref{assump} on $\epsilon_i$ and $X_i$ ,the fourth line follows from Lemma \ref{conc1} and the last line follows from Assumption \ref{assump-varest} that $\eta \lambda_{\min}(A) > c$ for an absolute constant $c > 0$. 
\\\\
But, observe that
\begin{align*}
\mathbf{R} &:= \mathbf{U}\cdot\eta d^{-\frac{1}{2}}t^{-\alpha}\sum_{k,k' = 1}^d \frac{a_k a_{k'}[A_{\sigma}]_{k,k'}}{\lambda_k + \lambda_{k'}}
\\
&\geq \mathbf{U}. \eta d^{-\frac{1}{2}}t^{-\alpha} (a^\top A_{\sigma} a) (2 \lambda_{\max}(A))^{-1}
\\
&\geq \mathbf{U}. [c(\eta\sigma^2_{\min})\lambda_{\min}(A) |a|^2d^{-\frac{1}{2}} t^{-\alpha} (\lambda_{\max}(A))^{-1}]
\\
&\geq \mathbf{U}. [c(\eta\sigma^2_{\min})(\eta\lambda_{\min}(A)) |a|^2d^{-\frac{1}{2}} t^{-\alpha} (\eta\lambda_{\max}(A))^{-1}]
\\
&\geq \mathbf{U}. [c(\eta\sigma^2_{\min})(\eta\lambda_{\min}(A)) |a|^2d^{-\frac{1}{2}} t^{-\alpha}]
\\
&\geq \mathbf{U}. [c (\eta \lambda_{\min}(A))^{-2} (\eta \sigma^2)t^{-2\alpha}|a|^2]
\\
&\geq \mathbf{U}.[c(\eta \sigma^2)t^{-2\alpha}|a|^2]
\\
&\geq [d^{\frac{1}{2}}t^{-(1 - \alpha)}][c.(\eta \sigma^2)t^{-2\alpha}]|a|^2,
\end{align*}
Here,
\begin{itemize}
    \item The third line follows from Lemma \ref{noise-moment-lower-bound}.
    \item  The fifth line follows from Assumption \ref{assump} as $$\eta \lambda_{\max}(A) < \eta \bar{\lambda} < C.$$
    \item The sixth line from Assumption \ref{assump} as
    $$\lim\limits_{t,d \to \infty} (\eta \lambda_{\min}(A))^{-3} (\sigma^2\sigma^{-2}_{\min})(\log t + \log d)^2d^{\frac{1}{2}}t^{-\alpha} = 0.$$
    \item The seventh line follows from Assumption \ref{assump} as $\eta \lambda_{\min}(A) < \eta \lambda_{\max}(A) < \eta \bar{\lambda} < C$
\end{itemize}
Together, these imply that $\sum_{i = 1}^{t-t_0}\eta^2_i\mathbb{E}[u^\top_i A_{\sigma}u_i]$ becomes smaller than $\mathbf{R}$ for all large enough $t,d$. Hence we are done. 
\end{proof}
\section{Auxiliary Results for the CLT Proof}\label{sec:app3}
The following results were used at many places in Sections \ref{useful}, \ref{first-term}, \ref{second-term}.
\subsection{Concentration Inequalities For Zero Mean Fluctuation}
\begin{lemma}\label{uniformconvexity3}
There exists an absolute constant $C$ such that for any random variable $U, V \in \mathbb{R}$ satisfying $\mathbb{E}[V|U] = 0$ and $p \geq 2$, we have  
\begin{align*}
\mathbb{E}[|U + V|^p]^{\frac{2}{p}} \leq \mathbb{E}[|U|^p]^{\frac{2}{p}} + C(p-1)\mathbb{E}[|V|^p]^{\frac{2}{p}}
\end{align*}
(This is  a special case of a similar inequality for Schatten-$p$ norms of random matrices, refer \textbf{Proposition 4.3} from \cite{huang2022matrix}).
\end{lemma}
\begin{proof}
To begin with, observe that since $p \geq 2$, we have
\begin{align*}
\frac{\mathbb{E}[|U+V|^p]^{\frac{2}{p}} + \mathbb{E}[|U - V|^p]^{\frac{2}{p}}}{2} &\leq \bigg(\frac{\mathbb{E}[|U+V|^p] + \mathbb{E}[|U-V|^p]}{2}\bigg)^{\frac{2}{p}}
\\
&\leq \mathbb{E}[|U|^p]^{\frac{2}{p}} + C'(p-1)\mathbb{E}[|V|^p]^{\frac{2}{p}}
\end{align*}
for an absolute constant $C'$. Here the first inequality follows from Jensen's on the function $x \to x^{\frac{2}{p}}$ and the second inequality follows from Lemma \ref{uniformconvexity2}. Next, observe that for any fixed $u \in \mathbb{R}$ and $p \geq 2$, the function $f_u(v) := |u - v|^p$ satisfies $\frac{\partial^2(f_u(v))}{\partial v^2} > 0$.
\\
Therefore, applying Jensen's inequality \bigg($\frac{\partial^2 f}{\partial x^2} > 0 \implies \mathbb{E}_{X}[f(X)] \geq f(\mathbb{E}[X])$\bigg) tells us that
\begin{align*}
\mathbb{E}[|U - V|^p] &= \mathbb{E}_{U}(\mathbb{E}_V |U - V|^p | U)
\\
&\geq \mathbb{E}_U(|U - \mathbb{E}[V|U]|^p | U)
\\
&= \mathbb{E}_U |U|^p
\end{align*}
Finally, these inequalities together imply that
\begin{align*}
\frac{\mathbb{E}[|U + V|^p]^{\frac{2}{p}} + \mathbb{E}[|U|^p]^{\frac{2}{p}}}{2} \leq \frac{\mathbb{E}[|U+V|^p]^{\frac{2}{p}} + \mathbb{E}[|U - V|^p]^{\frac{2}{p}}}{2} \leq \mathbb{E}[|U|^p]^{\frac{2}{p}} + C'(p-1)\mathbb{E}[|V|^p]^{\frac{2}{p}}
\end{align*}
for an absolute constant $C'$. Rearranging terms gives us that
\begin{align*}
\mathbb{E}[|U + V|^p]^{\frac{2}{p}} \leq \mathbb{E}[|U|^p]^{\frac{2}{p}} + 2C'(p-1) \mathbb{E}[|V|^p]^{\frac{2}{p}} = \mathbb{E}[|U|^p]^{\frac{2}{p}} + C(p-1) \mathbb{E}[|V|^p]^{\frac{2}{p}} 
\end{align*}
for an absolute constant $C$, as desired.
\end{proof}

\begin{lemma}\label{uniformconvexity2}
There exists an absolute constant $C$ so that for any random variables $U,V \in \mathbb{R}$ and $p \geq 2$, we have
\begin{align*}
\bigg(\frac{\mathbb{E}|U+V|^p + \mathbb{E}|U-V|^p}{2}\bigg)^{\frac{2}{p}} \leq \mathbb{E}[|U|^p]^{\frac{2}{p}} + C(p-1)\mathbb{E}[|V|^p]^{\frac{2}{p}}
\end{align*}
(This is a special case of a similar inequality for Schatten-$p$ norms of random matrices, refer \textbf{Corollary 4.2} from \cite{huang2022matrix}).
\end{lemma}
\begin{proof}
Raising both sides of Lemma \ref{uniformconvexity1} to the $\frac{p}{2}$ power tells us that
\begin{align*}
\frac{|a + b|^p + |a - b|^p}{2} \leq (a^2 + C(p-1)b^2)^{\frac{p}{2}}
\end{align*}
for all $a, b \in \mathbb{R}$ and $p \geq 2$. Substituting $a \to U$ and $b \to V$ and taking the expectation of both sides gives us that for any random variables $U, V \in \mathbb{R}$ and $p \geq 2$, we have
\begin{align*}
\frac{\mathbb{E}|U+V|^p + \mathbb{E}|U-V|^p}{2} \leq \mathbb{E}[(U^2 + C(p-1)V^2)^{\frac{p}{2}}]
\end{align*}
For a random variable $W \in \mathbb{R}$ and $n \geq 1$, let $\|W\|_n := \mathbb{E}[|W|^n]^{\frac{1}{n}}$. Minkowski's inequality for $L^n$ spaces gives us that $\|W_1 + W_2\|_n \leq \|W_1\|_n + \|W_2\|_n$ for any $n \geq 1$ and random variables $W_1, W_2 \in \mathbb{R}$. Using this with $W_1 := U^2$, $W_2 := C(p-1)V^2$ and $n := \frac{p}{2}$, we get
\begin{align*}
\mathbb{E}[(U^2 + C(p-1)V^2)^{\frac{p}{2}}]^{\frac{2}{p}} &= \|U^2 + C(p-1)V^2\|_{\frac{p}{2}}
\\
&\leq \|U^2\|_{\frac{p}{2}} + C(p-1)\|V^2\|_{\frac{p}{2}}
\\
&= \mathbb{E}[|U|^p]^{\frac{2}{p}} + C(p-1)\mathbb{E}[|V|^p]^{\frac{2}{p}}
\end{align*}
Together, these imply that
\begin{align*}
\bigg(\frac{\mathbb{E}|U + V|^p + \mathbb{E}|U-V|^p}{2}\bigg)^{\frac{2}{p}} &\leq \mathbb{E}[(U^2 + C(p-1)V^2)^{\frac{p}{2}}]^{\frac{2}{p}} 
\\
&\leq \mathbb{E}[|U|^p]^{\frac{2}{p}} + C(p-1)\mathbb{E}[|V|^p]^{\frac{2}{p}}
\end{align*}
as desired.
\end{proof}
\begin{lemma}\label{uniformconvexity1}
There exists an absolute constant $C>0$ so that for every $a, b \in \mathbb{R}$ and $p \geq 2$, we have that
\begin{align*}
\bigg(\frac{|a + b|^p + |a-b|^p}{2}\bigg)^{\frac{2}{p}} \leq a^2 + C(p-1)b^2.
\end{align*}
(This is a special case of the uniform smoothness property of Schatten classes, refer \textbf{Fact 4.1} from \cite{huang2022matrix}).  
\end{lemma}
\begin{proof}
If $|a| \leq |b|$ then since $p \geq 2$, $a^2 + (p-1)b^2 \geq b^2 + (p-1)a^2$. We may therefore assume that $|a| > |b| \geq 0$. Set $x = b/a$ and observe that $x \in [-1,1]$. We now wish to show that
\begin{align*}
\bigg(\frac{(1 + x)^p + (1 - x)^p}{2}\bigg) \leq (1 + C(p-1)x^2)^{\frac{p}{2}}
\end{align*}
for an absolute constant $C$. Substitute $p = 2m$, this is now equivalent to showing that there exists an absolute constant $C$ so that
\begin{align*}
\frac{(1 + x)^{2m} + (1 - x)^{2m}}{2} \leq (1 + C(2m-1)x^2)^m
\end{align*}
for all $m \geq 1$. 
\\\\
\textbf{Proof for integer $m$ :}
We will first show that for all integers $m \geq 1$, the inequality
\begin{align*}
\frac{(1 + x)^{2m} + (1-x)^{2m}}{2} \leq (1 + (2m-1)x^2)^m
\end{align*}
holds. To see this, observe that the above is equivalent to
\begin{align*}
\sum_{k = 0}^{m}\binom{2m}{2k}x^{2k} \leq \sum_{k = 0}^m \binom{m}{k} (2m-1)^k x^{2k}
\end{align*}
It therefore suffices to show that $\binom{2m}{2k} \leq (2m-1)^k \binom{m}{k}$ for all $0 \leq k \leq m$. This clearly holds for $k = 0$ so we may assume $1 \leq k \leq m$. Now, observe that
\begin{align*}
\frac{\binom{2m}{2k}}{\binom{m}{k}} &= \frac{2m (2m-1) \dots (2m - 2k + 1)}{m (m-1) \dots (m - k + 1)} \times \frac{k!}{(2k)!}
\\
&= \frac{2^k (2m - 1)(2m - 3)\dots (2m - (2k-1))}{(k+1)\dots (2k)}
\\
&\leq (2m - 1)(2m - 3) \dots (2m - (2k - 1))
\\
&\leq (2m - 1)^k
\end{align*}
as desired. 
\\\\
\textbf{Proof for non-integer $m$ :}
We will first show that the function $m \to \frac{(1 + x)^{2m} + (1 - x)^{2m}}{2}$ is increasing in $m$ for $m \geq 1$. To prove this, observe that for $1 \leq m_1 \leq m_2$, we have by Lyapunov's inequality that
\begin{align*}
\bigg(\frac{(1 + x)^{2m_1} + (1 - x)^{2m_1}}{2}\bigg)^{\frac{1}{m_1}} \leq \bigg(\frac{(1 + x)^{2m_2} + (1 - x)^{2m_2}}{2}\bigg)^{\frac{1}{m_2}} 
\end{align*}
Since $2m_2 \geq 1$, we also have that
\begin{align*}
\frac{(1 + x)^{2m_2} + (1 - x)^{2m_2}}{2} \geq \bigg(\frac{(1 + x) + (1 - x)}{2}\bigg)^{2m_2} = 1
\end{align*}
Together, these imply that
\begin{align*}
\frac{(1 + x)^{2m_1} + (1 - x)^{2m_1}}{2} &\leq \bigg(\frac{(1 + x)^{2m_2} + (1 - x)^{2m_2}}{2}\bigg)^{\frac{m_1}{m_2}}
\\
&\leq \bigg(\frac{(1 + x)^{2m_2} + (1 - x)^{2m_2}}{2}\bigg)^{\frac{m_1}{m_2}} \times \bigg(\frac{(1 + x)^{2m_2} + (1 - x)^{2m_2}}{2}\bigg)^{1 - \frac{m_1}{m_2}}
\\
&= \frac{(1 + x)^{2m_2} + (1 - x)^{2m_2}}{2}
\end{align*}
showing that $\frac{(1 + x)^{2m} + (1 - x)^{2m}}{2}$ is indeed increasing in $m$.
\\\\
Now take any non-integer $m > 1$ and let $n = \lceil m \rceil$. Clearly $m < n < m + 1$. By the above, we get that
\begin{align*}
\frac{(1 + x)^{2m} + (1 - x)^{2m}}{2} &\leq \frac{(1 + x)^{2n} + (1 - x)^{2n}}{2}
\\
&\leq (1 + (2n-1)x^2)^n
\\
&< (1 + (2m + 1)x^2)^{m+1}
\\
& = [(1 + (2m+1)x^2)(1 + (2m+1)x^2)^{\frac{1}{m}}]^m
\\
&\leq \bigg(\bigg(1 + (2m+1)x^2\bigg)\bigg(1 + \frac{(2m+1)x^2}{m}\bigg)\bigg)^{m}
\\
&\leq ((1 + C(2m-1)x^2)(1 + Cx^2))^m
\\
&\leq (1 + C(2m-1)x^2 + C(2m-1)x^4)^m
\\
&\leq (1 + C(2m-1)x^2)^m
\end{align*}
as desired, where the last inequality follows from the fact that $|x| \leq 1$.
\end{proof}
\subsection{Matrix Spectral Norm Bounds.}

\begin{lemma}\label{conchelper}
Let $\mathcal A_i$ be a positive definite, symmetric matrix.
We define the sequence $\mathcal A_{i,i}, \mathcal A_{i,i+1}, \dots \mathcal A_{i,t}$ recursively as follows:
\begin{enumerate}
    \item Set the initial term: $\mathcal A_{i,i} := \mathcal A_i$.
    \item For all $i+1 \leq k \leq t$, the subsequent terms are given by:
    \begin{align*}
    \mathcal A_{i,k} := \mathbb{E}_X\left[(I - \eta_{k}XX^\top)\mathcal A_{i,k-1}(I - \eta_{k}XX^\top)\right]
    \end{align*}
\end{enumerate}
Under Assumption \ref{assump}, we have for all $t,d \geq C_1$ that
\begin{align*}
\lambda_{\max}(\mathcal A_{i,k}) < C_2e^{-2\lambda_{\min}(A)\sum_{j = i+1}^k\eta_j} \lambda_{\max}(\mathcal A_i),
\end{align*}
for all $i+1 \leq k \leq t$. Here $C_1, C_2 > 0$ are absolute constants.
\end{lemma}
\begin{proof}
For the rest of the proof, we let $C > 0$ denote a sufficiently large and generic absolute constant.
\\\\
To begin with, observe that for any positive-definite, symmetric matrix $\mathcal A$, we have that
\begin{align*}
& \quad \lambda_{\max}(\mathbb{E}_X[(I - \eta_kX_kX_k^\top) \mathcal A (I - \eta_k X_kX_k^\top)])
\\
&= \sup_{u \in \mathbb{R}^d, |u| = 1} u^\top\mathbb{E}_X[(I -  \eta_kX_kX_k^\top)\mathcal A(I - \eta_kX_kX_k^\top)] u 
\\
&= \sup_{u \in \mathbb{R}^d, |u| = 1} \mathbb{E}_X [[(I - \eta_kX_kX_k^\top)u]^\top \mathcal A [(I - \eta_kX_kX_k^\top)u]]
\\
&\leq \sup_{u \in \mathbb{R}^d, |u| = 1} \lambda_{\max}(\mathcal A) \mathbb{E}_X |(I - \eta_kXX^\top)u|^2
\\
&\leq \sup_{u \in \mathbb{R}^d, |u| = 1} \lambda_{\max}(\mathcal A)(1 - 2\eta_k \lambda_{\min}(A) + d\eta^2_k\bar{\lambda}^2)|u|^2
\\
&\leq  \lambda_{\max}(\mathcal A)(1 - 2\eta_k\lambda_{\min}(A) + d\eta^2_k\bar{\lambda}^2)
\\
&\leq \lambda_{\max}(\mathcal A) e^{-2\eta_k \lambda_{\min}(A) + d\eta^2_k\bar{\lambda}^2}
\end{align*}
Here the third-last inequality follows using Lemma \ref{momenthelper2}. Using this recursively gives us for all $i+1 \leq k \leq t$ that
\begin{align*}
\lambda_{\max}(\mathcal A_{i, k}) &< e^{-2\eta_{k}\lambda_{\min}(A) + d\eta^2_{k}\bar{\lambda}^2} \lambda_{\max}(\mathcal A_{i, k-1})
\\
&\dots
\\
&< e^{-2\lambda_{\min}(A)\sum_{j = i+1}^k\eta_j + d\sum_{j=i+1}^k \eta^2_{j}\bar{\lambda}^2}  \lambda_{\max}(\mathcal A_i)
\\
&< e^{-2\lambda_{\min}(A)\sum_{j = i+1}^k\eta_j + \eta^2\bar{\lambda}^2\sum_{j=i+1}^kj^{-2\alpha}}  \lambda_{\max}(\mathcal A_i)
\\
&< e^{-2\lambda_{\min}(A)\sum_{j = i+1}^k\eta_j + C} \lambda_{\max}(\mathcal A_i)
\\
&< Ce^{-2\lambda_{\min}(A)\sum_{j=i+1}^k\eta_j} \lambda_{\max}(\mathcal A_i),
\end{align*}
as desired. Here the second-last inequality follows using the fact that $\sum_{j = 1}^\infty j^{-2\alpha} < C$ for $\alpha > \frac{1}{2}$.  
\end{proof}
\begin{lemma}\label{lem:calNnegEc}
Recall that 
\begin{align*}
    R_i := \prod_{j = i+1}^t (I - \eta_j X_j X^\top_j), \quad u_i:= R_i a;
\end{align*} 
and 
\begin{align*}
S_i := \prod_{j = 1}^{i-1} (I - \eta_{i-j}A), \quad v_i := S_i(\beta^* - \theta_0).
\end{align*}
Further recall from Notation \ref{notation} that 
\begin{align*}
\mathcal A_i := \mathbb{E}[(X_iX_i^\top - A)v_i v_i^\top (X_iX_i^\top - A) + \epsilon^2_iX_iX_i^\top + \epsilon_i X_iv^\top_i(X_iX_i^\top - A) + \epsilon_i (X_iX^\top_i - A)v_i X^\top_i].
\end{align*}
Under Assumption \ref{assump}, we have for all $t,d \geq C_1$ that
\begin{align*}
    \lambda_{\max}(\mathcal A_i) \leq C_2(\sigma^2 \bar{\lambda}  + \bar{\lambda}^2e^{-2\eta \lambda_{\min}(A)d^{-\frac{1}{2}}\sum_{j  = 1}^{i-1}j^{-\alpha}}|\beta^* - \theta_0|^2).
\end{align*}
Here $C_1, C_2 > 0$ are absolute constants.
\end{lemma}
\begin{proof}
Observe that $\mathcal A_i$ is a positive definite symmetric matrix. Now, observe for any fixed vector $u$ that
\begin{align*}
u^\top \mathcal A_i u &= \mathbb{E}[u^\top (XX^\top - A) v_i]^2 + \mathbb{E}[\epsilon^2(u^\top X)^2] + 2\mathbb{E}[\epsilon (u^\top X)(v^\top_i (XX^\top - A)u)]
\\
&\leq \mathbb{E}[u^\top (XX^\top - A) v_i]^2 + \mathbb{E}[\epsilon^2(u^\top X)^2] + 2\mathbb{E}[\epsilon^4]^{\frac{1}{4}}\mathbb{E}[(u^\top X)^4]^{\frac{1}{4}}\mathbb{E}[(u^\top( XX^\top - A)v_i)^2]^{\frac{1}{2}}
\\
&\leq C|u|^2 (\bar{\lambda}^2 |v_i|^2 + \sigma^2 \bar{\lambda} + \sigma \bar{\lambda}^{\frac{3}{2}}|v_i|).
\\
&\leq  C|u|^2 (\sigma^2 \bar{\lambda} + \bar{\lambda}^2 |v_i|^2)
\end{align*}
Here the third inequality follows from Lemma \ref{X-moment-upper-bound} and Assumption \ref{assump} on $\mathbb{E}[\epsilon^{4p_{\max}}]$ and $\mathbb{E}[(u^\top X)^{4p_{\max}}]$. Finally, substituting the upper bound on $|v_i|$ from Lemma \ref{S_imoment} gives us the desired result.
\end{proof}

\begin{lemma}\label{S_imoment}
Recall that $A := \mathbb{E}[XX^\top]$ and $S_i := \prod_{j = 1}^{i-1}(I - \eta_{i-j}A)$. Under Assumption \ref{assump}, we have that
\begin{align*}
|S_i (\beta^* - \theta_0)|^{2p} < e^{-2p\sum_{j = 1}^{i-1}\eta_{i-j}\lambda_{\min}(A)}|\beta^* - \theta_0|^{2p}
\end{align*}
for all $t,d \geq C$. Here $C > 0$ is an absolute constant. 
\end{lemma}
\begin{proof}
From Assumption \ref{assump}, we have $\eta \lambda_{\max}(A) < \eta \bar{\lambda} < C$ for an absolute constant $C > 0$. Since $\eta_i := \frac{\eta}{\sqrt{d}i^{\alpha}}$ for all $1\leq i \leq t$, this implies that $\eta_i \lambda_{\max}(A) < 1$ for all large enough $d$. Thus for all large enough $t,d$, we have  
\begin{align*}
0 < 1 -  \eta_j \lambda_{\max}(A) \leq \lambda_{\min}(I - \eta_jA)\leq \lambda_{\max}(I - \eta_jA) \leq 1 - \eta_j\lambda_{\min}(A) < e^{-\eta_j \lambda_{\min}(A)}
\end{align*}
for all $1 \leq j \leq t$. Using this gives us that
\begin{align*}
|S_i (\beta^* - \theta_0)|^{2p} &= |\prod_{j = 1}^{i-1}(I - \eta_jA)(\beta^* - \theta_0)|^{2p}
\\
&\leq e^{-2p\sum_{j=1}^{i-1}\eta_{i-j}\lambda_{\min}(A)}|\beta^* - \theta_0|^{2p},
\end{align*}
as desired.
\end{proof}
\subsection{Properties of the data $X,Y$.}
\begin{lemma}\label{X-moment-upper-bound}
Recall that $A := \mathbb{E}[XX^\top]$. Under Assumption \ref{assump}, we have for all $1 \leq p \leq p_{\max}$ and $u, v \in \mathbb{R}^d$ that
\begin{align*}
    \mathbb{E}[(u^\top (XX^\top - A)v)^{2p}] \leq (C \bar{\lambda})^{2p} |u|^{2p} |v|^{2p}.
\end{align*}
Here $C > 0$ is an absolute constant. 
\end{lemma}
\begin{proof}
Observe that 
\begin{align*}
\mathbb{E}[(u^\top (XX^\top - A)v)^{2p}] &\leq C^{2p} \mathbb{E}[(u^\top  XX^\top v)^{2p}]
\\
&\leq C^{2p}\mathbb{E}[(u^\top X)^{4p}]^{\frac{1}{2}}\mathbb{E}[(v^\top X)^{4p}]^{\frac{1}{2}}
\\
&\leq C^{2p}\bar{\lambda}^{2p} |u|^{2p}|v|^{2p},
\end{align*}
as desired. Here the last inequality follows from Assumption \ref{assump} on $\mathbb{E}[(u^\top X)^{4p_{\max}}]$ and Minkowski's inequality for $1 \leq p \leq p_{\max}$
\end{proof}
\begin{lemma}\label{noise-moment-lower-bound}
Recall that $\epsilon := Y - X^\top \beta^*$. Under Assumption \ref{assump}, we have that
\begin{align*}
\mathbb{E}[\epsilon^2 (u^\top X)^2]\geq \sigma_{\min}^2 \lambda_{\min}(A)|u|^2
\end{align*}
for all $u \in \mathbb{R}^d$. 
\end{lemma}
\begin{proof}
Observe that
\begin{align*}
\mathbb{E}[\epsilon^2 (u^\top X)^2] &= u^\top \mathbb{E}[\epsilon^2 XX^\top] u
\\
&\geq \sigma^2_{\min} \lambda_{\min}(A) |u|^2,
\end{align*}
as desired. Here the last inequality follows from Assumption \ref{assump} that $\lambda_{\min}(\mathbb{E}[\epsilon^2 XX^\top]) \geq \sigma^2_{\min} \lambda_{\min}(A)$. 
\end{proof}

\begin{lemma}\label{noise-moment-upper-bound}
Recall that $\epsilon := Y - X^\top \beta^*$. Under Assumption \ref{assump}, we have that
\begin{align*}
\mathbb{E}[\epsilon^{2p}(u^\top X)^{2p}] \leq \sigma^{2p}\bar{\lambda}^{p} |u|^{2p}
\end{align*}
for all $u \in \mathbb{R}^d$ and $1 \leq p \leq p_{\max}$.
\end{lemma}
\begin{proof}
Observe that 
\begin{align*}
\mathbb{E}[\epsilon^{2p}(u^\top X)^{2p}] &\leq \mathbb{E}[\epsilon^{4p}]^{\frac{1}{2}}\mathbb{E}[(u^\top X)^{4p}]^{\frac{1}{2}}
\\
&\leq \sigma^{2p}\bar{\lambda}^{p}|u|^{2p},
\end{align*}
as desired. Here the last line followed from Assumption \ref{assump} on $\mathbb{E}[\epsilon^{4p_{\max}}]$ and $\mathbb{E}[(u^\top X)^{4p_{\max}}]$ and Minkowski's inequality for $1 \leq p \leq p_{\max}$.
\end{proof}
\begin{lemma}\label{momenthelper1}
Under Assumption \ref{assump}, we have that
\begin{align*}
\mathbb{E}_{X}|(I - \eta_t XX^\top)v|^2  < (1 - 2\eta_{t}\lambda_{\min}(A) + d\eta_{t}^2\bar{\lambda}^2)|v|^2,
\end{align*}
for all $v \in \mathbb{R}^d$.
\end{lemma}
\begin{proof}
Observe for any fixed vector $v$ that
\begin{align*}
\mathbb{E}_{X}|(I - \eta_{t}XX^\top)v|^2 &= \mathbb{E}_X(|v|^2 - 2\eta_{t}(v^\top XX^\top v) + \eta_{t}^2 |XX^\top v|^2)
\\
&= \mathbb{E}_X(|v|^2 - 2\eta_{t}(v^\top A v) + \eta^2_{t} \mathbb{E}_X |XX^\top v|^2)
\\
&\leq (1 - 2\eta_{t}\lambda_{\min}(A) + d\eta_{t}^2\bar{\lambda}^2)|v|^2,
\end{align*}
as desired. Here the last inequality follows from  Lemma \ref{momenthelper2}.
\end{proof}

\begin{lemma}\label{momenthelper2}
Under Assumption \ref{assump}, we have that
\begin{align*}
\mathbb{E}[|XX^\top v|^{2p}] \leq d^p\bar{\lambda}^{2p} |v|^{2p},
\end{align*}
for all fixed $v \in \mathbb{R}^d$ and $1 \leq p \leq p_{\max}$. 
\end{lemma}
\begin{proof}
Observe that
\begin{align*}
\mathbb{E}[|XX^\top v|^{2p}] &= \mathbb{E}[(X^\top v)^{2p} |X|^{2p}]
\\
&\leq [\mathbb{E}[(X^\top v)^{4p}]]^{\frac{1}{2}}[\mathbb{E}[|X|^{4p}]]^{\frac{1}{2}}
\\
&\leq d^p \bar{\lambda}^{2p} |v|^{2p},
\end{align*}
as desired. Here the last inequality followed from Lemma \ref{momenthelper4}, Assumption \ref{assump} on $\mathbb{E}[(X^\top v)^{4p_{\max}}]$ and Minkowski's inequality for $1 \leq p \leq p_{\max}$. 
\end{proof}

\begin{lemma}\label{momenthelper4}
Under Assumption \ref{assump}, we have that
\begin{align*}
\mathbb{E}_X[|X|^{2p}] \leq d^p \bar{\lambda}^p,
\end{align*}
for all $1 \leq p \leq p_{\max}$. 
\end{lemma}
\begin{proof}
Let $e_1, e_2, \dots e_d \in \mathbb{R}^d$ denote any orthonormal basis vectors. Observe that 
\begin{align*}
\mathbb{E}|X|^{2p} &= \mathbb{E}(|X|^2)^p]
\\
&= \mathbb{E}[(\sum_{i = 1}^d \langle X, e_i \rangle^2)^p]
\\
&\leq \mathbb{E}[d^{p-1}\sum_{i = 1}^d \langle X, e_i \rangle^{2p}]
\\
&= d^{p-1}\sum_{i = 1}^d \mathbb{E}[\langle X, e_i \rangle^{2p}]
\\
&\leq d^p \bar{\lambda}^p,
\end{align*}
as desired. Here the third line follows from Jensen's and the last line follows from Assumption \ref{assump} on $\mathbb{E}[(X^\top u)^{2p}]$ and Minkowski's inequality for $1 \leq p \leq p_{\max}$.
\end{proof}
\subsection{Algebraic Identities.}
\begin{lemma}\label{matrix-algebra-identity}
We have the following identity-
\begin{align*}
    I - \sum_{j = 1}^{i-1}\bigg(\prod_{k = 1}^{i-1-j}(I - \eta_{i-k}A)\bigg)\eta_j A = \prod_{j = 1}^{i-1}(I - \eta_{i - j}A).
\end{align*}
\end{lemma}
\begin{proof}
We prove this by induction on $i$. Suppose it holds true for $i = k$ for some $k$. Now observe that
\begin{align*}
\prod_{j = 1}^k (I - \eta_{k+1 -j}A) &= (I - \eta_k A)\prod_{j = 1}^{k-1} (I - \eta_{k - j}A)
\\
&= (I - \eta_kA) \bigg[I - \sum_{j = 1}^{k-1}\bigg(\prod_{j' = 1}^{k-1-j} (I - \eta_{k - j'}A)\bigg)\eta_jA\bigg]
\\
&= I - \bigg[\sum_{j = 1}^{k-1}(I - \eta_k A)\bigg(\prod_{j' = 1}^{k-1-j}(I - \eta_{k - j'}A)\bigg)\eta_j A \bigg] - \eta_kA
\\
&= I - \bigg[\sum_{j = 1}^{k-1}\bigg(\prod_{j' = 1}^{k-j}(I - \eta_{k+1 - j'}A)\bigg)\eta_jA\bigg] - \eta_kA
\\
&= I - \sum_{j = 1}^{k}\bigg(\prod_{j' = 1}^{k-j}(I - \eta_{k+1 - j'}A)\bigg)\eta_jA,
\end{align*}
as desired. For $i = 1$, both sides are $I$ and the equality holds. Thus we are done by induction on $i$.
\end{proof}
\section{Auxiliary Results for Variance Estimation}
We will use the following notation for all results in this section. 
\\\\
Recall that
\begin{align*}
R_i:= \prod_{j = i+1}^{t}(I - \eta_{j}X_{j}X^\top_{j}), \quad u_i := R_ia
\end{align*}
and 
\begin{align*}
S_i := \prod_{j = 1}^{i-1}(I - \eta_{i-j}A), \quad v_i := S_i(\beta^* - \theta_0).
\end{align*}
Further, recall that
\begin{align*}
\epsilon := Y - X^\top \beta^*, \quad A := \mathbb{E}[XX^\top], \quad 
A_{\sigma} := \mathbb{E}[\epsilon^2 XX^\top].
\end{align*}
Let $\mathbf{e}_1, \mathbf{e_2}, \dots \mathbf{e}_d$ be an eigen-basis of $A$ with corresponding eigen-values $\lambda_1 \geq \lambda_2 \geq \dots \lambda_d > 0$. Finally for all $1 \leq k, k' \leq d$, let $a_k := \langle \mathbf{e}_k, a \rangle$ and $[A_{\sigma}]_{k,k'} := \langle  \mathbf{e}_k, A_\sigma \mathbf{e}_{k'}\rangle$ denote the respective components of $a$ and $A_{\sigma}$ in the above basis.
\subsection{MSE Of The Plug-In Estimator.}
\begin{lemma}\label{euclidean-norm-bound}
Under Assumption \ref{assump}, we have for all $t,d \geq C_1$ that
\begin{align*}
    \mathbb{E}|\theta_t - \beta^*|^2 &\leq C_2d^{\frac{1}{2}}t^{-\alpha}(\eta \lambda_{\min}(A))^{-1}(\eta \sigma^2).
\end{align*}
Here $C_1, C_2 > 0$ are absolute constants.
\end{lemma}
\begin{proof}
For the rest of the proof, we let $C > 0$ and $c>0$ respectively denote large and small enough generic absolute constants.
\\\\
Observe for any fixed $a \in \mathbb{R}^d$ and all $t,d \geq C$ that,
\begin{align*}
    \mathbb{E}\langle a, \theta_t - \beta^*\rangle^2 &= [\mathbb{E}\langle a, \theta_t - \beta^* \rangle]^2 + \Var \langle a, \theta_t \rangle 
    \\
    &\leq e^{-2\eta \lambda_{\min}(A)d^{-\frac{1}{2}}t^{1-\alpha}}|a|^2 |\theta_0 - \beta^*|^2 + C \eta d^{-\frac{1}{2}}t^{-\alpha} \sum_{k,k' = 1}^d \frac{a_k a_{k'}[A_{\sigma}]_{k,k'}}{\lambda_k + \lambda_{k'}}
    \\
    &\leq e^{-c(\log t + \log d)^2}(td)^C (\eta\sigma^2) |a|^2 + C\eta d^{-\frac{1}{2}}t^{-\alpha} (\lambda_{\min}(A))^{-1}\sigma^2 \bar{\lambda}|a|^2
    \\
    &\leq e^{-c(\log t + \log d)^2}(td)^C (\eta\sigma^2) |a|^2 + Cd^{-\frac{1}{2}}t^{-\alpha} (\eta\lambda_{\min}(A))^{-1}(\eta\sigma^2)|a|^2
    \\
    &\leq Cd^{-\frac{1}{2}}t^{-\alpha} (\eta\lambda_{\min}(A))^{-1}(\eta\sigma^2)|a|^2,
\end{align*}
Here the first inequality follows from Lemma \ref{lem:bias-expression} and Theorem \ref{lem:VarAsymp}, the second inequality follows from Lemma \ref{noise-moment-upper-bound} and Assumption \ref{assump} on $|\theta_0 - \beta^*|$, and the third inequality follows from Assumption \ref{assump} that $\eta \bar{\lambda} < C$.
\\\\
This gives us that
\begin{align*}
\mathbb{E}|\theta_t - \beta^*|^2 &= \mathbb{E}[\sum_{i = 1}^d \langle e_i, \theta_t - \beta^* \rangle^2]
\\
&\leq Cd^{\frac{1}{2}}t^{-\alpha}(\eta \lambda_{\min}(A))^{-1}(\eta \sigma^2),
\end{align*}
for all $t,d \geq C$, as desired.
\end{proof}

\subsection{Exact First Order Noise Term} 
\begin{lemma}\label{variance-sigma-term}
Under Assumption \ref{assump}, we have for all $t, d \geq C_1$ that
\begin{align*}
\sum_{i = 1}^t \eta_i^2 \mathbb{E}_{u_i}(u_i^\top A_{\sigma} u_i) = (1+\mathcal E)\eta d^{-\frac{1}{2}}t^{-\alpha}\sum_{k,k' = 1}^d \frac{a_k a_{k'}[A_{\sigma}]_{k,k'}}{\lambda_{k} + \lambda_{k'}},
\end{align*}
where $|\mathcal E| \leq C_2(\log t + \log d)^2 [(\eta \lambda_{\min}(A))^{-1}d^{\frac{1}{2}}t^{-(1-\alpha)} + (\eta\lambda_{\min}(A))^{-3}\sigma^2\sigma^{-2}_{\min}d^{\frac{1}{2}}t^{-\alpha}]$. Here $C_1, C_2 > 0$ are absolute constants.
\end{lemma}
\begin{proof}
Throughout this proof, we let $C > 0$ denote a large enough and generic absolute constant.
\\\\
Lemma \ref{variance-auxilliary-helper} gives us that
\begin{align*}
\sum_{i = 1}^t \eta_i^2 \mathbb{E}_{u_i}(u_i^\top A_{\sigma} u_i) = \sum_{i = 1}^t \eta_i^2(\mathbb{E}[u_i])^\top A_{\sigma} (\mathbb{E}[u_i]) + \sum_{i = 1}^t\mathcal E_i
\end{align*}
where $0 < \mathcal E_i <  C(\sigma^2 \bar{\lambda}) \eta_i^2 e^{-2\eta \lambda_{\min}(A) d^{-\frac{1}{2}}\sum_{j = i+1}^t j^{-\alpha}}(\sum_{j = i+1}^t j^{-2\alpha}) |a|^2$. Now, Lemma \ref{variance-error-control} tells us for all $t, d \geq C$ that
\begin{align*}
    \sum_{i = 1}^t \mathcal E_i &\leq C\eta^2(\sigma^2 \bar{\lambda})(\log t + \log d)^2 t^{-2\alpha} (\eta\lambda_{\min}(A))^{-2}|a|^2
    \\
    &\leq C\eta\sigma^2(\log t + \log d)^2 t^{-2\alpha} (\eta\lambda_{\min}(A))^{-2}|a|^2
    \\
    &\leq \bigg[\eta d^{-\frac{1}{2}}t^{-\alpha}\sum_{k,k' = 1}^d\frac{a_k a_{k'}[A_{\sigma}]_{k,k'}}{\lambda_k + \lambda_{k'}}\bigg]\bigg[C\sigma^2(\log t + \log d)^2 d^{\frac{1}{2}}t^{-\alpha} (\eta\lambda_{\min}(A))^{-2}|a|^2 \bigg(\sum_{k,k' = 1}^d\frac{a_k a_{k'}[A_{\sigma}]_{k,k'}}{\lambda_k + \lambda_{k'}}\bigg)^{-1}\bigg]
    \\
    &\leq \bigg[\eta d^{-\frac{1}{2}}t^{-\alpha}\sum_{k,k' = 1}^d\frac{a_k a_{k'}[A_{\sigma}]_{k,k'}}{\lambda_k + \lambda_{k'}}\bigg]\bigg[C\sigma^2(\log t + \log d)^2 d^{\frac{1}{2}}t^{-\alpha} (\eta\lambda_{\min}(A))^{-2}|a|^2 \lambda_{\max}(A)(a^\top A_{\sigma}a)^{-1}\bigg]
    \\
    &\leq \bigg[\eta d^{-\frac{1}{2}}t^{-\alpha}\sum_{k,k' = 1}^d\frac{a_k a_{k'}[A_{\sigma}]_{k,k'}}{\lambda_k + \lambda_{k'}}\bigg]\bigg[C\sigma^2(\log t + \log d)^2 d^{\frac{1}{2}}t^{-\alpha} (\eta\lambda_{\min}(A))^{-2}\sigma^{-2}_{\min}\lambda_{\max}(A)\lambda_{\min}(A)^{-1}\bigg]
    \\
    &\leq C(\log t + \log d)^2 d^{\frac{1}{2}}t^{-\alpha} (\eta\lambda_{\min}(A))^{-3}\sigma^2\sigma^{-2}_{\min}\bigg[\eta d^{-\frac{1}{2}}t^{-\alpha}\sum_{k,k' = 1}^d\frac{a_k a_{k'}[A_{\sigma}]_{k,k'}}{\lambda_k + \lambda_{k'}}\bigg].
\end{align*}
Here the first inequality follows from assumptions \ref{assump} that $\eta \bar{\lambda} < C$ and the second last inequality follows from Lemma \ref{noise-moment-lower-bound}. Further, Lemma \ref{variance-simplification} tells us for all $t,d \geq C$ that
\begin{align*}
    \sum_{i = 1}^t \eta_i^2 (\mathbb{E}[u_i])^\top A (\mathbb{E}[u_i]) = (1+\mathcal E)\eta d^{-\frac{1}{2}}t^{-\alpha}\sum_{k,k' = 1}^d \frac{a_k a_{k'}[A_{\sigma}]_{k,k'}}{\lambda_{k} + \lambda_{k'}},
\end{align*}
where $|\mathcal E| \leq Ct^{-(1 - \alpha)}d^{\frac{1}{2}}(\log t + \log d)^2 (\eta \lambda_{\min}(A))^{-1}$. Combining these gives us the desired result.
\end{proof}

\begin{lemma}\label{variance-simplification}
Under Assumption \ref{assump}, we have for all $t,d \geq C_1$ that
\begin{align*}
\mathbf{S}:= \sum_{i = 1}^t \eta_i^2 (\mathbb{E}[u_i])^\top  A_{\sigma} (\mathbb{E}[u_i]) = (1 + \mathcal E)\eta d^{-\frac{1}{2}}t^{-\alpha}\sum_{k,k' = 1}^d \frac{a_{k}a_{k'}[A_{\sigma}]_{k,k'}}{(\lambda_k + \lambda_{k'})},
\end{align*}
where $|\mathcal E| \leq \frac{C_2 t^{\alpha - 1}d^{\frac{1}{2}}(\log t + \log d)^2}{\eta \lambda_{\min}(A)}$. Here $C_1, C_2 > 0$ are absolute constants. 
\end{lemma}
\begin{proof} 
Throughout this proof, we let $C > 0$ denote a large enough and generic absolute constant.
\\\\
To begin with, observe that
\begin{align*}
    \mathbb{E}[R_i] = \prod_{j = i+1}^{t}(I - \eta_{j}A) = \prod_{j = i+1}^{t}(I - \eta d^{-\frac{1}{2}}j^{-\alpha}A).
\end{align*}
Thus working in the $\mathbf{e_1}, \mathbf{e_2}, \dots \mathbf{e_d}$ basis (where  $\mathbb{E}[R_i]$ becomes a diagonal matrix), we get that
\begin{align*}
(\mathbb{E}[u_i])^\top A_{\sigma} (\mathbb{E}[u_i]) &= \sum_{k,k' = 1}^d a_k a_{k'} [ A_{\sigma}]_{k,k'} \prod_{j = i+1}^t[(1 - \eta d^{-\frac{1}{2}}j^{-\alpha}\lambda_{k})(1 - \eta d^{-\frac{1}{2}}j^{-\alpha}\lambda_{k'})]
\\
&= \sum_{k,k' = 1}^d a_k a_{k'} [A_{\sigma}]_{k,k'} \prod_{j = i+1}^t[1 - \eta d^{-\frac{1}{2}}j^{-\alpha}(\lambda_k + \lambda_{k'}) + \eta^2 d^{-1} j^{-2\alpha} \lambda_k \lambda_{k'}]
\end{align*}
This further tells us that
\begin{align*}
\mathbf{S} &:= \sum_{i = 1}^t \eta_i^2 (\mathbb{E}[u_i])^\top  A_{\sigma} (\mathbb{E}[u_i])
\\
&=\eta^2d^{-1}\sum_{i = 1}^t i^{-2\alpha}\sum_{k,k' = 1}^d a_k a_{k'}[A_{\sigma}]_{k,k'} \prod_{j=i+1}^t[1 - \eta d^{-\frac{1}{2}} j^{-\alpha} (\lambda_k + \lambda_{k'}) + \eta^2 d^{-1} j^{-2\alpha} \lambda_{k} \lambda_{k'}]
\\
&= \eta^2d^{-1}\sum_{k,k' = 1}^d a_k a_{k'}[A_{\sigma}]_{k,k'} \sum_{i = 1}^ti^{-2\alpha}\prod_{j=i+1}^t[1 - \eta d^{-\frac{1}{2}} j^{-\alpha} (\lambda_k + \lambda_{k'}) + \eta^2 d^{-1} j^{-2\alpha} \lambda_{k} \lambda_{k'}]
\end{align*}
Similar to the proof of Lemma \ref{variance-error-control}, consider the cut-off $t_0 := \frac{Kd^{\frac{1}{2}}t^{\alpha} (\log t + \log d)}{\eta (\lambda_k + \lambda_{k'})}$, where $K > 0$ is an absolute constant. Below we show that by chosing $K$ large enough we have for all $t,d \geq C$ that
\begin{align*}
    &\underbrace{\sum_{i = 1}^{t-t_0} i^{-2\alpha}\prod_{j=i+1}^t[1 - \eta d^{-\frac{1}{2}} j^{-\alpha} (\lambda_k + \lambda_{k'}) + \eta^2 d^{-1} j^{-2\alpha} \lambda_{k} \lambda_{k'}] \leq Ct^{-K}d^{-K}}_{(\mathbf{I})}
    \\
    &\underbrace{\sum_{i = t-t_0}^{t} i^{-2\alpha}\prod_{j=i+1}^t[1 - \eta d^{-\frac{1}{2}} j^{-\alpha} (\lambda_k + \lambda_{k'}) + \eta^2 d^{-1} j^{-2\alpha} \lambda_{k} \lambda_{k'}] = \frac{(1 + \mathcal E)d^{\frac{1}{2}}t^{-\alpha}}{\eta (\lambda_k + \lambda_{k'})}}_{(\mathbf{II})},
\end{align*}
where $|\mathcal E| \leq \frac{CK^2 t^{\alpha - 1}d^{\frac{1}{2}}(\log t + \log d)^2}{\eta \lambda_{\min}(A)}$.  Let $\mathbf{U} := \frac{CK^2 t^{\alpha - 1}d^{\frac{1}{2}}(\log t + \log d)^2}{\eta \lambda_{\min}(A)}$.
\\\\
Now because of Assumption \ref{assump}, we have $\eta \lambda_{\max}(A) <  C$, thus $(\mathbf{I})$ can be made arbitrarily smaller than $\frac{\mathbf{U}\cdot d^{\frac{1}{2}}t^{-\alpha}}{\eta (\lambda_k + \lambda_{k'})}$ by choosing the absolute constant $K$ large enough. This shows that
\begin{align*}
\sum_{i = 1}^t i^{-2\alpha} \prod_{j = i+1}^t [1 - \eta d^{-\frac{1}{2}}j^{-\alpha}(\lambda_k + \lambda_{k'}) + \eta^2 d^{-1} j^{-2\alpha} \lambda_k \lambda_{k'})] = \frac{(1 + \mathcal E)d^{\frac{1}{2}}t^{-\alpha}}{\eta (\lambda_k + \lambda_{k'})},
\end{align*}
where $|\mathcal E| \leq \frac{Ct^{\alpha - 1}d^{\frac{1}{2}}(\log t + \log d)^2}{\eta \lambda_{\min}(A)}$. Substituting this into the expression for $\mathbf{S}$ gives us that
\begin{align*}
\mathbf{S} &= (1 + \mathcal E)\eta^2 d^{-1} \sum_{k,k' = 1}^d \frac{a_k a_{k'} [A_{\sigma}]_{k,k'}d^{\frac{1}{2}}t^{-\alpha}}{\eta (\lambda_k + \lambda_{k'})}
\\
&= (1 + \mathcal E)\eta d^{-\frac{1}{2}}t^{-\alpha}\sum_{k,k' = 1}^d \frac{a_{k}a_{k'}[A_{\sigma}]_{k,k'}}{(\lambda_k + \lambda_{k'})},
\end{align*}
as desired.
\\\\
It now remains to prove $\mathbf{(I)}$ and $\mathbf{(II)}$ which we do below.
\\\\
\textsc{Proof of $(\mathbf{I})$:} Since $\eta \lambda_{\max}(A) < \eta \bar{\lambda} < C$ (by Assumption \ref{assump}), we have for all large enough $d$ that $\eta d^{-\frac{1}{2}}j^{-\alpha}[\lambda_k + \lambda_{k'}] < 1$, for all $1 \leq k, k' \leq d$. This implies that, 
\begin{align*}
\prod_{j = i+1}^t (1 - \eta d^{-\frac{1}{2}}j^{-\alpha} [\lambda_k + \lambda_{k'}] + \eta^2 d^{-1} j^{-2\alpha} \lambda_{k}\lambda_{k'}) &< e^{-\eta (\lambda_k + \lambda_{k'}) d^{-\frac{1}{2}} \sum_{j = i+1}^t j^{-\alpha} + \eta^2 \lambda_{k} \lambda_{k'}d^{-1} \sum_{j = i+1}^tj^{-2\alpha}} 
\\&< e^{-\eta (\lambda_k + \lambda_{k'})t_0 t^{-\alpha}d^{-\frac{1}{2}} + C(\eta \lambda_{\max}(A))^2d^{-1}}
\\&< Ce^{-\eta (\lambda_k + \lambda_{k'}) t_0 t^{-\alpha} d^{-\frac{1}{2}}}.
\end{align*}
Now, substituting the value of $t_0 := \frac{Kd^{\frac{1}{2}}t^{\alpha} (\log t + \log d)}{\eta (\lambda_k + \lambda_{k'})}$ gives us the desired result.
\\\\
\textsc{Proof of $(\mathbf{II})$:} Let 
\begin{align*}
x_1 := \eta d^{-\frac{1}{2}}(\lambda_k + \lambda_{k'}), \quad x_2 := \eta^2 d^{-1} \lambda_k \lambda_{k'}.
\end{align*}
Observe that $x_2j^{-2\alpha} < x_1j^{-\alpha}$ for all $1 \leq j \leq t$ and all large enough $d$ (since $\frac{x_2j^{-2\alpha}}{x_1j^{-\alpha}} < \eta d^{-\frac{1}{2}} j^{-\alpha} \lambda_k < \eta d^{-\frac{1}{2}} \lambda_{\max}(A) <  Cd^{-\frac{1}{2}}$ by Assumption \ref{assump}). Further observe by Assumption \ref{assump} that $x_1 < \frac{1}{2}$ for all large enough $d$. This implies that
\begin{align*}
0 < \eta d^{-\frac{1}{2}}j^{-\alpha} (\lambda_k + \lambda_{k'}) - \eta^2 d^{-1} j^{-2\alpha}\lambda_{k} \lambda_{k'} < \frac{1}{2}
\end{align*}
for all large enough $d$. Now, observe that $e^{- x - x^2} < 1 - x < e^{-x}$ for $x \in (0, \frac{1}{2}]$. Thus we have for all large enough $d$ that
\begin{align*}
&e^{- (x_1j^{-\alpha} - x_2j^{-2\alpha}) - (x_1j^{-\alpha} - x_2j^{-2\alpha})^2}<  (1 - \eta d^{-\frac{1}{2}}j^{-\alpha} (\lambda_k + \lambda_{k'}) + \eta^2 d^{-1} j^{-2\alpha} \lambda_{k} \lambda_{k'}) < e^{-(x_1j^{-\alpha} - x_2j^{-2\alpha})}
\\
\implies& e^{-x_1j^{-\alpha} - x_1^2j^{-2\alpha} - x^2_2j^{-4\alpha}} <  (1 - \eta d^{-\frac{1}{2}}j^{-\alpha} (\lambda_k + \lambda_{k'}) + \eta^2 d^{-1} j^{-2\alpha} \lambda_{k} \lambda_{k'}) < e^{-x_1j^{-\alpha} + x_2j^{-2\alpha}}
\\
\implies & e^{-x_1 j^{-\alpha} - 2x_1^2j^{-2\alpha}} < (1 - \eta d^{-\frac{1}{2}}j^{-\alpha} (\lambda_k + \lambda_{k'}) + \eta^2 d^{-1} j^{-2\alpha} \lambda_{k} \lambda_{k'}) < e^{-x_1j^{-\alpha} + x_2 j^{-2\alpha}},
\end{align*}
Here the last inequality follows from the observation made above that $x_2 < x_1 < \frac{1}{2}$ for all large enough $d$. Multiplying this from $j = i+1$ to $t$ gives us that
\begin{align*}
\prod_{j = i+1}^t (1 - \eta d^{-\frac{1}{2}}j^{-\alpha} (\lambda_k + \lambda_{k'}) + \eta^2d^{-1}j^{-2\alpha} \lambda_{k} \lambda_{k'}) = e^{-\eta d^{-\frac{1}{2}}(\lambda_k + \lambda_{k'})\sum_{j = i+1}^t j^{-\alpha} + \mathcal E_1}
\end{align*}
where 
\begin{align*}
    |\mathcal E_1| &\leq (\max\{x_2, 2x^2_1\})\sum_{j = i+1}^t j^{-2\alpha}
    \\
    &\leq C\eta^2 \lambda_{\max}(A)^2 d^{-1}\sum_{j = i+1}^t j^{-2\alpha}
    \\
    &\leq C \eta^2 \lambda_{\max}(A)^2 d^{-1} t_0 (t-t_0)^{-2\alpha}
    \\
    &\leq C d^{-1} t_0 t^{-2\alpha}
    \\
    &\leq \frac{CK d^{-\frac{1}{2}}t^{-\alpha} (\log t + \log d)}{\eta (\lambda_k + \lambda_{k'})}
\end{align*}
This tells us that
\begin{align*}
\mathbf{(II)} = (1 + \mathcal E_1)\sum_{i = t-t_0}^t i^{-2\alpha} e^{-\eta d^{-\frac{1}{2}} (\lambda_k + \lambda_{k'})\sum_{j = i+1}^t j^{-\alpha}}
\end{align*}
where $|\mathcal E_1| \leq \frac{CK d^{-\frac{1}{2}}t^{-\alpha} (\log t + \log d)}{\eta (\lambda_k + \lambda_{k'})}$. Further observe that
\begin{align*}
\sum_{j = i+1}^t j^{-\alpha} &= \sum_{j = i+1}^t t^{-\alpha}(j/t)^{-\alpha}
\\
&= t^{-\alpha} \sum_{j = i+1}^t (j/t)^{-\alpha}
\\
&= (t-i)t^{-\alpha} + t^{-\alpha}\sum_{j = i+1}^t ((1 - (t-j)/t)^{-\alpha} - 1)
\\
&= (t-i)t^{-\alpha} + \mathcal E_2,
\end{align*}
where
\begin{align*}
    |\mathcal E_2| &< t^{-\alpha}\sum_{j = i+1}^t ((1 - (t-j)/t)^{-\alpha} - 1)
    \\ 
    &< Ct^{-\alpha} \sum_{j = i+1}^t (t-j)/t
    \\
    &< Ct^{-\alpha}((t-i)^2/t)
    \\
    &< \frac{CK^2 t^{\alpha - 1} d(\log t + \log d)^2}{\eta^2 (\lambda_k + \lambda_{k'})^2}
\end{align*}
This implies that
\begin{align*}
(\mathbf{II}) = (1 + \mathcal E_3)\sum_{i = t-t_0}^t i^{-2\alpha} e^{-\eta d^{-\frac{1}{2}}(\lambda_k + \lambda_{k'})(t-i)t^{-\alpha}},
\end{align*}
where 
\begin{align*}
|\mathcal E_3| &\leq |\mathcal E_1| + |\mathcal E_2|
\\
&\leq \frac{CK d^{-\frac{1}{2}}t^{-\alpha} (\log t + \log d)}{\eta (\lambda_k + \lambda_{k'})} + \frac{CK^2 t^{\alpha - 1} d^{\frac{1}{2}}(\log t + \log d)^2}{\eta (\lambda_k + \lambda_{k'})}
\\
&\leq \frac{CK^2 t^{\alpha - 1} d^{\frac{1}{2}}(\log t + \log d)^2}{\eta (\lambda_k + \lambda_{k'})}
\end{align*}
Further simplification gives us that.
\begin{align*}
(\mathbf{II}) &= (1 + \mathcal E_3)\sum_{i = t-t_0}^t i^{-2\alpha} e^{-\eta d^{-\frac{1}{2}}(\lambda_k + \lambda_{k'}) (t-i)t^{-\alpha}}
\\
&= (1 + \mathcal E_3)t^{-2\alpha}\sum_{i = t-t_0}^t (1 - (t-i)/t)^{-2\alpha} e^{-\eta d^{-\frac{1}{2}}(\lambda_k + \lambda_{k'}) (t-i)t^{-\alpha}}
\\
&= (1 + \mathcal E_3)(1 + \mathcal E_4)[t^{-2\alpha}\sum_{i = t-t_0}^t e^{-\eta d^{-\frac{1}{2}}(\lambda_k + \lambda_{k'}) (t-i)t^{-\alpha}}],
\end{align*}
where $|\mathcal E_4| \leq \frac{Ct_0}{t} \leq \frac{CK d^{\frac{1}{2}}t^{\alpha - 1}(\log t + \log d)}{\eta (\lambda_k  + \lambda_{k'})}$. Finally, observe that
\begin{align*}
\sum_{i = t-t_0}^t e^{-\eta d^{\frac{1}{2}}(\lambda_k + \lambda_{k'})(t-i)t^{-\alpha}} &= \sum_{i = 0}^{t_0} e^{-\eta d^{-\frac{1}{2}}(\lambda_k + \lambda_{k'})t^{-\alpha}i}
\\
&= \frac{1 - e^{-\eta d^{-\frac{1}{2}}(\lambda_k + \lambda_{k'})t^{-\alpha}(t_0 + 1)}}{1 - e^{-\eta d^{-\frac{1}{2}}(\lambda_k + \lambda_{k'})t^{-\alpha}}}
\\
&= \frac{1 + \mathcal E_5}{1 - e^{-\eta d^{-\frac{1}{2}}(\lambda_k + \lambda_{k'})t^{-\alpha}}}
\\
&= \frac{(1 + \mathcal E_5)(1 + \mathcal E_6)d^{\frac{1}{2}}t^{\alpha}}{\eta (\lambda_k + \lambda_{k'})}
\end{align*}
where $|\mathcal E_5| \leq Ct^{-K}d^{-K}$ and $|\mathcal E_6| \leq C d^{-\frac{1}{2}}t^{-\alpha}$. Combining these tells us that
\begin{align*}
(\mathbf{II}) = \frac{(1 + \mathcal E_7)d^{\frac{1}{2}}t^{-\alpha}}{\eta (\lambda_k + \lambda_{k'})}, 
\end{align*}
where $|\mathcal E_7| \leq \max\{|\mathcal E_3|, |\mathcal E_4|, |\mathcal E_5|, |\mathcal E_6|\} \leq \frac{CK^2 t^{\alpha - 1}d^{\frac{1}{2}}(\log t + \log d)^2}{\eta (\lambda_k + \lambda_{k'})} \leq \frac{CK^2 t^{\alpha - 1}d^{\frac{1}{2}}(\log t + \log d)^2}{\eta \lambda_{\min}(A)}$ (for large enough choice of absolute constant $K$), as desired. 
\end{proof}

\subsection{Bounds On Second Order Noise Terms}
\begin{lemma}\label{variance-auxilliary-helper}
Under Assumption \ref{assump}, we have for all $t, d \geq C_1$ that
\begin{align*}
0 \leq \mathbb{E}_{u_i}(u_i^\top A_{\sigma} u_i) 
- (\mathbb{E}[u_i])^\top A_{\sigma} (\mathbb{E}[u_i]) \leq C_2(\sigma^2 \bar{\lambda})e^{-2\eta \lambda_{\min}(A)d^{-\frac{1}{2}}\sum_{j = i+1}^tj^{-\alpha}}(\sum_{j = i+1}^tj^{-2\alpha})|a|^2.
\end{align*}
Here $C_1, C_2 > 0$ are absolute constants.
\end{lemma}
\begin{proof}
Throughout the proof, we let $C > 0$ denote a large enough and generic absolute constant.
\\\\
For all $i+1 \leq k \leq t+1$, define $u_{k,t}$ as the running  product
\begin{align*}
u_{k,t} := \bigg[\prod_{j = k}^{t}(I - \eta_{j}X_{j}X^\top_{j})\bigg]a.
\end{align*}
In particular, $u_{i+1,t} := R_ia$ and $u_{t+1,t} := a$. Further, we also define the sequence of matrices $\{\mathcal A_{i,k}\}_{k = i}^{t}$ recursively as $\mathcal A_{i,i} := A_{\sigma}$ and 
\begin{align*}
\mathcal A_{i,k} := (I - \eta_{k}A) \mathcal A_{i,k-1} (I - \eta_{k}A)
\end{align*}
for all $i+1 \leq k \leq t$. Recall because of Assumption \ref{assump} that $\eta \lambda_{\max}(A) < C$ for an absolute constant $C > 0$. This implies that for all large enough $d$, we have $\eta_k \lambda_{\max}(A) = \frac{\eta \lambda_{\max}(A)}{\sqrt{d}k^{\alpha}} < 1$ for all $1 \leq k \leq t$. This further tells us that 
\begin{align*}
0 < 1 - \eta_{k}\lambda_{\max}(A) \leq \lambda_{\min}(I - \eta_{k}A) \leq \lambda_{\max}(I - \eta_{k}A) \leq 1 - \eta_{k}\lambda_{\min}(A) < e^{-\eta_{k}\lambda_{\min}(A)}
\end{align*}
for all $i+1 \leq k \leq t$. In particular, this gives us that
\begin{align*}
\lambda_{\max}(\mathcal A_{i,k}) < e^{-2\lambda_{\min}(A) \sum_{j - i+1}^k \eta_j} \lambda_{\max}(\mathcal A_{i,i}) = e^{-2\lambda_{\min}(A) \sum_{j - i+1}^k \eta_j}\lambda_{\max}( A_{\sigma}),
\end{align*}
for all $i+1 \leq k \leq t$, which will be useful later in the proof.
\\\\
Now observe for all $i + 1 \leq k \leq t$ that
\begin{align*}
\mathbb{E}[u_{k,t}^\top \mathcal A_{i,k-1} u_{k,t}] &= \mathbb{E}_{u_{k+1, t}}[u_{k+1, t}^\top \mathcal A_{i,k} u_{k+1,t}] + \eta_k^2\mathbb{E}_{u_{k+1,t}, X}[u_{k+1,t}^\top (XX^\top - A) \mathcal A_{i,k}(XX^\top - A)u_{k+1,t}]
\\
&\leq \mathbb{E}_{u_{k+1,t}}[u_{k+1, t}^\top \mathcal A_{i,k} u_{k+1, t}] + \eta_k^2 \lambda_{\max}(\mathcal A_{i,k}) \mathbb{E}|(XX^\top - A)u_{k+1, t}|^2
\\
&\leq \mathbb{E}_{u_{k+1, t}}[u_{k+1, t}^\top \mathcal A_{i,k} u_{k+1, t}] + C\eta_k^2 \lambda_{\max}(\mathcal A_{i,k}) \mathbb{E}|XX^\top u_{k+1, t} |^2
\\
&\leq \mathbb{E}_{u_{k+1, t}}[u_{k+1, t}^\top \mathcal A_{i,k} u_{k+1, t}] + Cd\bar{\lambda}^2\eta_k^2 \lambda_{\max}(\mathcal A_{i,k})  \mathbb{E}|u_{k+1, t}|^2
\\
&\leq \mathbb{E}_{u_{k+1, t}}[u_{k+1, t}^\top \mathcal A_{i,k} u_{k+1,t}] + Cd\bar{\lambda}^2\eta_k^2 \lambda_{\max}(\mathcal A_{i,k})(\mathbb{E}|u_{k+1, t}|^4)^{\frac{1}{2}}
\\
&\leq \mathbb{E}_{u_{k+1,t}}[u_{k+1,t}^\top \mathcal A_{i,k} u_{k+1,t}] + Cd\bar{\lambda}^2\eta_k^2 \lambda_{\max}(\mathcal A_{i,k}) e^{-2\eta \lambda_{\min}(A)d^{-\frac{1}{2}}\sum_{j = k+1}^{t}j^{-\alpha}}|a|^2
\\
&\leq \mathbb{E}_{u_{k+1,t}}[u_{k+1,t}^\top \mathcal A_{i,k} u_{k+1,t}] + Ck^{-2\alpha}\lambda_{\max}(\mathcal A_{i,k}) e^{-2\eta \lambda_{\min}(A)d^{-\frac{1}{2}}\sum_{j = k+1}^{t}j^{-\alpha}}|a|^2
\\
&\leq \mathbb{E}_{u_{k+1, t}}[u_{k+1,t}^\top \mathcal A_{i,k} u_{k+1,t}] + Ck^{-2\alpha}\lambda_{\max}(A_{\sigma}) e^{-(2\eta \lambda_{\min}(A)d^{-\frac{1}{2}}\sum_{j = i+1}^t j^{-\alpha})}  |a|^2
\\
&\leq \mathbb{E}_{u_{k+1,t}}[u_{k+1,t}^\top \mathcal A_{i,k} u_{k+1,t}] + Ck^{-2\alpha}\lambda_{\max}(A_{\sigma}) e^{-2\eta \lambda_{\min}(A)d^{-\frac{1}{2}}\sum_{j = i+1}^t j^{-\alpha}}  |a|^2
\end{align*}
Here the fourth line follows from Lemma \ref{momenthelper1}, sixth line follows from Lemma \ref{conc1}, seventh line follows from Assumption \ref{assump} that $\eta \bar{\lambda} < C$ and eighth line follows from the upper bound on $\lambda_{\max}(\mathcal A_{i,k})$ proved above. 
\\\\
Adding up all such inequalities from $k = i+1$ to $t$ gives us that
\begin{align*}
\mathbb{E}[u_{i+1,t}^\top \mathcal A_{i,i} u_{i+1,t}] - a^\top \mathcal A_{i,t} a &\leq C\lambda_{\max}(A_{\sigma}) e^{-2\eta \lambda_{\min}(A)d^{-\frac{1}{2}}\sum_{j = i+1}^tj^{-\alpha}}(\sum_{j = i+1}^tj^{-2\alpha}) |a|^2
\\
&\leq C(\sigma^2 \bar{\lambda})e^{-2\eta \lambda_{\min}(A)d^{-\frac{1}{2}}\sum_{j = i+1}^t j^{-\alpha}}(\sum_{j = i+1}^t j^{-2\alpha})|a|^2,
\end{align*}
as desired. For the lower bound, we can directly apply Cauchy-Schwartz Inequality ($\mathbb{E}|U|^2 \geq |\mathbb{E}U|^2$) to the vector $U := \sqrt{\mathcal A_i}u_i$. Thus both the lower and upper bound follow completing the proof.
\end{proof}

\begin{lemma}\label{variance-error-control}
Define $\mathcal E_i := \eta^2_ie^{-2\eta \lambda_{\min}(A) d^{-\frac{1}{2}}\sum_{j = i+1}^t j^{-\alpha}}(\sum_{j = i+1}^tj^{-2\alpha})$ for all $1 \leq i \leq t$. Under Assumptions \ref{assump}, we have that
\begin{align*}
    \sum_{i = 1}^t \mathcal E_i \leq C\eta^2(\log t + \log d)^2 t^{-2\alpha} (\eta\lambda_{\min}(A))^{-2}
\end{align*}
for all $t, d \geq C$. Here $C > 0$ represents an absolute constant. 
\end{lemma}
\begin{proof}
Consider a cut-off $t_0 \in (1, t)$, to be fixed later. We have for $i\leq t-t_0$ that
\begin{align*}
|\mathcal E_i| &\leq \eta^2_ie^{-2\eta \lambda_{\min}d^{-\frac{1}{2}}(A)t_0 t^{-\alpha}}\sum_{j = i+1}^t j^{-2\alpha}
\\
&\leq C\eta^2_ie^{-2\eta \lambda_{\min}(A) d^{-\frac{1}{2}}t_0t^{-\alpha}}
\\
&\leq C\eta^2d^{-1} i^{-2\alpha}e^{-2\eta \lambda_{\min}(A) d^{-\frac{1}{2}}t_0t^{-\alpha}}
\end{align*}
for an absolute constant $C > 0$. On the other hand, for $i \geq t-t_0$, we have
\begin{align*}
|\mathcal E_i| &\leq \eta^2 d^{-1}i^{-2\alpha}\sum_{j = t-t_0}^t j^{-2\alpha}
\\
&\leq \eta^2 d^{-1}t_0(t - t_0)^{-4\alpha}.
\end{align*}
Together, these imply that
\begin{align*}
\sum_{i = 1}^t |\mathcal E_i| &\leq \sum_{i = 1}^{t-t_0}|\mathcal E_i| + \sum_{i = t-t_0}^t |\mathcal E_i|
\\
&\leq C\eta^2d^{-1}(e^{-2\eta \lambda_{\min}(A)d^{-\frac{1}{2}}t_0t^{-\alpha}}(\sum_{i  = 1}^{t-t_0}i^{-2\alpha}) + t_0^2(t  - t_0)^{-4\alpha})
\\
&\leq C\eta^2d^{-1}(e^{-2\eta \lambda_{\min}(A)d^{-\frac{1}{2}}t_0t^{-\alpha}} + t_0^2(t  - t_0)^{-4\alpha})
\end{align*}
We can now choose $t_0 := \frac{K t^{\alpha} d^{\frac{1}{2}}(\log t + \log d)}{2\eta \lambda_{\min}(A)}$ for an absolute constant $K >0$ and get that
\begin{align*}
    \sum_{i = 1}^t |\mathcal E_i| \leq C\eta^2d^{-1}\bigg(t^{-K}d^{-K} + \frac{CK^2(t^{2\alpha}d)(\log t + \log d)^2(t^{-4\alpha})}{4\eta^2 \lambda_{\min}(A)^2}\bigg)
\end{align*}
\textbf{Note.} Here $(t - t_0)^{-4\alpha} < C t^{-4\alpha}$ for all $t,d \geq C$ follows by Assumption \ref{assump} that $$\lim\limits_{t,d \to \infty}(\eta \lambda_{\min}(A))^{-1}(\log t + \log d)^2 d^{\frac{1}{2}}t^{-(1 - \alpha)} = 0.$$ 
We can make the first term above arbitrarily smaller than the second by choosing $K > 0$ to be a large enough absolute constant. This implies that
\begin{align*}
    \sum_{i = 1}^t |\mathcal E_i| &\leq \frac{C(\log t + \log d)^2 t^{-2\alpha}}{\lambda_{\min}(A)^2}
    \\
    &\leq C\eta^2(\log t + \log d)^2 t^{-2\alpha} (\eta\lambda_{\min}(A))^{-2},
\end{align*}
for all $t,d \geq C$, as desired.
\end{proof}
\subsection{Fast-Decay Of Initialization Bias Dependent Terms.}
\begin{lemma}\label{variance-exp-term}
Under Assumption \ref{assump}, we have for all $t, d \geq C_1$ that
\begin{align*}
\sum_{i = 1}^t \eta_i^2 \mathbb{E}|u_i|^2|v_i|^2 \leq C_2 \eta^2 d^{-1}e^{-2\eta \lambda_{\min}(A) d^{-\frac{1}{2}}t^{1-\alpha}}|a|^2|\beta^* - \theta_0|^2.
\end{align*}
Here $C_1, C_2 > 0$ are absolute constants. 
\end{lemma}
\begin{proof}
Throughout the proof, we let $C > 0$ denote a large enough and generic absolute constant.
\\\\
Lemma \ref{conc1} gives us that 
\begin{align*}
\mathbb{E}|u_i|^2 &\leq \mathbb{E}[|u_i|^4]^{\frac{1}{2}} 
\\
&\leq Ce^{-2\eta \lambda_{\min}(A)d^{-\frac{1}{2}}\sum_{j = i+1}^t j^{-\alpha}}|a|^2
\end{align*}
Now, Lemma \ref{S_imoment} gives us that
\begin{align*}
|v_i|^2
&\leq e^{-2\eta\lambda_{\min}(A) d^{-\frac{1}{2}}\sum_{j = 1}^{i-1}j^{-\alpha}}|\beta^* - \theta_0|^2
\end{align*}
These imply that
\begin{align*}
\sum_{i = 1}^t \eta_i^2 \mathbb{E}|u_i|^2 |v_i|^2 &\leq C\sum_{i = 1}^t \eta_i^2 e^{-2\eta \lambda_{\min}(A) d^{-\frac{1}{2}}(\sum_{j = 1}^t j^{-\alpha} - i^{-\alpha})} |a|^2 |\beta^* - \theta_0|^2
\\
&\leq C\sum_{i = 1}^t \eta_i^2 e^{-2\eta \lambda_{\min}(A) d^{-\frac{1}{2}}\sum_{j = 1}^t j^{-\alpha}}|a|^2 |\beta^* - \theta_0|^2
\\
&\leq C\sum_{i = 1}^t \eta_i^2 e^{-2\eta \lambda_{\min}(A)d^{-\frac{1}{2}}t^{1 - \alpha}}|a|^2|\beta^* - \theta_0|^2
\\
&\leq C \eta^2 d^{-1}e^{-2\eta \lambda_{\min}(A) d^{-\frac{1}{2}}t^{1-\alpha}}|a|^2|\beta^* - \theta_0|^2,
\end{align*}
as desired. Here the last inequality follows using the fact that $\sum_{i = 1}^\infty i^{-2\alpha} < C$ for $\alpha > \frac{1}{2}$. 
\end{proof}
\begin{lemma}\label{variance-non-linear-term}
Under Assumption \ref{assump}, we have for all  $t,d \geq C_1$ that
\begin{align*}
    |\sum_{i = 1}^t \eta_i^2\mathbb{E}[\epsilon_i (u_i^\top X_i)^2 (X_i^\top v_i)]| \leq C_2d^{-1}(\sigma \sqrt{\eta})|a|^2 |\beta^* - \theta_0|e^{-\eta \lambda_{\min}(A)d^{-\frac{1}{2}}t^{1-\alpha}}.
\end{align*}
Here $C_1, C_2 > 0$ represent absolute constants.
\end{lemma}
\begin{proof}
Throughout the proof, we let $C > 0$ denote a large enough and generic absolute constant.
\\\\
Observe that
\begin{align*}
 |\mathbb{E}[\epsilon_i (u_i^\top X_i)^2 (X_i^\top v_i)]| &\leq \mathbb{E}[\epsilon^4_i]^{\frac{1}{4}}  \mathbb{E}[(u_i^\top X_i)^4]^{\frac{1}{2}}\mathbb{E}[(X^\top_iv_i)^4]^{\frac{1}{4}}
 \\
 &\leq (\sigma \bar{\lambda}^{\frac{3}{2}})\mathbb{E}[|u_i|^4]^{\frac{1}{2}} \mathbb{E}[|v_i|^4]^{\frac{1}{4}}
 \\
 &\leq C(\sigma \bar{\lambda}^{\frac{3}{2}}) |a|^2 |\beta^* - \theta_0| e^{-2\eta \lambda_{\min}(A)d^{-\frac{1}{2}}\sum_{j = i+1}^t j^{-\alpha}} \cdot e^{-\eta \lambda_{\min}(A) d^{-\frac{1}{2}}\sum_{j = 1}^{i-1}j^{-\alpha}}
 \\
 &\leq C(\sigma \bar{\lambda}^{\frac{3}{2}})|a|^2|\beta^* - \theta_0|e^{-\eta \lambda_{\min}(A)d^{-\frac{1}{2}}\sum_{j = 1}^t j^{-\alpha} + \eta \lambda_{\min}(A) d^{-\frac{1}{2}}i^{-\alpha}}
 \\
 &\leq C(\sigma \bar{\lambda}^{\frac{3}{2}})|a|^2 |\beta^* - \theta_0|e^{-\eta \lambda_{\min}(A)d^{-\frac{1}{2}}t^{1-\alpha}}.
\end{align*}
Here the third inequality follows using Lemma \ref{conc1} and Lemma \ref{S_imoment}, and the second last inequality follows from Assumption \ref{assump} that $\eta \lambda_{\min}(A) < \eta \bar{\lambda}  < C$. This implies that
\begin{align*}
|\sum_{i = 1}^t \eta^2_i \mathbb{E}[\epsilon_i (u_i^\top X_i)^2 (X_i^\top v_i)]| &\leq \sum_{i = 1}^t \eta^2_i|\mathbb{E}[\epsilon_i (u^\top_i X_i)^2 (X_i^\top v_i)]|
\\
&\leq C\eta^2d^{-1}(\sigma \bar{\lambda}^{\frac{3}{2}})|a|^2 |\beta^* - \theta_0|e^{-\eta \lambda_{\min}(A)d^{-\frac{1}{2}}t^{1-\alpha}}\sum_{i = 1}^t i^{-2\alpha}
\\
&\leq Cd^{-1}(\sigma \sqrt{\eta})|a|^2 |\beta^* - \theta_0|e^{-\eta \lambda_{\min}(A)d^{-\frac{1}{2}}t^{1-\alpha}},
\end{align*}
as desired. Here the last inequality follows from assumptions \ref{assump} that $\eta \bar{\lambda} < C$ and the fact that $\sum_{i=1}^t i^{-2\alpha} < C$ for $\alpha > \frac{1}{2}$.
\end{proof}
\section{Comparison With The Methodology From \cite{chang2023inference} For Projection Parameters Inference.}\label{methodolgy-difference}
\cite{chang2023inference} is a recent work on inference for projection parameters in linear regression, which achieved the best dimension scaling of $t \gtrsim d^{3/2}$ compared to prior works, and operated in the same "assumption-lean" setting (see Assumprions \ref{assump}) as our work. In this section, we highlight the key methodolgical differences which allow us to significantly improve the dimension scaling (to $t \gtrsim d^{1 + \delta}$ for any $\delta > 0$) over their $t \gtrsim d^{3/2}$.
\\\\
\textbf{Inference methodology from \cite{chang2023inference}.} \cite{chang2023inference} constructs Berry-Essen bounds for $\sqrt{t}(a^\top\hat{\beta} - a^\top \beta^*)$, where $\hat{\beta}$ is ordinary least squares estimator (OLSE) given by
\begin{align*}
    \hat{\beta} := \bigg[\frac{\sum_{i = 1}^t X_i X^\top_i}{t}\bigg]^{-1}\frac{\sum_{i = 1}^t X_i Y_i}{t}
\end{align*}
Let $\hat{A} := \frac{\sum_{i = 1}^t X_i X^\top_i}{t}$ and $\hat{\Gamma} := \frac{\sum_{i = 1}^t X_i Y_i}{t}$. They use the decomposition 
\begin{align*}
     \frac{a^\top (\hat{\beta} - \beta)}{\sqrt{Var(a^\top (\hat{\beta}))}}
     \sim\sqrt{t}a^\top [\hat{\beta} - \beta^*] 
     &= \sqrt{t}a^\top \bigg[A^{-1} \frac{1}{t} \sum_{i = 1}^t X_i \epsilon_i\bigg] + \sqrt{t}a^\top \bigg[A^{-1} (A - \hat{A}) A^{-1} \frac{1}{t} \sum_{i = 1}^t X_i \epsilon_i \bigg]
    \\
    &+ \sqrt{t}a^\top\bigg[\hat{A}^{-1} (A - \hat{A}) A^{-1} (A - \hat{A}) A^{-1} \frac{1}{t}\sum_{i = 1}^t X_i \epsilon_i\bigg]
\end{align*}
They show that the sum of first two terms behaves as $\mathcal U + \mathcal B$, where $\mathcal U$ is an approximately normal random variable and $\mathcal B$ is a bias term which can be estimated and explicitly removed. 
\\\\
Let the term on the second line be $\mathcal R$, that is
\begin{align*}
\mathcal R := \sqrt{t}a^\top\bigg[\hat{A}^{-1} (A - \hat{A}) A^{-1} (A - \hat{A}) A^{-1} \frac{1}{t}\sum_{i = 1}^t X_i \epsilon_i\bigg]
\end{align*}
they observe that it scales roughly as $\sim \sqrt{t}|a|| \|\hat{A} - A\|_{op}^2 \bigg|\frac{1}{t}\sum_{i = 1}^tX_i \epsilon_i \bigg|,$ which is of the order $\sim \sqrt{t} |a| (d/t) (d/t)^{\frac{1}{2}} = |a| (d^{3/2}/t),$ which is precisely the reason why they need $t \gtrsim d^{3/2}$.
\\\\
\textbf{Why does the online SGD based method achieve significantly better dimension scaling?} 
Instead of the $\hat{\beta}$ above, online SGD learns the estimator $\theta_t$, whose expression is given in Lemma \ref{prop:iterate}. Using this, we found that
\begin{align*}
\frac{a^\top(\theta_t - \beta^*)}{\sqrt{\Var \langle a, \theta_t \rangle}} &\sim d^{\frac{1}{4}}t^{\frac{\alpha}{2}}a^\top \bigg[\sum_{i = 1}^t \eta_i \bigg(\prod_{k = 0}^{t-i-1}(I - \eta_{t-k}X_{t-k}X^\top_{t-k})\bigg)X_i\epsilon_i\bigg]
\end{align*}
Let $M_{t-i} := \eta_i a^\top \bigg(\prod_{k = 0}^{t-i-1}(I - \eta_{t-k}X_{t-k}X^\top_{t-k})\bigg)X_i\epsilon_i$, and observe that $\mathbb{E}[M_{t-i} | X_t, \dots, X_{i+1}] = 0$. Thus, the SGD based estimator naturally has a sum of martingale difference sequence structure, allowing us to use the more powerful martingale central limit theorems to control it's Berry-Essen bound. On the other hand, the expression for OLSE $\hat{\beta}$ above has no such structure and needs explicit high-probability control on the error $\mathcal R$, leading to poor dimension scaling.
\section{Comparison With Prior Works On Non-asymptotic SGD CLT.}\label{prior-sgd-clt-comparison}
In this section, we provide a detailed comparison with prior works on non-asymptotic SGD CLT \cite{anastasiou2019normal, shao2022berry, durmus2022finite, durmus2021tight, samsonov2024gaussian, Khusainov2025GaussianApprox, WuLiWeiRinaldo2025InferenceTD, Samsonov2025StatisticalInferenceLSA}.
\subsection{Comparison with \cite{anastasiou2019normal, shao2022berry}}
\cite{anastasiou2019normal} and \cite{shao2022berry} consider optimizing a function $f(\theta_t)$ using stochastic gradient descent, under the assumption that the stochastic gradient is of the form
\begin{align*}
    g(\theta_{t-1}) = \nabla f(\theta_{t-1}) + \zeta_t.
\end{align*}
\cite{anastasiou2019normal} assumes that $\zeta_t$ satisfies $\mathbb{E}[\zeta_t \zeta^\top_t | \mathcal F_{t-1}] = V$, where $V$ does not depend on $t$ and satisfies $\alpha \leq \lambda_{\min}(V) \leq \lambda_{\max}(V) \leq \beta$ for some absolute constants $\alpha, \beta$. \cite{shao2022berry} also assumes weak temporal dependence of $\zeta_t$ and $O(1)$ spectral norm of $\mathbb{E}[\zeta_t \zeta^\top_t]$ (see Theorem 4 in \cite{anastasiou2019normal} and Theorem 3.4 from \cite{shao2022berry}). These are similar to the assumptions made in the analysis of \emph{zeroth-order} SGD. 
\\\\
On the other hand, for the \emph{first-order} online SGD update in equation \ref{eq:ithiterate}, we have $\zeta_t := (X_tX^\top_t - A)(\theta_{t-1} - \beta^*) + \epsilon_t X_t$ , whose conditional variance depends on $\theta_{t-1}$ (among other quantities), which itself depends on all the data till time $t-1$. Furthermore, the spectral norm of $\mathbb{E}[\zeta_t\zeta^\top_t|$ can also grow as $\sim \sqrt{d}$.
\\\\
Thus the results from \cite{anastasiou2019normal} and \cite{shao2022berry} are not applicable to our setting. Moreover, their Berry-Essen bounds require $t \gtrsim d^4$ to go to zero (compared to $t \gtrsim d^{1+\delta}$ in our case).
\subsection{Comparison with \cite{durmus2022finite, durmus2021tight, samsonov2024gaussian, Khusainov2025GaussianApprox, WuLiWeiRinaldo2025InferenceTD, Samsonov2025StatisticalInferenceLSA}}
Recent line of work (\cite{durmus2022finite}, \cite{durmus2021tight}, \cite{samsonov2024gaussian}, \cite{Samsonov2025StatisticalInferenceLSA}, \cite{Khusainov2025GaussianApprox}, \cite{WuLiWeiRinaldo2025InferenceTD})  has established non-asymptotic SGD CLTs for the linear stochastic approximation (LSA) problem, but don't emphasize the growth of dimension-dependent factors for their rates. While their results improve the dependence on $t$ in the \emph{fixed-dimension} setting, we found after tracking the dimension dependent terms that their results yield significantly weaker dimension scaling compared to our $t \gtrsim d^{1 + \delta}$ in the growing dimension regime. As representative examples, we show this for the latest works \cite{Samsonov2025StatisticalInferenceLSA}, \cite{WuLiWeiRinaldo2025InferenceTD} and \cite{Khusainov2025GaussianApprox}. 
\\\\
\textbf{Dimension Scaling In \cite{Samsonov2025StatisticalInferenceLSA}.} \cite{Samsonov2025StatisticalInferenceLSA} focuses on the LSA setting and defines the quantity $C_{\mathbf{A}} := \sup\|A_t\|$, where $A_t$ is the incoming observation of $A$. Observe that $A_t = X_tX_t^\top$ in our online SGD setup. Thus, the quantity $C_{\mathbf{A}}$ scales as $\|XX^\top\| = |X|^2 \sim d$ in the online SGD setup. They also define a noise vector $\varepsilon$, which will be equal to $(Y - X^\top \beta^*)X$ in the online SGD setup. They denote $|\varepsilon|_{\infty} := \sup |\varepsilon|$, which will scale as $\sqrt{d}$ in the online SGD setting. Finally, they also let $\lambda_{\min} := \lambda_{\min}(\mathbb{E}[(Y - X^\top \beta^*)^2 XX^\top])$, which can be assumed to scale as $\Theta(1)$ for simplicity (for eg. if $A = I$ and errors are independent of $X$ with unit variance).
\\\\
They also provide Berry-Esseen bounds for projection parameters (substitute $m = 1$ in their Remark $4$), whose first term scales as
\begin{align*}
    \frac{C_4}{\lambda_{\min}t^{\frac{1}{4}}} \sim \frac{C^2_{\mathbf{A}}|\varepsilon|_{\infty}}{\lambda_{\min}t^{\frac{1}{4}}} \sim \frac{d^{5/2}}{t^{1/4}}.
\end{align*}
yielding a dimension scaling of at-most $t \gtrsim d^{10}$ in our growing-dimensional online SGD setting.
\\\\
\textbf{Dimension Scaling In \cite{WuLiWeiRinaldo2025InferenceTD}.} Consider the Berry-Essen bound (Theorem 3.2) from \cite{WuLiWeiRinaldo2025InferenceTD}. Their first term is of the order 
\begin{align*}
Tr(\Gamma)\lambda_{\max}(\Gamma^{-1})t^{-\alpha/2}, \quad \alpha \in (\frac{1}{2}, 1 )    
\end{align*}
for a problem dependent positive defintite symmetric matrix $\Gamma$. But $Tr(\Gamma)\lambda_{\max}(\Gamma^{-1}) \geq d$, therefore their dimensional-scaling is restricted to at-most $t \gtrsim d^{\frac{2}{\alpha}} \geq d^2$ (and possibly even lower if we track the other terms) for vanishing CLT error rates.
\\\\
\textbf{Dimension Scaling In \cite{Khusainov2025GaussianApprox}.} Similarly \cite{Khusainov2025GaussianApprox} runs SGD with constant step-size $\alpha := \frac{\log t}{t}$ and the first term in their Berry-Esseen bound (Theorem 1 of their paper) is of the order $C_1 \sqrt{\alpha}$, where $C_1 \geq C_{\Delta, 0}$ and $C_{\Delta, 0}$ (defined in equation 30 of their paper) grows as
\begin{align*}
    C_{\Delta, 0} \sim \sqrt{d}\mathbb{E}[|\epsilon X|^{3}] \sim d^2,
\end{align*}
implying the restricted dimensional scaling of (at-most) $t \gtrsim d^4$ for vanishing CLT error rates.

\end{document}